\pdfoutput=1
\RequirePackage{fix-cm}
\documentclass{svjour3}                     % onecolumn (standard format)
\smartqed  % flush right qed marks, e.g. at end of proof
\usepackage{amsmath,amsfonts,amssymb,graphicx,tikz,wrapfig,booktabs,mathtools}
\usepackage{kbordermatrix}
\usepackage{upgreek}
\usepackage{enumerate}
\usepackage{hyperref}

\DeclareMathAlphabet\mathbfcal{OMS}{cmsy}{b}{n}
\DeclarePairedDelimiter{\ceil}{\lceil}{\rceil}
\DeclarePairedDelimiter\floor{\lfloor}{\rfloor}
\newcommand\PP{{\mathbb{P}}}
\renewcommand\SS{{\mathbb{S}}}

\newcommand\Rbb{{\mathbb{R}}}
\newcommand\B{\mathbf{B}}
\newcommand\Zbb{{\mathbb{Z}}}
\newcommand\w{{\mathbf{w}}}
\newcommand\e{{\mathbf{e}}}
\newcommand\vb{{\mathbf{v}}}
\renewcommand{\L}{{\mathbfcal L}}
\newcommand{\D}{{\mathcal D}}
\renewcommand{\H}{{\mathcal F}}
\renewcommand{\l}{{\ell}}
\newcommand{\sign}{\mathsf{sign}}
\newcommand{\tabhead}{\textbf}
\newcommand{\bpi}{{\bar{\pi}}}
\newcommand{\bPi}{{\bar{\Pi}}}

\begin{document}

\title{Signal Processing on the Permutahedron: Tight Spectral Frames for Ranked Data Analysis}

\author{Yilin Chen \and Jennifer DeJong \and Tom Halverson \and David I Shuman}

\authorrunning{Chen, DeJong, Halverson, Shuman} 

\institute{Department of Mathematics, Statistics, and Computer Science  \\Macalester College \\
              Saint Paul, MN 55105 USA\\
              \email{\{ychen2,jdejong,halverson,dshuman1\}@macalester.edu}           
              }

\date{Received: date / Accepted: date}

\maketitle

\begin{abstract}
Ranked data sets, where $m$ judges/voters specify a preference ranking of $n$ objects/candidates, are increasingly prevalent in contexts such as political elections, computer vision, recommender systems, and bioinformatics. The vote counts for each ranking can be viewed as an $n!$ data vector lying on the permutahedron, which is a Cayley graph of the symmetric group with vertices labeled by permutations and an edge when two permutations differ by an adjacent transposition. 
Leveraging combinatorial representation theory and 
recent progress in signal processing on graphs, 
we investigate a novel, scalable transform method to interpret and exploit structure in ranked data. 
We represent data on the permutahedron using an overcomplete dictionary of atoms, each of which captures both smoothness information about the data (typically the focus of spectral graph decomposition methods in graph signal processing) and structural information about the data (typically the focus of symmetry decomposition methods from representation theory). 
These atoms have a more naturally interpretable structure than any known basis for signals on the permutahedron, and they form a Parseval frame, ensuring beneficial numerical properties such as energy preservation. We develop specialized algorithms and open software that take advantage of the symmetry and structure of the permutahedron to 
improve the scalability of the proposed method,  
making it more
applicable to the high-dimensional ranked data found in applications. 

\end{abstract}

%%%%%%%%%%%%%%%%%%%%%%%%%%%%%%%%%%%%%%%%%%%%%%%%%%%%%%%%%%%%%%%%%%%%%%%%
\section{Introduction} \label{Se:intro}

Ranked data consist of $m$ judges/voters specifying a preference ranking of $n$ objects/candidates. While methods for analyzing such rankings date back to the late 18th century in the context of social choice theory (e.g., \cite{borda1784memoire,nicolas1785essai,arrow1950difficulty}), ranked data are increasingly prevalent in contexts such as computer vision \cite{huang2008efficient,huang2009fourier}, recommender systems \cite{shani2011evaluating,wang2014active}, image processing \cite{wong2013mode,basha2012photo}, crowdsourced subjective labeling \cite{bennett2009learning,chen2013pairwise,stoyanovich2015analyzing}, peer grading \cite{raman2015bayesian}, metasearch \cite{akritidis2011effective}, sports analytics \cite{devlin2020identifying}, computational geometry \cite{jiang2014fourier}, and bioinformatics \cite{breitling2004rank,deconde2006combining,li2017comparative,uminsky2019detecting} (see \cite[Sec 2.2]{sibony2016multi} for an excellent, comprehensive overview of application areas). Moreover, an increasing number of 
cities, states, colleges and universities, organizations, and corporations are using ranked choice voting for elections \cite{rcv}.

The vote counts for each possible ranking of $n$ objects is an $n!$-dimensional data vector  in $\Rbb^{n!}$. We view this vector as lying on the \emph{permutahedron}, denoted $\PP_n$ and also referred to by some as the \emph{permutation polytope} \cite{thompson1993generalized}. The permutahedron has vertices labeled by permutations and an edge when two permutations differ by transposing adjacent entries in the permutation.  For example, in $\PP_5$, the permutation 25134 
corresponds to ranking candidate 2 first, candidate 5 second, candidate 1 third, and so on, and it is connected by an edge to each of 52134, 21534, 25314, and 25143. The permutahedron $\PP_n$ is the Cayley graph of the symmetric group $\SS_n$ induced by the generating set of adjacent transpositions (see Sec.~\ref{Se:Fourier}), and a signal on the permutahedron $\PP_n$ is a function $f: \SS_n \to \Rbb$. In this context, $f(\sigma)$ equals the number of votes for the permutation $\sigma.$

To deal with the scale of this data (factorial in the number of candidates), it is critical to construct efficient and meaningful representations that highlight salient features of the ranking tallies. Specifically, we follow the common signal processing approach of constructing a dictionary of atoms and representing a signal on the permutahedron as a linear combination of these atoms. For 
audio signals and images, as well as signals residing on more general weighted graphs, Fourier, time-frequency, curvelet, shearlet, bandlet, and other dictionaries have led to resounding successes in visual analysis of data, statistical analysis of data, compression, and as regularizers in machine learning and ill-posed inverse problems such as inpainting, denoising, and classification (see, e.g., \cite[Sec. II]{rubinstein_dict_learning} for an excellent historical overview of dictionary design methods and signal transforms).

In general, desirable properties when designing dictionaries include: (i) the dictionaries comprise an orthonormal basis or tight frame for the signal space, ensuring that the contribution of each atom can be computed via an inner product with the signal, and the energy of the signal is equal to the energy of the transform coefficients; (ii) the atoms have an interpretable structure, so that the inner products between the signal and each atom are informative; (iii) it is numerically efficient to apply the dictionary analysis and synthesis operators (forward and inverse transforms); and (iv) signals of certain mathematical classes can be represented exactly or approximately as sparse linear combinations of a subset of the dictionary atoms \cite{shuman2020localized}. 

The main contributions of this work are as follows. First, we leverage techniques and ideas from both signal processing on graphs and combinatorial representation theory to propose a novel dictionary construction that can be used to transform high-dimensional ranked data in order to find, interpret, and exploit structural patterns in the rankings. Each of the atoms in the overcomplete dictionaries we propose captures both smoothness information about the data (typically the focus of spectral graph decomposition methods in graph signal processing) and structural information about the data (typically the focus of symmetry decomposition methods from representation theory). Second, we prove that the proposed dictionaries comprise tight Parseval frames and therefore preserve the energy of the signal (Thm. \ref{Th:tight_frame} and Thm. \ref{Th:reduced_tight_frame} in Sec. \ref{Se:proposed}). Third, we demonstrate the application of the proposed transform methods and show how the interpretable structure of the atoms can lead to insights on real ranked data sets (Sec. \ref{Se:interpretation}). Fourth, we investigate numerical challenges, and propose novel algorithms that take advantage of the symmetry and structure of the permutahedron to enhance the scalability of our implementations for applying the analysis operators arising from our proposed dictionaries (Sec. \ref{Se:comp}). Fifth, we relate the proposed transform methods to related methods from combinatorial representation theory, graph signal processing, and statistical modeling for ranked data (Sec. \ref{Se:related} and Sec. \ref{Se:revisited} ).

\clearpage
%%%%%%%%%%%%%%%%%%%%%%%%%%%%%%%%%%%%%%%%%%%%%%%%%%%%%%%%%%%%%%%%%%%%%%%%
\section{Example Data Sets}
\begin{wrapfigure}[23]{r}{.35\textwidth}
\vspace{-.6in}
\begin{minipage}{\linewidth}
\centerline{\small{${\bf f}~$}} 
\centerline{\includegraphics[width=.95\linewidth,page=1]{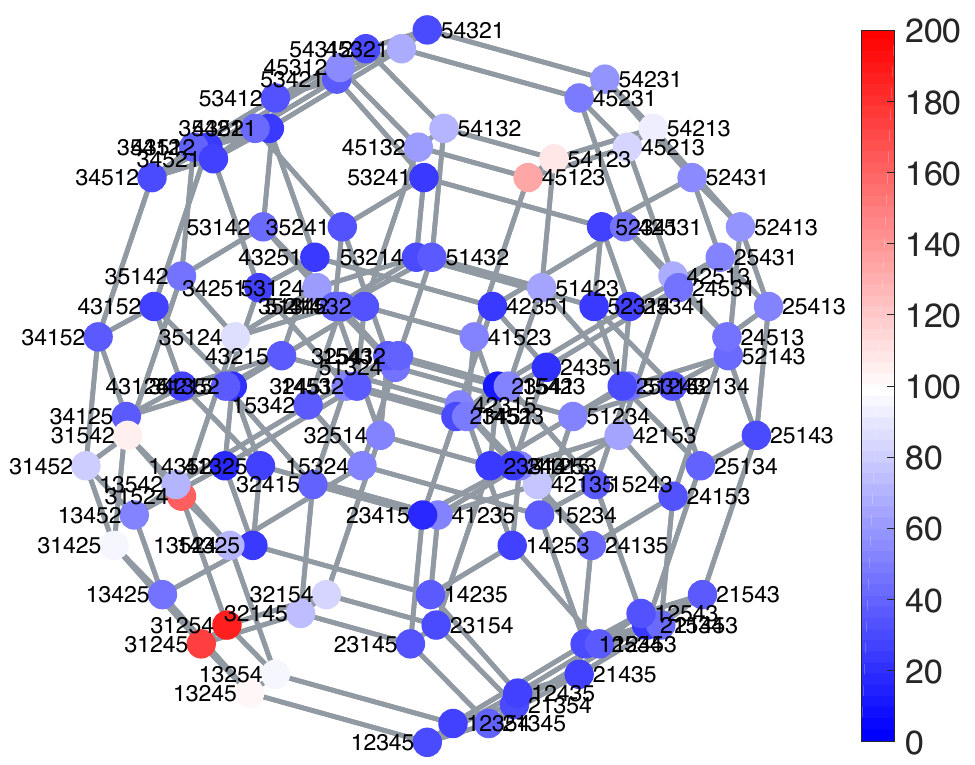}}
\centerline{\small{${\bf g}~$}}
\vspace{.02in}
\hspace{.13in}
\includegraphics[width=.75\linewidth,page=1]{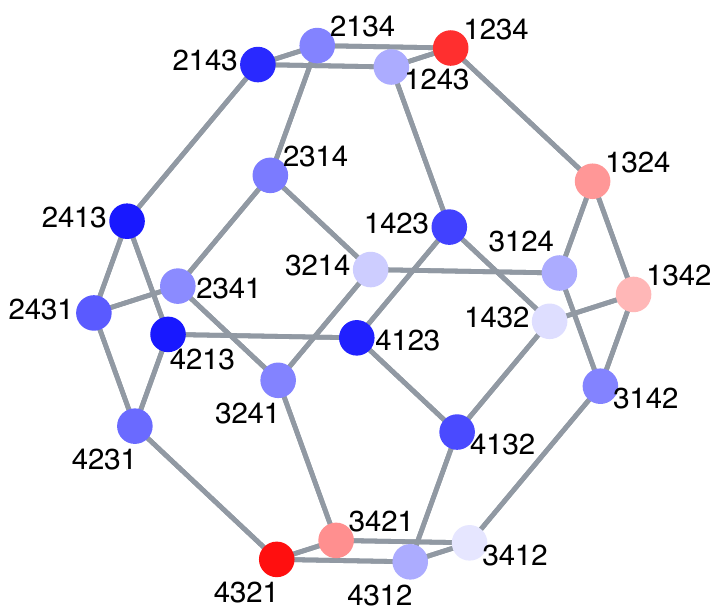}
\includegraphics[height=1.5in,page=1]{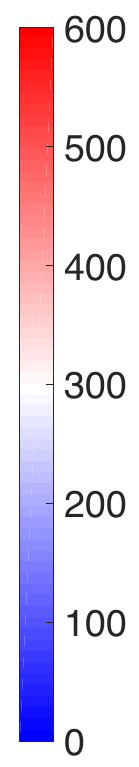}
\end{minipage}
\caption{Top: The signal ${\bf f}$ on $\PP_5$ 
represents the vote tallies of 5738 complete rankings of five candidates for APA president in 1980. Bottom: 
Ranked choice voting data ${\bf g}$ for the 2017 Minneapolis City Council Ward 3 election, which featured four candidates, shown on $\PP_4$.}
\label{Fig:signals} 
\end{wrapfigure}
\vspace{-.18in}
We use the following three ranked data sets in running examples throughout this article.

\vspace{-.18in}
\subsection{1980 American Psychological Association Presidential Election}
\vspace{-.16in}
In Fig. \ref{Fig:signals} (top), on the permutahedron $\PP_5$  we plot  the vote tallies of the 5738 American Psychological Association (APA) members who ranked all five candidates for APA president in 1980 (out of the 15449 total ballots cast) \cite{chamberlin1984social}, \cite[Tab. 1]{diaconis1989generalization}. Under the instant runoff (Hare) voting system in which the votes for the candidate with the fewest first place votes in each iteration are transferred to the next ranked candidate on those ballots, candidate 1 was the winner. 

\vspace{-.18in}
\subsection{2017 Minneapolis City Council Ward 3 Election}
\vspace{-.18in}
In Fig. \ref{Fig:signals} (bottom),  on the permutahedron $\PP_4$ we plot the vote tallies of the 5055 voters who ranked at least three of the four candidates for the Minneapolis City Council Ward 3 seat in 2017 (out of  
9578 total valid votes cast) \cite{mpls_election}. If a voter ranked three candidates, we assume that the unranked candidate was the voter's fourth choice. Candidates 1 
to 4 are  
Ginger Jentzen (Socialist-Alternative), 
Samantha Pree-Stinson (Green), 
Steve Fletcher (Democratic-Farmer-Labor), and  
Tim Bildsoe (Democratic-Farmer-Labor), respectively. Pree-Stinson began as a Democratic-Farmer-Labor (DFL) candidate, but was later endorsed by the Green Party. Fletcher was endorsed by the DFL Party.
Jentzen received the most first place votes, but Fletcher won the election under the instant runoff (Hare) voting system utilized by Minneapolis.  
\begin{wrapfigure}[15]{r}{.62\textwidth}
\vspace{-.4in}
\begin{minipage}{.33\linewidth}
\centering
\vspace{.1in}
\scalebox{0.8}{
\begin{tabular}{c l}
\toprule

\tabhead{Index} & \tabhead{Sushi Type}    \\
\midrule
1 & Shrimp \\
2 & Sea eel \\
3 & Tuna \\
4 & Squid \\
5 & Sea urchin \\
6 & Salmon roe \\
7 & Egg \\
8 & Fatty tuna \\
9 & Tuna roll \\
10 (0) & Cucumber roll \\
\bottomrule
\end{tabular}
}
\vspace{.068in}

\centerline{\small{(a)}}
\end{minipage}
\begin{minipage}{.66\linewidth}
\centerline{\includegraphics[width=\linewidth]{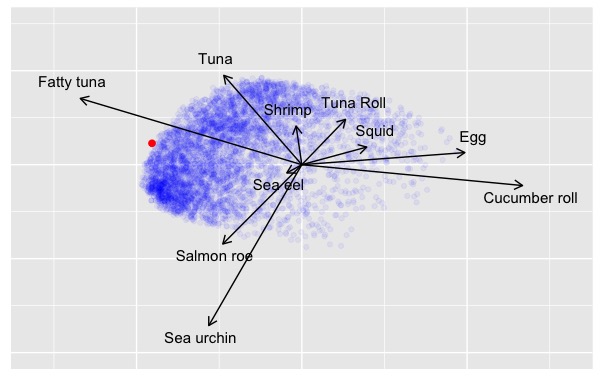}}
\centerline{\small{~(b)}}
\end{minipage}
\caption{(a) Candidate items in the sushi preference data set \cite{kamishima2003nantonac,kamishima2009efficient}. (b) Projection of sushi preference signal, ${\bf h}$, onto two dimensions using Gabriel's biplot. Each blue dot represents a single voter (of the 5000 total), and the transparency is indicative of the number of voters, with darker shades indicating a higher density of voters. The poles associated with the popular sushi items tend to fall towards the left on the fatty tuna-cucumber roll axis. The projection of the Condorcet ranking is plotted in red.}
\label{Fig:sushi}
\end{wrapfigure}
A discussion of the candidates' views and an interesting analysis of the voting results by geographical regions within the ward is contained in \cite{streets}.

\vspace{-.18in}
\subsection{Sushi Preference Data}\label{Se:sushi}
\vspace{-.18in}
For a data set with more candidates ($n=10$), we use Kamishima's type A set of sushi preferences \cite{kamishima2003nantonac,kamishima2009efficient}, in which 5000 people provide complete preference rankings for the ten different types of sushi 
listed in Fig. \ref{Fig:sushi}(a). Since plotting a signal on the permutahedron $\PP_{10}$ (a 9-dimensional object) is not particularly informative, we show in Fig. \ref{Fig:sushi}(b) the projection of the data into a two-dimensional space via Gabriel's biplot \cite{gabriel2014biplots}, \cite{cohen1980analysis}, \cite[Sec. 2.2]{marden2014analyzing}, which is similar to a principal components projection except with a different choice of center. Chen \cite[Sec. 5.3]{chen2014order} notes that there is a full Condorcet ranking: 
$${\small\hbox{Fatty tuna} \succ \hbox{Tuna} \succ \hbox{Salmon roe} \succ \hbox{Shrimp} \succ \hbox{Sea eel} \succ \hbox{Sea urchin} \succ \hbox{Squid} \succ \hbox{Tuna roll} \succ \hbox{Egg} \succ \hbox{Cucumber roll}.}$$
\noindent That is, in a preference comparison of any pair of two sushi items, the majority of voters would prefer the item that falls earlier in this ranking; i.e., the majority of voters prefer fatty tuna to any other item, the majority of voters prefer tuna to any item besides fatty tuna, and so forth. 

\clearpage

%%%%%%%%%%%%%%%%%%%%%%%%%%%%%%%%%%%%%%%%%%%%%%%%%%%%%%%%%%%%%%%%%%%%%%%%
\section{Related Work} \label{Se:related}

Analysis of ranked data has a long history in the mathematical psychology and statistics literatures. Approaches have largely focused on parametric statistical models including order statistic models, distance-based models, and pairwise comparison models (see, e.g., \cite{lin2010rank}, \cite{liu2019model}, \cite{marden2014analyzing}, \cite[Sec. 2.5]{sibony2016multi}, and \cite{yu2019analysis} for excellent overviews of these models). We focus our attention here on linear transforms for ranked data that attempt to identify structure in the data by taking inner products between the $n!$-dimensional vote tally (signal on the permutahedron) and building block signals that have some interpretable structure. 

\vspace{-.15in}

\subsection{The Fourier Analysis on the Symmetric Group Approach}\label{Se:Fourier}

The vertices of the permutahedron $\PP_n$ are labeled by the symmetric group $\SS_n$ of permutations of $\{1, \ldots,n\}$. For example $\sigma = \left(\begin{smallmatrix} 1 & 2 & 3 & 4 & 5 \\  2 & 5 & 1 & 3 & 4 \end{smallmatrix}\right) = 25134\in \SS_5$ is the bijective function $1 \mapsto 2, 2 \mapsto 5,$ etc.. We let $(i,j)$ denote the transposition (in cycle notation) that exchanges $i$ and $j$ and fixes the other entries, and the adjacent transpositions are $(i,i+1), 1 \le i \le n-1$.  The group operation in $\SS_n$ is composition of functions, and right multiplication by an adjacent transposition $(i,i+1)$ exchanges the candidates in positions $i$ and $i+1$. For example, $\left(\begin{smallmatrix} 1 & 2 & 3 & 4 & 5 \\  2 & 5 & 1 & 3 & 4  \end{smallmatrix}\right) (2,3) = \left(\begin{smallmatrix} 1 & 2 & 3 & 4 & 5 \\  2 & 1 & 5 & 3 & 4  \end{smallmatrix}\right)$. In this context, the permutahedron $\PP_n$ is the Cayley graph with 
vertices labeled by $\SS_n$ and edges  $\{ (\sigma, \sigma s) \mid \sigma \in \SS_n, s \in S\}$,
where $S = \{(i,i+1) \mid 1 \le i \le n-1\}$. Each transposition is its own inverse, so the generating set $S$ is closed under inverses and $\PP_n$ is a simple graph that is $(n-1)$-regular, since $|S| = n-1$. 

Using right multiplication of transpositions is natural when studying ranked data, because two permutations are adjacent if and only if they differ by transposing adjacent candidates in the ranking. Under left multiplication, two rankings are adjacent if and only if they differ by swapping the candidates ranked $i$th and $(i+1)$st.  For example, $(2,3) \cdot \left(\begin{smallmatrix} 1 & 2 & 3 & 4 & 5 \\  2 & 5 & 1 & 3 & 4  \end{smallmatrix}\right) = \left(\begin{smallmatrix} 1 & 2 & 3 & 4 & 5 \\  3 & 5 & 1 & 2 & 4  \end{smallmatrix}\right)$. The  adjacent transpositions $S$ satisfy the Coxeter relations \cite[2.12.10]{sagan2013symmetric} and give $\SS_n$ the structure of a reflection group. If one uses the full set of transpositions $\{(i,j) \mid 1 \le i < j \le n\}$, then the generating set is a full conjugacy class in $\SS_n$. In that case, the Cayley graph is said to be quasi-abelian, and the Laplacian eigenvalues and eigenvectors are especially nice, but they are less interpretable in the context of ranked data as we discuss in Sec.~\ref{Se:gsp}.

Data on $\PP_n$ lives in the underlying real vector space $\Rbb[ \SS_n]$ (or, equivalently $\Rbb^{n!}$) spanned by the symmetric group $\SS_n$, which is called the \emph{group algebra} of $\SS_n$ and has a canonical basis $\{ \e_\sigma\}_{\sigma \in \SS_n}$.\footnote{The results of this paper are true over the complex numbers; however, we use  $\Rbb[ \SS_n]$ as ranked data are real valued.} A signal on $\SS_n$ is a function $f: \SS_n \to \Rbb$, which we view as a vector in $\Rbb[\SS_n]$ by ${\bf f} = \sum_{\sigma \in \SS_n} f(\sigma) \e_\sigma$.  Non-commutative harmonic analysis attempts to find structure in ranked data by decomposing them into subspaces that are  set-wise invariant under the relabeling (left multiplication) and/or re-ranking (right multiplication)  of the candidates \cite{diaconis1989generalization}.
Specifically, a permutation $\sigma \in \SS_n$ acts on the right and on the left, respectively, of a signal ${\bf f}$, by
\begin{equation}\label{eq:left-right-actions}
{\bf f} \sigma = \sum_{\tau \in \SS_n} f(\tau) \e_{\tau \sigma} =  \sum_{\eta \in \SS_n} f(\eta \sigma^{-1}) \e_{\eta} \qquad\hbox{and}\qquad
\sigma {\bf f} = \sum_{\tau \in \SS_n} f(\tau) \e_{\sigma \tau} =  \sum_{\beta \in \SS_n} f( \sigma^{-1}\beta) \e_{\beta}.
\end{equation}
The vector space $\Rbb[\SS_n]$ decomposes into the direct sum of orthogonal subspaces, called \emph{isotypic components}
\begin{equation}\label{eq:isotypic}
\Rbb[\SS_n] \cong \bigoplus_{\gamma \vdash n} W_\gamma,
\end{equation}
where the sum is over all integer partitions $\gamma=[\gamma_1,\gamma_2,\ldots,\gamma_\ell]$ of $n$, denoted by $\gamma \vdash n$.
The subspaces $W_\gamma$ are invariant under both relabeling and reindexing by $\SS_n$; 
that is,  if $\w \in W_\gamma$,   then $\w\sigma \in W_\gamma$ and  $\sigma\w \in W_\gamma$  for any permutation $\sigma \in \SS_n$.  Thus, $W_\gamma$ is both a left and a right submodule of $\Rbb[\SS_n]$.  
We refer to the integer partition $\gamma$ as the \emph{shape} or \emph{symmetry type} (we use these terms interchangeably) of $W_\gamma$.

%%%%%%%%%%%%%%%%%%%%%%%%%%%%%%%%%%%%%%%%%%%%%%%%%%%%%%%%%%%%%%%%%%%
\begin{figure}[t]
\begin{minipage}{.25\linewidth}
\centerline{\small{$~~~~~~{\bf f}$}}
\vspace{.036in}
\includegraphics[height=1.2in,page=1]{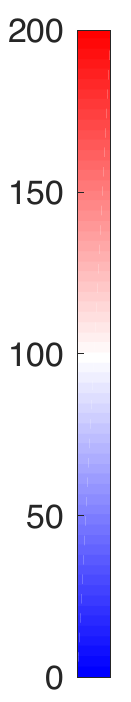}
\includegraphics[width=.8\linewidth,page=1]{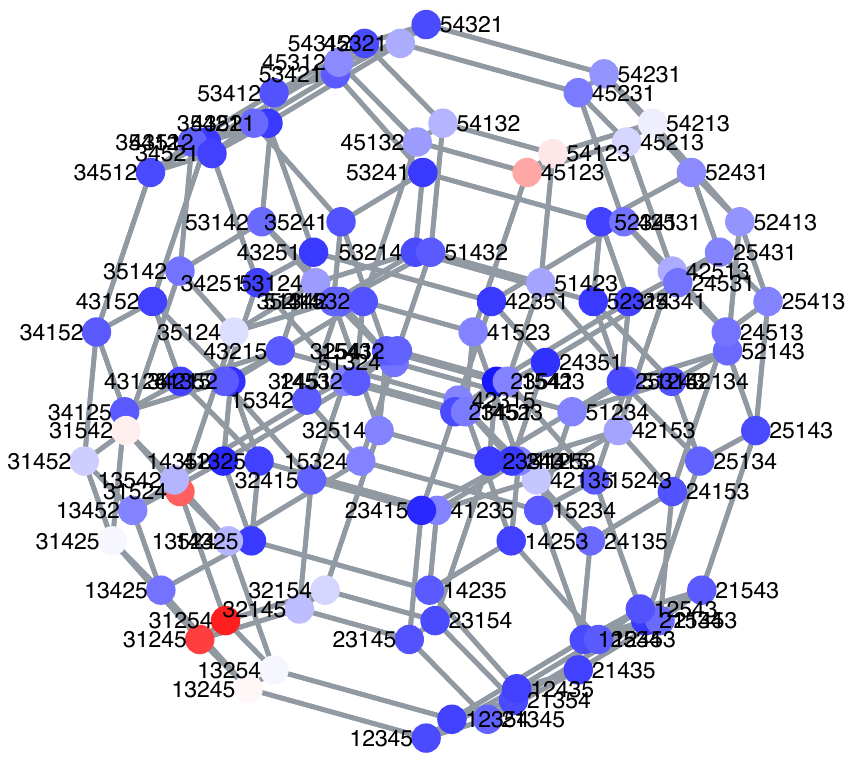}
\centerline{~~~~~~~~~\small 378756.0}
\end{minipage}
\begin{minipage}{.02\linewidth}
\vspace{.2in}
\centerline{\small ~~=}
\end{minipage}
\begin{minipage}{.18\linewidth}
\centerline{\small{$~~{\bf f}_{\begin{tikzpicture}[scale=.15,line width=1.0pt] 
\draw (0,0) rectangle (1,1); \draw (1,0) rectangle (2,1); \draw (2,0) rectangle (3,1);  \draw (3,0) rectangle (4,1);  \draw (4,0) rectangle (5,1); 
\end{tikzpicture}}$}}
\vspace{.046in}
\includegraphics[width=1.1\linewidth,page=1]{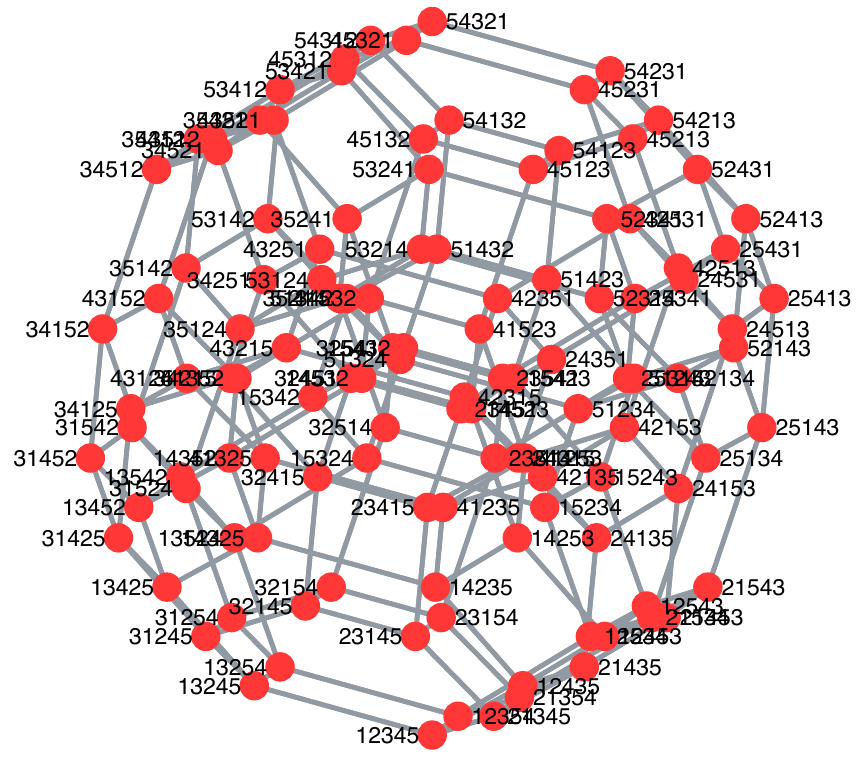} 
\centerline{~~~~~\small 274372.0}
\end{minipage}
\begin{minipage}{.02\linewidth}
\vspace{.2in}
\centerline{\small ~~+}
\end{minipage}
\begin{minipage}{.18\linewidth}
\centerline{\small{$~~{\bf f}_{\begin{tikzpicture}[scale=.15,line width=1.0pt] 
\draw (0,0) rectangle (1,1); \draw (1,0) rectangle (2,1); \draw (2,0) rectangle (3,1);  \draw (3,0) rectangle (4,1); \draw (0,-1) rectangle (1,0); 
\end{tikzpicture}}$}}
\includegraphics[width=1.1\linewidth,page=2]{figures/diac_f_isodecomp}
\centerline{~~~~~\small 35790.6} 
\end{minipage}
\begin{minipage}{.02\linewidth}
\vspace{.2in}
\centerline{\small ~~+}
\end{minipage}
\begin{minipage}{.18\linewidth}
\centerline{\small{$~~{\bf f}_{\begin{tikzpicture}[scale=.15,line width=1.0pt] 
\draw (0,0) rectangle (1,1); \draw (1,0) rectangle (2,1); \draw (2,0) rectangle (3,1); 
\draw (0,-1) rectangle (1,0); \draw (1,-1) rectangle (2,0); 
\end{tikzpicture}}$}}
\includegraphics[width=1.1\linewidth,page=3]{figures/diac_f_isodecomp}
\centerline{~~~~~\small  55098.4} 
\end{minipage} 
\hspace{.2in}
\begin{minipage}{.08\linewidth}
\includegraphics[height=1.2in]{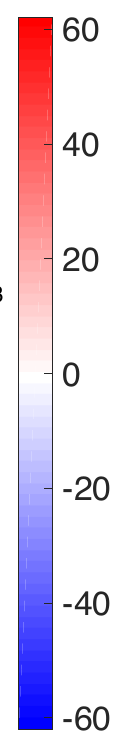} 
\end{minipage}\\

\hspace{.8in}
\begin{minipage}{.02\linewidth}
\vspace{.2in}
\centerline{\small ~~+}
\end{minipage}
\begin{minipage}{.18\linewidth}
\vspace{.046in}
\centerline{\small{$~~{\bf f}_{\begin{tikzpicture}[scale=.15,line width=1.0pt] 
\draw (0,0) rectangle (1,1); \draw (1,0) rectangle (2,1); \draw (2,0) rectangle (3,1); 
\draw (0,-1) rectangle (1,0); 
\draw (0,-2) rectangle (1,-1); 
\end{tikzpicture}}$}}
\vspace{.046in}
\includegraphics[width=1.1\linewidth,page=4]{figures/diac_f_isodecomp}
\centerline{~~~~\small  9384.9} 
\end{minipage}
\begin{minipage}{.02\linewidth}
\vspace{.2in}
\centerline{\small ~~+}
\end{minipage}
\begin{minipage}{.18\linewidth}
\vspace{.046in}
\centerline{\small{$~~{\bf f}_{\begin{tikzpicture}[scale=.15,line width=1.0pt] 
\draw (0,0) rectangle (1,1); \draw (1,0) rectangle (2,1); 
\draw (0,-1) rectangle (1,0); \draw (1,-1) rectangle (2,0); 
\draw (0,-2) rectangle (1,-1); 
\end{tikzpicture}}$}}
\vspace{.046in}
\includegraphics[width=1.1\linewidth,page=5]{figures/diac_f_isodecomp} 
\centerline{~~~~\small 3264.2} 
\end{minipage}
\begin{minipage}{.02\linewidth}
\vspace{.2in}
\centerline{\small ~~+}
\end{minipage}
\begin{minipage}{.18\linewidth}
\centerline{\small{$~~{\bf f}_{\begin{tikzpicture}[scale=.15,line width=1.0pt] 
\draw (0,0) rectangle (1,1); \draw (1,0) rectangle (2,1); 
\draw (0,-1) rectangle (1,0); 
\draw (0,-2) rectangle (1,-1); 
\draw (0,-3) rectangle (1,-2);  
\end{tikzpicture}}$}}
\vspace{.046in}
\includegraphics[width=1.1\linewidth,page=6]{figures/diac_f_isodecomp}  
\centerline{~~~~\small 819.7}
\end{minipage}
\begin{minipage}{.02\linewidth}
\vspace{.2in}
\centerline{\small ~~+}
\end{minipage}
\begin{minipage}{.18\linewidth}
\centerline{\small{$~~{\bf f}_{\begin{tikzpicture}[scale=.15,line width=1.0pt] 
\draw (0,0) rectangle (1,1); 
\draw (0,-1) rectangle (1,0); 
\draw (0,-2) rectangle (1,-1); 
\draw (0,-3) rectangle (1,-2); 
\draw (0,-4) rectangle (1,-3); 
\end{tikzpicture}}$}} 
\includegraphics[width=1.1\linewidth,page=7]{figures/diac_f_isodecomp}  
\centerline{~~~~\small 26.1}
\end{minipage}
\vspace{.1in}
\hrule height 1pt
\vspace{.1in}
\begin{minipage}{.02\linewidth}
\vspace{.13in}
\includegraphics[height=.85in,page=1]{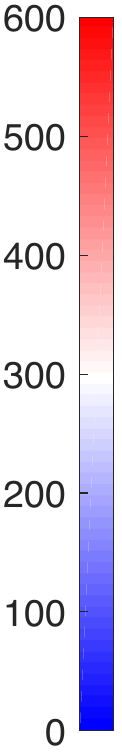}
\end{minipage}
\begin{minipage}{.14\linewidth}
\centerline{\small{$~~~~{\bf g}$}}
\vspace{.16in}
\includegraphics[height=.85in,page=2]{figures/gdecomp}
\centerline{~~~~\small 1607249.0}
\end{minipage}
\begin{minipage}{.01\linewidth}
\vspace{.2in}
\centerline{\small ~~=}
\end{minipage}
\begin{minipage}{.14\linewidth}
\centerline{\small{$~~{\bf g}_{\begin{tikzpicture}[scale=.15,line width=1.0pt] 
\draw (0,0) rectangle (1,1); \draw (1,0) rectangle (2,1); \draw (2,0) rectangle (3,1);  \draw (3,0) rectangle (4,1); 
\end{tikzpicture}}$}}
\vspace{.15in}
\includegraphics[height=.85in,page=3]{figures/gdecomp}
\centerline{~~~~\small 1064709.4}
\end{minipage}
\begin{minipage}{.01\linewidth}
\vspace{.2in}
\centerline{\small ~+}
\end{minipage}
\begin{minipage}{.14\linewidth}
\centerline{\small{$~~{\bf g}_{\begin{tikzpicture}[scale=.15,line width=1.0pt] 
\draw (0,0) rectangle (1,1); \draw (1,0) rectangle (2,1); \draw (2,0) rectangle (3,1); \draw (0,-1) rectangle (1,0); 
\end{tikzpicture}}$}}
\vspace{.1in}
\includegraphics[height=.85in,page=4]{figures/gdecomp}
\centerline{~~~~\small 355201.6} 
\end{minipage}
\begin{minipage}{.01\linewidth}
\vspace{.2in}
\centerline{\small ~+}
\end{minipage}
\begin{minipage}{.14\linewidth}
\centerline{\small{$~~{\bf g}_{\begin{tikzpicture}[scale=.15,line width=1.0pt] 
\draw (0,0) rectangle (1,1); \draw (1,0) rectangle (2,1);  
\draw (0,-1) rectangle (1,0); \draw (1,-1) rectangle (2,0); 
\end{tikzpicture}}$}}
\vspace{.1in}
\includegraphics[height=.85in,page=5]{figures/gdecomp}
\centerline{~~~~\small  137575.3} 
\end{minipage} 
\begin{minipage}{.01\linewidth}
\vspace{.2in}
\centerline{\small ~+}
\end{minipage}
\begin{minipage}{.14\linewidth}
\centerline{\small{$~~{\bf g}_{\begin{tikzpicture}[scale=.15,line width=1.0pt] 
\draw (0,0) rectangle (1,1); \draw (1,0) rectangle (2,1);  
\draw (0,-1) rectangle (1,0); 
\draw (0,-2) rectangle (1,-1); 
\end{tikzpicture}}$}}
\vspace{.05in}
\includegraphics[height=.85in,page=6]{figures/gdecomp}
\centerline{~~~~\small  47942.6} 
\end{minipage} 
\begin{minipage}{.01\linewidth}
\vspace{.2in}
\centerline{\small ~+}
\end{minipage}
\begin{minipage}{.14\linewidth}
\centerline{\small{$~~{\bf g}_{\begin{tikzpicture}[scale=.15,line width=1.0pt] 
\draw (0,0) rectangle (1,1);  
\draw (0,-1) rectangle (1,0); 
\draw (0,-2) rectangle (1,-1); 
\draw (0,-3) rectangle (1,-2); 
\end{tikzpicture}}$}}
\includegraphics[height=.85in,page=7]{figures/gdecomp}
\centerline{~~~~\small  1820.0} 
\end{minipage} 
\hspace{.05in}
\begin{minipage}{.02\linewidth}
\vspace{.13in}
\includegraphics[height=.85in,page=8]{figures/gdecomp}
\end{minipage}
\caption{Decomposition of the signals of Fig. \ref{Fig:signals} into orthogonal vectors in each of their respective isotypic components. The numbers below the signals show how the energy $\lVert{\bf f}\rVert^2$ splits into the energies $\lVert{\bf f}_{\gamma}\rVert^2$. The first of these energies is a function of the number of voters (e.g., $274372 = 5738^2/120$ for the APA data of \cite{diaconis1989generalization}). The second is a measure of the first order effect given by the number of voters who rank candidate $i$ in position $j$. The third and fourth capture the unordered and ordered second order (pair) effects, respectively, net of the first order marginal effects.}\label{Fig:isodecomp} 
\vspace{0.1in}
\hrule height 1.5pt
\vspace{-.1in}
\end{figure}

%%%%%%%%%%%%%%%%%%%%%%%%%%%%%%%%%%%%%%%%%%%%%%%%%%%%%%%%%%%%%%%%%%%
The  isotypic components further decompose into a direct sum of \emph{irreducible} left and right $\SS_n$-submodules, respectively, as
\begin{equation}\label{eq:irreducible_decomposition}
\Rbb[\SS_n] \cong \bigoplus_{\gamma \vdash n} 
W_\gamma \cong \bigoplus_{\gamma \vdash n}  \bigoplus_{i = 1}^{d_\gamma} V_{\gamma,i} \cong  \bigoplus_{\gamma \vdash n}  \bigoplus_{i = 1}^{d_\gamma} V_{\gamma,i}^\ast,
\end{equation}
where $V_{\gamma,i} \cong V_{\gamma,j}$ and $V_{\gamma,i}^\ast \cong V_{\gamma,j}^\ast$ are isomorphic as left and right $\SS_n$-modules, respectively, and $V_{\gamma,i} \not\cong V_{\rho,j}$ and $V_{\gamma,i}^\ast \not\cong V_{\rho,j}^\ast$ if $\gamma \not= \rho$.  
The modules $V_{\gamma,i}$ are left invariant, meaning $\sigma \vb \in V_{\gamma,i}$ for all $\vb \in V_{\gamma,i}$ and $\sigma \in \SS_n$.
% and thus they are invariant under the relabeling of candidates.
The modules $V_{\gamma,i}^\ast=\{ f: V_{\gamma,i} \to \Rbb\}$ are the dual vector spaces of linear functionals on $V_{\gamma,i}$ and are invariant under the right action $(f\sigma)(\vb) = f(\sigma \vb)$. %  and $V_{\gamma,i}^\ast$ is invariant under re-indexing the candidates.  
The left action of $\sigma \in \SS_n$ on permutations replaces $i$ with $\sigma(i)$ and the right action replaces the entry in the $i$th position with the entry in the $\sigma(i)$th position, so the left modules $V_{\gamma,i}$  are invariant under relabeling the candidates and the right modules are invariant under re-indexing the candidates.
For shorthand, we write $\bigoplus_{i = 1}^{d_\gamma} V_{\gamma,i}$ as $V_{\gamma}^{\oplus d_\gamma}$ and  $\bigoplus_{i = 1}^{d_\gamma} V_{\gamma,i}^\ast$ as $(V_{\gamma}^\ast)^{\oplus d_\gamma}$.
The famous isomorphism in \eqref{eq:irreducible_decomposition} was proved around 1900 in the work of Frobenius and Burnside. It was extended in the 1920s to hold for topological groups in the Peter-Weyl theorem.
This decomposition has  special properties:  
(i) the dimension $d_\gamma = \dim(V_{\gamma})$ equals the multiplicity of $V_{\gamma}$ in $W_\gamma$;  
(ii) $d_\gamma$ also equals the number of standard Young tableaux of shape $\lambda$ and can be computed using the hook formula (see, e.g.,  \cite[3.10]{sagan2013symmetric}); 
(iii) $\dim(W_\gamma) = d_\gamma^2$; and 
(iv) the only left submodules of $\Rbb[\SS_n]$ isomorphic to $V_{\gamma}$ appear in $W_\gamma$. The same properties hold for the right action with the dual spaces.

In Fig. \ref{Fig:isodecomp}, we display the APA 
data and Minneapolis City Council election data from Fig. \ref{Fig:signals} 
as the sum of projections onto  
the 
orthogonal subspaces $W_{\gamma}$.

\emph{How can this approach be used to find structure in ranked data?} 
The isotypic decomposition \eqref{eq:isotypic} has a close relation with marginal statistics. The first order marginals of ranked data, shown in Fig. \ref{Fig:marginals1} for an election with four 
\begin{wrapfigure}[8]{r}{0.55\textwidth}
\centering
%\vspace{-.38in}
{\includegraphics[width=3.5in]{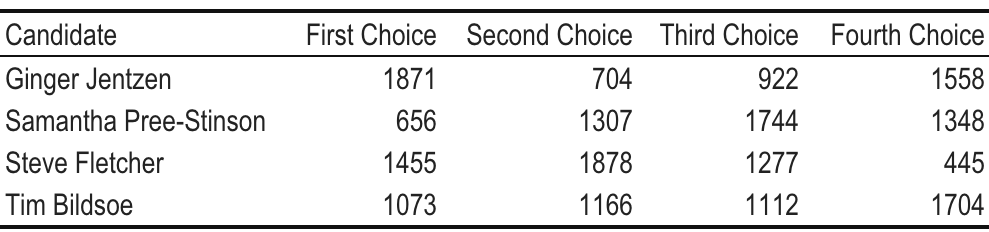}}
\caption{First order marginals for the 2017 Minneapolis City Council Ward 3 election.}\label{Fig:marginals1}
\end{wrapfigure}
candidates, capture how many voters placed candidate $i$ in ranking position $j$. There are two types of second order marginals. The unordered type, shown in Fig. \ref{Fig:marginals2}(a) for the same election, capture how many voters placed candidates $i$ and $i^\prime$ in ranking positions $j$ and $j^\prime$, in either order. The ordered type, shown in Fig. \ref{Fig:marginals2}(b), capture how many voters placed candidate $i$ in ranking position $j$ and candidate $i^\prime$ in ranking position $j^\prime$. Higher order marginals that capture how many voters placed $k>2$ specific candidates in $k$ specific ranking positions may be totally unordered, totally ordered, or partially ordered (e.g., candidates $i$ and $i^{\prime}$ are in positions $j$ and $j^\prime$ in either order, and candidate $i^{\prime\prime}$ is in position $j^{\prime\prime}$). In all of these marginal tables, each row and each column sum to the total number of votes (5055 in Fig. \ref{Fig:marginals1} and Fig. \ref{Fig:marginals2}).

\begin{figure}[b]
\vspace{-.2in}
\centering
{\includegraphics[width=6in]{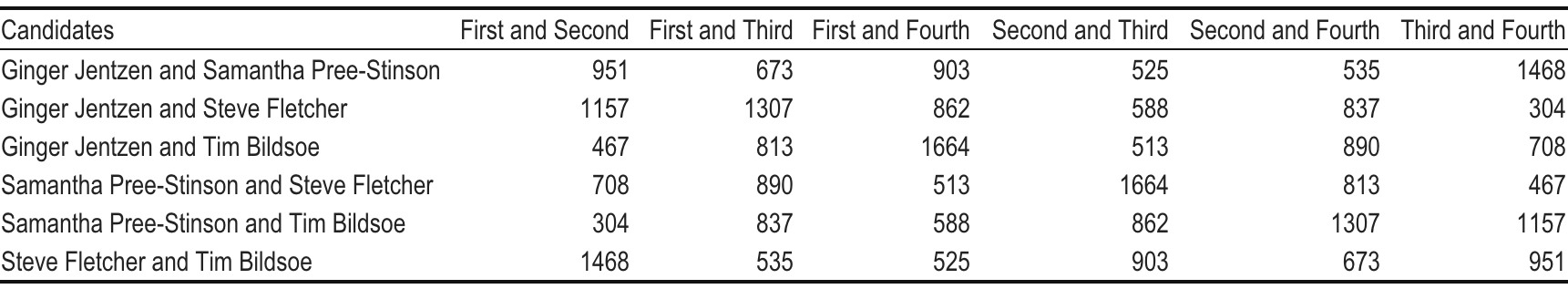}}
\centerline{\small{(a)}}
\vspace{.05in}

{\includegraphics[width=5.5in]{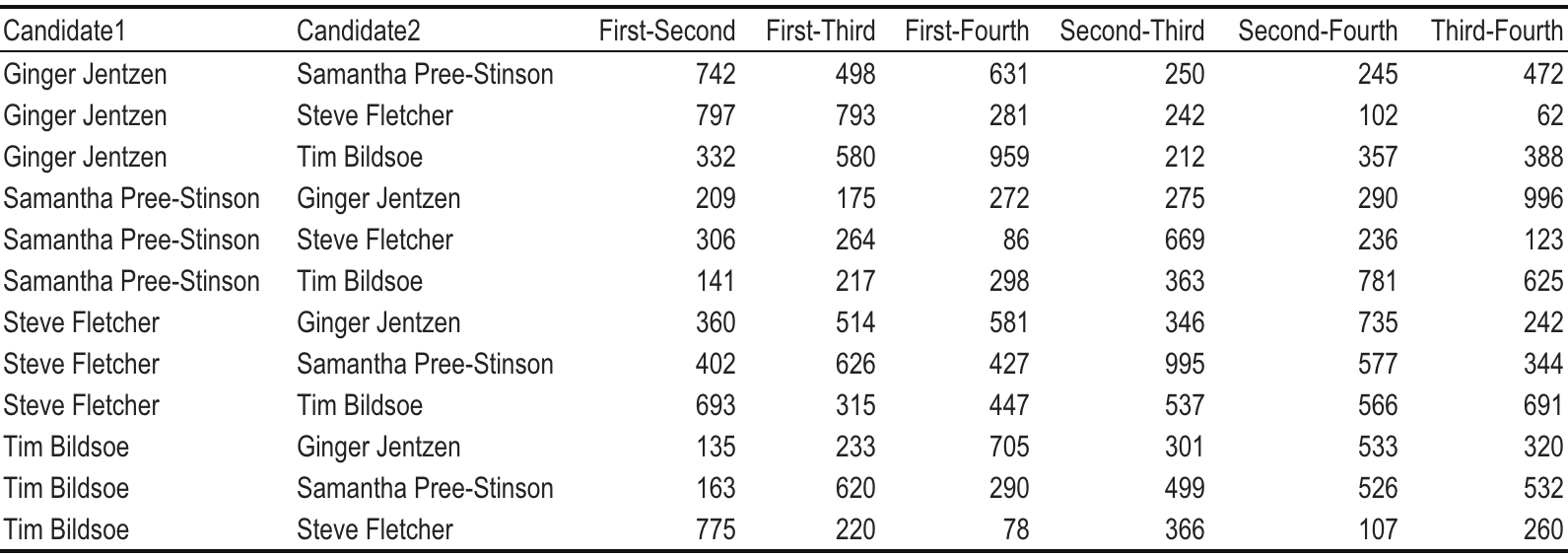}}
\centerline{\small{(b)}}
\caption{Second order marginals for the 2017 Minneapolis City Council Ward 3 election. (a) Second order unordered marginals. (b) Second order ordered marginals. Each row in (a) is the sum of the two rows in (b) with the same pair of candidates. In (b), we are only displaying half of the columns in the table, as, e.g., the ``Second-First'' column contains the same entries as the ``First-Second'' column in a different order. This is why the row sums are not equal to the number of voters in the displayed table.}\label{Fig:marginals2}
\end{figure}

Spectral analysis through projections onto the isotypic components captures order effects related to a specific marginal, net of the structure found in the ``less complicated'' marginals \cite{diaconis1989generalization}, \cite[Sec. 2.6.1]{marden2014analyzing}. For example, Fig. \ref{Fig:marginals1} shows that of all first order effects, two of the most significant are that Steve Fletcher was far more likely than would be the case under a uniform distribution to be listed as the second choice, and far less likely to be listed as the fourth choice. These first order effects also appear in the second order marginals of Fig. \ref{Fig:marginals2}, as well as higher order marginals. However, the projection of the signal ${\bf g}$ from Fig. \ref{Fig:signals} onto $W_{\begin{tikzpicture}[scale=.15,line width=1.0pt] 
\draw (0,0) rectangle (1,1); \draw (1,0) rectangle (2,1);   
\draw (0,-1) rectangle (1,0);  \draw (1,-1) rectangle (2,0); 
\end{tikzpicture}}$ captures information about the second order unordered marginals net of the number of voters (zero order) and first order marginals. Similarly, the projection of  ${\bf g}$ onto  $W_{\begin{tikzpicture}[scale=.15,line width=1.0pt] 
\draw (0,0) rectangle (1,1); \draw (1,0) rectangle (2,1);   
\draw (0,-1) rectangle (1,0);  \draw (0,-2) rectangle (1,-1); 
\end{tikzpicture}}$ captures information about the second order ordered marginals, net of the zero and first order marginals and the second order unordered marginals. 

More generally, the indices of orthogonal subspaces $W_\gamma$ in \eqref{eq:isotypic} have a natural partial ordering, referred to as  \emph{dominance ordering}. The shape $\nu=[\nu_1,\nu_2,\ldots,\nu_\l]$ strictly dominates shape $\gamma = [\gamma_1,\gamma_2,\ldots,\gamma_{\l^\prime}] \neq \nu$, denoted by $\nu \vartriangleright \gamma$, if
$\sum_{i=1}^j \nu_i \geq \sum_{i=1}^j \gamma_i$ for each $j=1,2,\ldots,\max\{\l,\l^{\prime}\}$. For example, with $n=6$,
$\begin{array}{c}{\begin{tikzpicture}[scale=.15,line width=1.0pt] 
\draw (0,0) rectangle (1,1); \draw (1,0) rectangle (2,1); \draw (2,0) rectangle (3,1); \draw (3,0) rectangle (4,1);
\draw (0,-1) rectangle (1,0); \draw (1,-1) rectangle (2,0); 
\end{tikzpicture}}\end{array}$ dominates 
$\begin{array}{c}{\begin{tikzpicture}[scale=.15,line width=1.0pt] 
\draw (0,0) rectangle (1,1); \draw (1,0) rectangle (2,1); \draw (2,0) rectangle (3,1); \draw (3,0) rectangle (4,1);
\draw (0,-1) rectangle (1,0); 
\draw (0,-2) rectangle (1,-1); 
\end{tikzpicture}}\end{array}$
and 
$\begin{array}{c}{\begin{tikzpicture}[scale=.15,line width=1.0pt] 
\draw (0,0) rectangle (1,1); \draw (1,0) rectangle (2,1); \draw (2,0) rectangle (3,1); 
\draw (0,-1) rectangle (1,0); \draw (1,-1) rectangle (2,0); \draw (2,-1) rectangle (3,0); 
\end{tikzpicture}}\end{array}$,
but the latter two shapes are incomparable (neither one dominates the other). The projection of a signal onto an isotypic component $W_\gamma$ captures information about the marginals corresponding to shape $\gamma$, net of the marginals corresponding to all shapes strictly preceding $\gamma$ in dominance order.

With this justification, returning to the APA example in Fig.  \ref{Fig:isodecomp}, Diaconis \cite[Sec. 2C]{diaconis1989generalization} interprets the relatively large contribution on isotypic component $W_{\begin{tikzpicture}[scale=.15,line width=1.0pt] 
\draw (0,0) rectangle (1,1); \draw (1,0) rectangle (2,1); \draw (2,0) rectangle (3,1); 
\draw (0,-1) rectangle (1,0); \draw (1,-1) rectangle (2,0); 
\end{tikzpicture}}$, corresponding to the two-row shape $[3,2]$, as 
the contribution of  unordered pair (second order) effects, and he argues
that this is related to the fact that the APA divides into two groups: academics and clinicians. This Fourier analysis approach has also found application in ``Q-sort" data in psychology \cite[5B]{Diaconis-book}, balanced incomplete block designs \cite{rockmore1997some},
multiple object tracking \cite{kondor2007multi}, finding graph invariants \cite{kondor2008skew}, the quadratic assignment problem \cite{kondor2010fourier}, computational geometry \cite{jiang2014fourier}, genomic data analysis \cite{uminsky2019detecting}, and sports analytics \cite{devlin2020identifying}.

\begin{wrapfigure}[15]{r}{0.45\textwidth}
\vspace{-.35in}
\footnotesize
\begin{minipage}{.95\linewidth}
\centering
\begin{minipage}{.48\linewidth}
\centering
\centerline{\small{~~~${\bf f^a}$}} 
\vspace{.02in}
\includegraphics[width=1\linewidth,page=1]{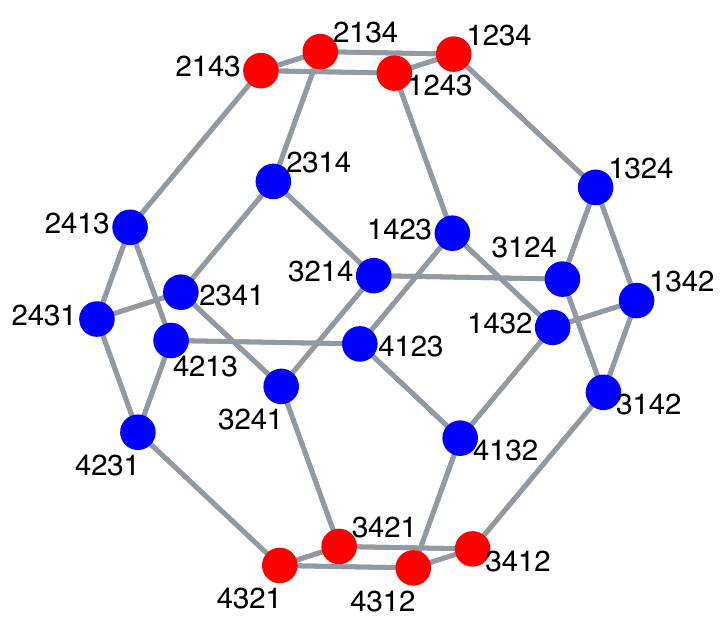}
\end{minipage}
\begin{minipage}{.48\linewidth}
\centering
\centerline{\small{~~~${\bf f^b}$}} 
\vspace{.02in}
\includegraphics[width=\linewidth,page=2]{figures/same_iso_decomp} 
\end{minipage}
\\
\vspace{.18in}
\begin{tabular}{|c|c|c|c|c|c|}
\hline
      $\gamma$ & $\begin{tikzpicture}[scale=.15,line width=1.0pt] 
\draw (0,0) rectangle (1,1); \draw (1,0) rectangle (2,1);  \draw (2,0) rectangle (3,1); \draw (3,0) rectangle (4,1);
\end{tikzpicture}$
      & 
      $\begin{tikzpicture}[scale=.15,line width=1.0pt] 
\draw (0,0) rectangle (1,1); \draw (1,0) rectangle (2,1);  \draw (2,0) rectangle (3,1); 
\draw (0,-1) rectangle (1,0); 
\end{tikzpicture}$
 & $\begin{tikzpicture}[scale=.15,line width=1.0pt] 
\draw (0,0) rectangle (1,1); \draw (1,0) rectangle (2,1); 
\draw (0,-1) rectangle (1,0); \draw (1,-1) rectangle (2,0); 
\end{tikzpicture}$
&
 $\begin{tikzpicture}[scale=.15,line width=1.0pt] 
\draw (0,0) rectangle (1,1); \draw (1,0) rectangle (2,1); 
\draw (0,-1) rectangle (1,0); 
\draw (0,-2) rectangle (1,0); 
\end{tikzpicture}$
 & 
 $\begin{tikzpicture}[scale=.15,line width=1.0pt] 
\draw (0,0) rectangle (1,1); 
\draw (0,-1) rectangle (1,0); 
\draw (0,-2) rectangle (1,0); 
\draw (0,-3) rectangle (1,0); 
\end{tikzpicture}$ \\
      \hline
      $\lVert{\bf f}_{\gamma}\rVert^2$ & 2.67 & 0 & 5.33 & 0 & 0\\
      \hline
    \end{tabular}
\end{minipage}
\begin{minipage}{.03\linewidth}
\centering
\includegraphics[height=1in,page=3]{figures/same_iso_decomp}
\vspace{.35in}
\end{minipage}
\caption{Two signals on $\PP_4$ with the same energy decomposition into isotypic components; i.e., $\lVert{\bf f^a}_\gamma \rVert=\lVert{\bf f^b}_\gamma \rVert:=\lVert{\bf f}_\gamma \rVert$ for all $\gamma$.}\label{Fig:same_iso}
\end{wrapfigure}
\emph{What are the limitations of this approach?} As shown in Fig. \ref{Fig:same_iso}, two signals with different structure and support can have exactly the same energy decomposition into isotypic components, limiting the amount of information that can be extracted from this energy decomposition alone. 

A more refined approach is to use the Fourier transform of ${\bf f}$ on the symmetric group, which is defined as the set of matrices $\{\hat f(\gamma) \mid \gamma \text{ a partition of $n$}\}$, where  for each integer partition $\gamma$, $\hat f(\gamma)$ is defined as the sum $\sum_{\sigma \in \SS_n} f(\sigma) \rho_\gamma(\sigma)$ \cite{Diaconis-book}. Here,
$\rho_\gamma(\sigma)$ is the matrix of the permutation $\sigma$ as a linear transformation on the $\SS_n$-invariant subspace $V_{\gamma,i}$ (for any $i$). This set $\{\hat f(\gamma) \mid \gamma \text{ a partition of $n$}\}$ has a total of $n!$ matrix entries, which are viewed as the Fourier coefficients of ${\bf f}$.  Unfortunately, there is no natural choice of basis of the irreducible components $V_{\gamma,i}$ (see for example Diaconis \cite[p.~955]{diaconis1989generalization}), and therefore  ad-hoc methods are used to interpret these values. The standard choices are Young's seminormal basis or Young's orthogonal basis \cite[8A]{Diaconis-book},  \cite{clausen1993fast}, \cite{diaconis1993efficient}, \cite{Kondor-multiresolution}, which are used largely because they are well-adapted to fast computation, but they lack interpretability.

\newcommand{\shapefour}{$W_{\begin{tikzpicture}[scale=.15,line width=1.0pt] 
\draw (0,0) rectangle (1,1); \draw (1,0) rectangle (2,1);   \draw (2,0) rectangle (3,1); 
\draw (3,0) rectangle (4,1);  
\end{tikzpicture}}$}
\newcommand{\shapethreeone}{$W_{\begin{tikzpicture}[scale=.15,line width=1.0pt] 
\draw (0,0) rectangle (1,1); \draw (1,0) rectangle (2,1);   \draw (2,0) rectangle (3,1); 
\draw (0,-1) rectangle (1,0);  
\end{tikzpicture}}$}
\newcommand{\shapetwotwo}{$W_{\begin{tikzpicture}[scale=.15,line width=1.0pt] 
\draw (0,0) rectangle (1,1); \draw (1,0) rectangle (2,1);   
\draw (0,-1) rectangle (1,0);  \draw (1,-1) rectangle (2,0); 
\end{tikzpicture}}$}

\begin{figure}[b]
    \vspace{-0.4in}
\hrule height 1.5pt
\vspace{.08in}
\footnotesize
\hfill
\begin{minipage}{.42\linewidth}
\begin{minipage}{.47\linewidth}
\centering
\centerline{\small{~${\boldsymbol \delta}_{\{3,4\},\{1,2\}}$}} 
\vspace{.1in}
\includegraphics[width=1\linewidth,page=1]{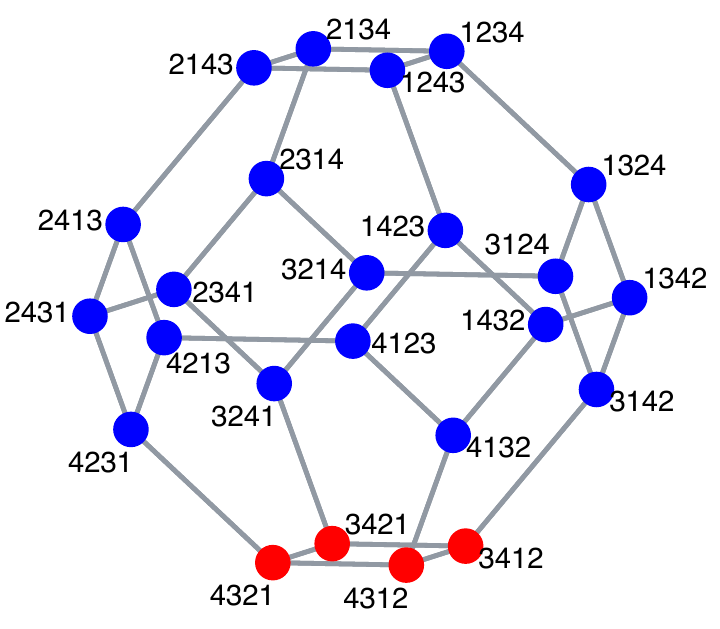}
\end{minipage}
\begin{minipage}{.47\linewidth}
\centering
\centerline{\small{~${\boldsymbol \delta}_{\{1,2\},\{2,3\}}$}} 
\vspace{.1in}
\includegraphics[width=\linewidth,page=2]{figures/delta22} 
\end{minipage}
\begin{minipage}{.035\linewidth}
\centering
\vspace{.2in}
\includegraphics[height=1.1in,page=3]{figures/same_iso_decomp}
\end{minipage} %\vspace{.15in} 
\centerline{\small{(a)~~~}}
\end{minipage}
\hfill
\hfill
\begin{minipage}{.42\linewidth}
\begin{minipage}{.47\linewidth}
\centering
\centerline{\small{~${\boldsymbol \delta}_{{\{3,4\},\{1,2\}}_{\begin{tikzpicture}[scale=.15,line width=1.0pt] 
\draw (0,0) rectangle (1,1); \draw (1,0) rectangle (2,1); 
\draw (0,-1) rectangle (1,0); \draw (1,-1) rectangle (2,0); 
\end{tikzpicture}}}$}} 
\vspace{.02in}
\includegraphics[width=1\linewidth,page=3]{figures/delta22}
\end{minipage}
\begin{minipage}{.47\linewidth}
\centering
\centerline{\small{~${\boldsymbol \delta}_{{\{1,2\},\{2,3\}}_{\begin{tikzpicture}[scale=.15,line width=1.0pt] 
\draw (0,0) rectangle (1,1); \draw (1,0) rectangle (2,1); 
\draw (0,-1) rectangle (1,0); \draw (1,-1) rectangle (2,0); 
\end{tikzpicture}}}$}} 
\vspace{.02in}
\includegraphics[width=\linewidth,page=4]{figures/delta22} 
\end{minipage}
\begin{minipage}{.035\linewidth}
\centering
\vspace{.25in}
\includegraphics[height=1.1in,page=5]{figures/delta22}
\end{minipage}
\centerline{\small{(b)~~~}}
\end{minipage}
\hfill
\hfill
\caption{Examples of Mallows' interpretable second order (unordered) functions \cite[Sec. 2C]{diaconis1989generalization}. (a) Two (of the thirty-six) functions of the form ${\boldsymbol \delta}_{{\{i,i^\prime\},\{j,j^\prime\}}}$, that are equal to 1 on the vertices where candidates $i$ and $i^{\prime}$ are in ranking positions $j$ and $j^\prime$ (in either order), and equal to 0 elsewhere. (b) The orthogonal projections of these interpretable 
functions onto the isotypic component {\protect\shapetwotwo}. Diaconis \cite[Sec. 2C]{diaconis1989generalization} suggests to take the inner product between each of these projected interpretable functions and the signal to identify second order effects in the data, net of the zero order and first order effects found in {\protect\shapefour} and {\protect\shapethreeone}.}\label{Fig:second_order}
\end{figure}

Alternatively, \cite[Sec. 2C]{diaconis1989generalization} applies Mallow's method: construct an overcomplete spanning set for $W_\gamma$ by projecting interpretable functions  
that capture $k$th order effects ($k=\sum_{i=2}^\ell \gamma_i$) onto $W_\gamma$, and then take inner products between these projections and the signal. In Fig. \ref{Fig:second_order}, we show two of the thirty-six interpretable second order functions and their projections in the overcomplete spanning set for $W_{\begin{tikzpicture}[scale=.15,line width=1.0pt] 
\draw (0,0) rectangle (1,1); \draw (1,0) rectangle (2,1); 
\draw (0,-1) rectangle (1,0); \draw (1,-1) rectangle (2,0); 
\end{tikzpicture}}
$. 
In general, there are $m_\gamma^2$ of these spanning vectors for the $d_\gamma^2$-dimensional space $W_\gamma$, where $m_\gamma=\frac{n!}{\prod_{l=1}^{\ell} \gamma_l !}$.

For this particular isotypic component $W_{\begin{tikzpicture}[scale=.15,line width=1.0pt] 
\draw (0,0) rectangle (1,1); \draw (1,0) rectangle (2,1); 
\draw (0,-1) rectangle (1,0); \draw (1,-1) rectangle (2,0); 
\end{tikzpicture}}
$, nine projected functions each appear four times in the thirty-six vectors; for example,
\begin{align}\label{Eq:fouratoms}
{\boldsymbol \delta}_{{\{3,4\},\{1,2\}}_{\begin{tikzpicture}[scale=.15,line width=1.0pt] 
\draw (0,0) rectangle (1,1); \draw (1,0) rectangle (2,1); 
\draw (0,-1) rectangle (1,0); \draw (1,-1) rectangle (2,0); 
\end{tikzpicture}}}= %&=
{\boldsymbol \delta}_{{\{1,2\},\{3,4\}}_{\begin{tikzpicture}[scale=.15,line width=1.0pt] 
\draw (0,0) rectangle (1,1); \draw (1,0) rectangle (2,1); 
\draw (0,-1) rectangle (1,0); \draw (1,-1) rectangle (2,0); 
\end{tikzpicture}}} %\nonumber \\
%&
={\boldsymbol \delta}_{{\{3,4\},\{3,4\}}_{\begin{tikzpicture}[scale=.15,line width=1.0pt] 
\draw (0,0) rectangle (1,1); \draw (1,0) rectangle (2,1); 
\draw (0,-1) rectangle (1,0); \draw (1,-1) rectangle (2,0); 
\end{tikzpicture}}}=
{\boldsymbol \delta}_{{\{1,2\},\{1,2\}}_{\begin{tikzpicture}[scale=.15,line width=1.0pt] 
\draw (0,0) rectangle (1,1); \draw (1,0) rectangle (2,1); 
\draw (0,-1) rectangle (1,0); \draw (1,-1) rectangle (2,0); 
\end{tikzpicture}}}.
\end{align}

For the 2017 Minneapolis City Council Ward 3 election data ${\bf g}$ shown in Fig. \ref{Fig:signals}, the largest inner products
\begin{align}\label{Eq:ip22}
\langle {\bf g},{\boldsymbol \delta}_{{\{i,i^\prime\},\{j,j^\prime\}}_{\begin{tikzpicture}[scale=.15,line width=1.0pt] 
\draw (0,0) rectangle (1,1); \draw (1,0) rectangle (2,1); 
\draw (0,-1) rectangle (1,0); \draw (1,-1) rectangle (2,0); 
\end{tikzpicture}}}\rangle=
\langle {\bf g}_{\begin{tikzpicture}[scale=.15,line width=1.0pt] 
\draw (0,0) rectangle (1,1); \draw (1,0) rectangle (2,1); 
\draw (0,-1) rectangle (1,0); \draw (1,-1) rectangle (2,0); 
\end{tikzpicture}},{\boldsymbol \delta}_{{\{i,i^\prime\},\{j,j^\prime\}}} \rangle
\end{align}
are the ones between the signal and ${\boldsymbol \delta}_{{\{3,4\},\{1,2\}}_{\begin{tikzpicture}[scale=.15,line width=1.0pt] 
\draw (0,0) rectangle (1,1); \draw (1,0) rectangle (2,1); 
\draw (0,-1) rectangle (1,0); \draw (1,-1) rectangle (2,0); 
\end{tikzpicture}}}$ and the three identical projections listed in \eqref{Eq:fouratoms}. This can be seen visually by thinking about the inner products in the form listed in the right-hand side of \eqref{Eq:ip22} and then examining ${\bf g}_{\begin{tikzpicture}[scale=.15,line width=1.0pt] 
\draw (0,0) rectangle (1,1); \draw (1,0) rectangle (2,1); 
\draw (0,-1) rectangle (1,0); \draw (1,-1) rectangle (2,0); 
\end{tikzpicture}}$ in Fig. \ref{Fig:isodecomp}. Again, an interpretation of these inner products is that net of zero and first order marginal effects (i.e., starting from ${\bf g} - 
{\bf g}_{\begin{tikzpicture}[scale=.15,line width=1.0pt] 
\draw (0,0) rectangle (1,1); \draw (1,0) rectangle (2,1); \draw (2,0) rectangle (3,1); \draw (3,0) rectangle (4,1); 
\end{tikzpicture}}
-{\bf g}_{\begin{tikzpicture}[scale=.15,line width=1.0pt] 
\draw (0,0) rectangle (1,1); \draw (1,0) rectangle (2,1); \draw (2,0) rectangle (3,1); 
\draw (0,-1) rectangle (1,0); 
\end{tikzpicture}}$), the pairs of candidates $\{1,2\}$ and $\{3,4\}$ are likely to appear together in ranking positions $\{1,2\}$ or $\{3,4\}$.
See \cite{uminsky2019detecting} for an additional application of this method to genomic data. 

The new approach we propose in Sec. \ref{Se:proposed} also yields an overcomplete spanning set, but we  directly construct each spanning vector as an interpretable function in $V_{\gamma,i}$ (and in fact in a more precisely defined eigenspace that is a strict subset of $V_{\gamma,i}$).

Finally, it is noteworthy that the non-commutative Fourier analysis approach does not make any direct use of the permutahedron or any other underlying graph structure for that matter. Rather, it is completely independent of the choice of generating set of the symmetric group. We return to this point at the end of Sec. \ref{Se:gsp}.

\subsection{The Graph Signal Processing Approach} \label{Se:gsp}
Within the last ten years, researchers in the field of graph signal processing \cite{shuman_emerging_SPM_2013,ortega2018graph} have developed new methods to identify and exploit structure in data residing on the vertices of weighted or unweighted graphs. Signals on the unweighted permutahedron, as shown in Fig. \ref{Fig:signals}, fit into this framework. While, to our knowledge, such ranked data on the permutahedron have not been analyzed with graph signal processing techniques, the natural first method to apply would be the \emph{graph Fourier transform}.  The graph Laplacian matrix is defined as $\L:={\bf D}-{\bf A}$, where ${\bf A}$ is the adjacency matrix of the permutahedron and the diagonal degree matrix ${\bf D}$ is equal to $(n-1){\bf I}_{n!}$, where ${\bf I}_k$ is a $k \times k$ identity matrix, since the permutahedron is an $(n-1)$-regular graph. The matrix $\L$ is symmetric with nonnegative eigenvalues, and each eigenvalue $\lambda$ is associated with an orthogonal eigenspace $U_\lambda$. 
The signal space $\Rbb[\SS_n]$ decomposes into a direct sum of these orthogonal eigenspaces, 
\begin{equation}\label{eq:eigendecomp}
\Rbb[\SS_n] \cong \bigoplus_\lambda U_\lambda. %,
\end{equation}
\vspace{-.2in}

\begin{wrapfigure}[9]{r}{.44\textwidth}
\vspace{-.37in}
\begin{minipage}{.48\linewidth}
\centering
{\includegraphics[width=1\linewidth,page=2]{figures/gft}}
\end{minipage}
\begin{minipage}{.48\linewidth}
\centering
{\includegraphics[width=1\linewidth,page=1]{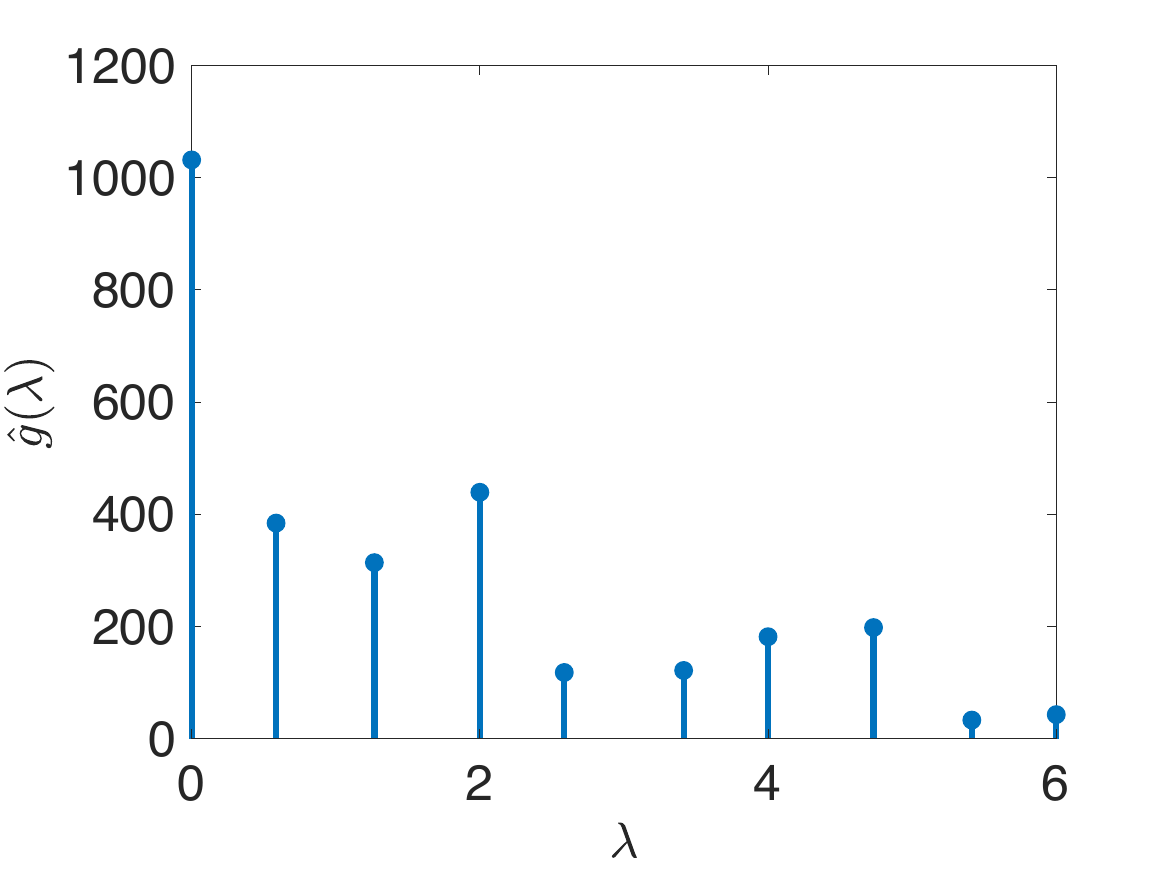}}
\end{minipage}
\vspace{-.1in}
\caption{Graph Fourier transforms of the
two signals from Fig. \ref{Fig:signals}. 
%Both signals have 
For each, more energy is concentrated on the smoother eigenvectors associated with the lower Laplacian eigenvalues of the corresponding permutahedron.}\label{Fig:gft}
\end{wrapfigure}
A common definition of the graph Fourier transform is$\hat{f}(\lambda_{\l}):=|\langle {\bf f},{\bf u}_{\l} \rangle |$, where ${\bf u}_{\l}$ is the eigenvector associated with the $\l$th eigenvalue of $\L$ \cite{shuman_emerging_SPM_2013}.  Since this definition depends on the particular choice of the Laplacian eigenvectors in the case of repeated eigenvalues (which occur in $\PP_n$), we define the graph Fourier transform here as $\hat{f}(\lambda):=\lVert{\bf f}_\lambda \rVert$, where ${\bf f}_\lambda$ is the orthogonal projection of ${\bf f}$ onto the eigenspace $U_\lambda$. In the absence of repeated eigenvalues, these two definitions coincide.

\emph{How can the graph Fourier transform and other graph signal processing techniques be used to find structure in ranked data?} For each unit norm Laplacian eigenvector ${\bf u}_\lambda$ associated with eigenvalue $\lambda$, we have 
\begin{align}\label{Eq:smoothness}
\lambda={\bf u}_\lambda^{\top}{\L}{\bf u}_\lambda=\sum_{(i,j)\in {\cal E}}[u_\lambda(i)-u_\lambda(j)]^2, 
\end{align}
where ${\cal E}$ are the edges of the permutahedron. Therefore, the eigenvectors associated with lower Laplacian eigenvalues are \emph{smoother} in the sense that the values vary less across neighboring vertices, as shown in Fig. \ref{Fig:ortho_basis}. The graph Fourier transform provides a decomposition of the energy of the signal into the energy in each Laplacian eigenspace, 
yielding information about the signal's smoothness, as shown in Fig. \ref{Fig:gft}.

\begin{figure}[t!]
\vspace{.05in}
\centering
\begin{minipage}{.42\linewidth}
\centering
\includegraphics[width=.45\linewidth,page=1]{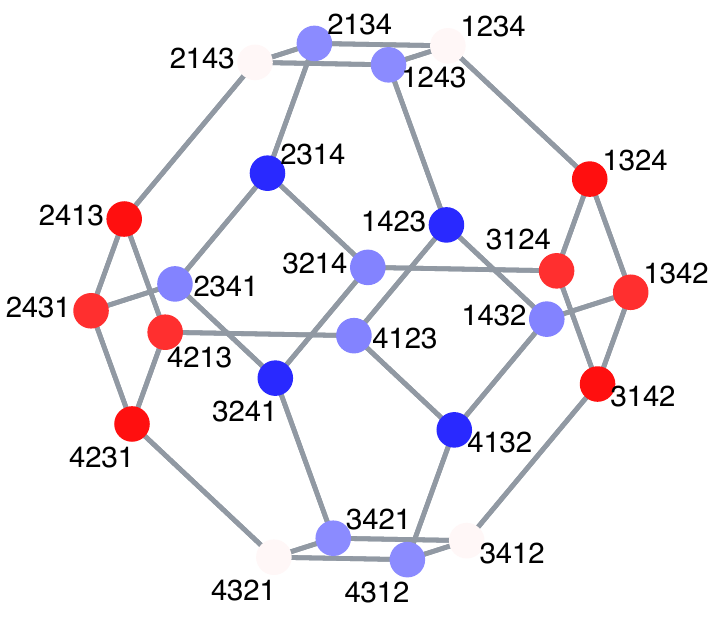}
\includegraphics[width=.45\linewidth,page=2]{figures/evecs_v2} \\
\vspace{-.18in}
${\small \underbrace{\hspace{2in}}_{\hbox{Orthornomal basis for }{U}_{1.2679}}}$
\end{minipage}
\begin{minipage}{.42\linewidth}
\centering
\includegraphics[width=.45\linewidth,page=3]{figures/evecs_v2}
\includegraphics[width=.45\linewidth,page=4]{figures/evecs_v2} \\
\vspace{-.18in}
${\small \underbrace{\hspace{2in}}_{\hbox{Orthornomal basis for }{U}_{4.7321}}}$
\end{minipage}
\begin{minipage}{.08\linewidth}
\centering
\includegraphics[height=1in]{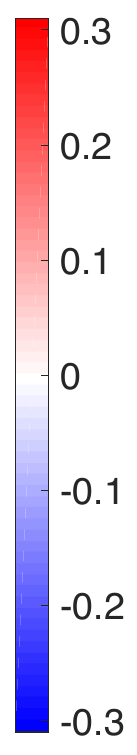}
\vspace{.2in}
\end{minipage}
\caption{Orthonormal bases for the graph Laplacian eigenspaces of $\PP_4$ associated with the eigenvalues $\lambda=1.2679$ and $\lambda=4.7321$. Both of these eigenspaces are two-dimensional subspaces of the 24-dimensional signal space. Note that the Laplacian eigenvectors associated with the lower eigenvalue are smoother in the sense that the values vary less across neighboring vertices (see, e.g., the jumps in value from 1234 to 1243 in the third vector above or from 1324 to neighboring 1342 in the fourth vector).}\label{Fig:ortho_basis}
\vspace{0.1in}
\hrule height 1.5pt
\vspace{-.2in}
\end{figure}

\emph{What are the limitations of applying the standard graph Fourier transform to signals on the permutahedron?} First, because the graph Laplacian eigenvalues and eigenvectors of the permutahedron are not known in closed form, 
it is computationally intensive to compute the graph Fourier transform (${\cal O}(n!^3)$), and not tractable for ranked data with more than seven or eight candidates. Second, there is not a natural orthonormal basis for each eigenspace that preserves the structure and symmetry of the graph. 
All but the first and last Laplacian eigenvalues of the permutahedron are repeated multiple times, and thus there are infinitely many choices of orthornomal bases for the associated eigenspaces $\{U_\lambda\}_{\lambda \notin \{0,2(n-1)\}}$.
As shown in Fig. \ref{Fig:ortho_basis} for two Laplacian eigenspaces of $\PP_4$, the numerical computation of a basis is not guaranteed to preserve any sort of symmetry, leading to less interpretable basis vectors. Moreover, and third, different isotypic components may contain the same Laplacian eigenvalue, and thus, it is not even guaranteed that the numerically computed basis vectors live in a single isotypic component. 

\begin{wrapfigure}[13]{r}{0.3\textwidth}
\vspace{-.34in}
\centerline{\includegraphics[width=.9\linewidth]{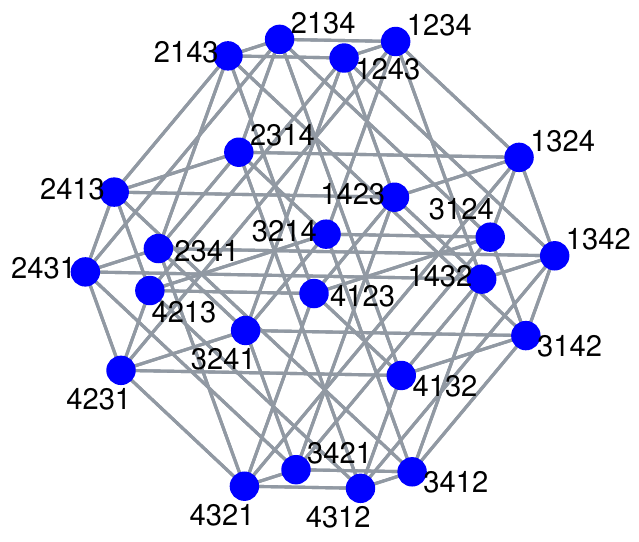}}
\caption{The Cayley graph $\Upgamma_n$ on the symmetric group induced by the generating set all transpositions, for $n=4$.}\label{Fig:all_transp_network}
\end{wrapfigure}
\emph{Why is the permutahedron the right graph to represent the underlying data domain?}
An alternative choice is the Cayley graph of the symmetric group induced by the generating set of all transpositions (not just neighboring transpositions), which is shown in Fig. \ref{Fig:all_transp_network} and which we denote by $\Upgamma_n$. This graph has some nice mathematical properties: (i) the isotypic components $W_\gamma$ are each spanned by eigenvectors associated with a single  Laplacian eigenvalue of the Cayley graph $\Upgamma_n$ \cite[Prop. 1]{kondor2007multi};  (ii) the Laplacian eigenvalues and eigenvectors of $\Upgamma_n$ are known in closed form \cite[Thm. 1.1]{rockmore2002fast}, \cite[Prop. 2]{kondor2007multi}, \cite[Thm. III.1]{mahya_sampta}; and, (iii) moreover, the Laplacian eigenvalues to which the isotypic components correspond increase according to dominance ordering.
So if $\nu \vartriangleright \gamma$, then any vector ${\bf f}_\nu$ in the isotypic component $W_\nu$ is smoother  with respect to the Cayley graph $\Upgamma_n$ (recall the definition of smoothness in \eqref{Eq:smoothness}) than any vector ${\bf f}_\gamma$ in $W_\gamma$ with the same norm as ${\bf f}_\nu$, since 
$ \frac{{\bf f}_\nu^{\top} \L_{\Upgamma_n} {\bf f}_\nu}{{\bf f}_\nu^{\top}{\bf f}_\nu}=\lambda_{\nu}<\lambda_\gamma= \frac{{\bf f}_\gamma^{\top} \L_{\Upgamma_n} {\bf f}_\gamma}{{\bf f}_\gamma^{\top}{\bf f}_\gamma}$. In this sense, the isotypic components provide a notion of frequency, with vectors residing in isotypic components later in the dominance ordering representing ``more complex'' (less smooth) functions, \emph{with respect to the Cayley graph $\Upgamma_n$ induced by the generating set of all transpositions} \cite{kondor2007multi}.

Despite these nice mathematical properties, the permutahedron is the more appropriate domain on which to develop new techniques for analyzing the structure of most ranked data sets, due to the different notions of distance the two underlying graphs capture \cite{kondor2010ranking}. The structure of %the Cayley graph 
$\Upgamma_n$ captures an appropriate notion of distance between the permutations in applications such as multi-object tracking, where, e.g., the object trajectories (slots) are continuously visible on radar or camera, but the object identities (candidates) associated with each trajectory are only revealed at certain time instances (e.g., pilot reports by radio, observations captured by a security camera) \cite{kondor2007multi}. In this situation, there is not necessarily a physically-meaningful \begin{wrapfigure}[23]{r}{.44\textwidth}
\begin{minipage}{.05\linewidth}
\vspace{.1in}
\centerline{\scriptsize{\bf \rotatebox{90}{Permutahedron $\PP_4$}}} 
\end{minipage}
\begin{minipage}{.45\linewidth}
\centering
\centerline{\scriptsize{~~~${\bf f^a}$}} 
\vspace{.02in}
{\includegraphics[width=1.05\linewidth,page=1]{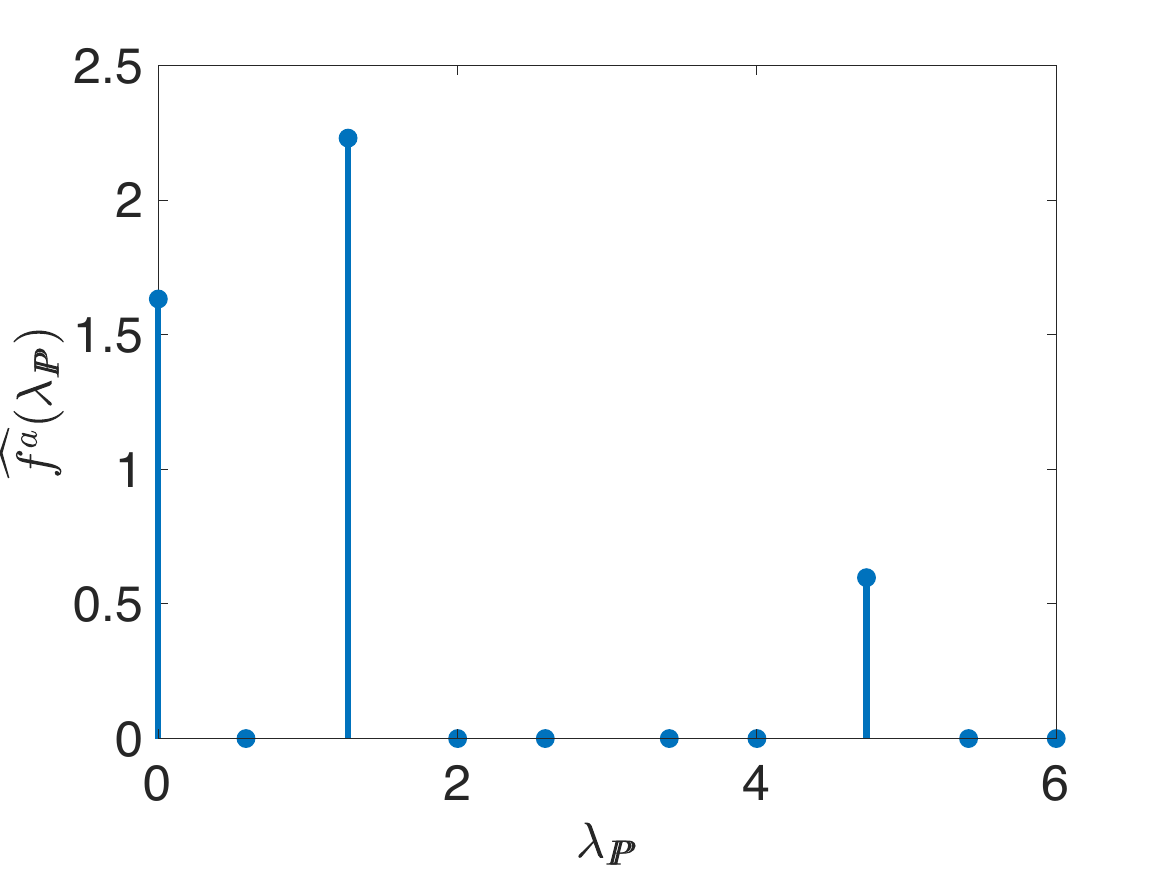}}
\end{minipage}
\begin{minipage}{.45\linewidth}
\centering
\vspace{-.016in}
\centerline{\scriptsize{~~~~${\bf f^b}$}}
\vspace{.02in}
{\includegraphics[width=1.05\linewidth,page=2]{figures/same_iso_gft}}
\end{minipage} \vspace{.1in} \\
\begin{minipage}{.05\linewidth}
\vspace{-.06in}
\centerline{\scriptsize{\bf \rotatebox{90}{Cayley graph $\Upgamma_4$}}} 
\end{minipage}
\begin{minipage}{.45\linewidth}
\centering
{\includegraphics[width=1.05\linewidth,page=3]{figures/same_iso_gft}}
\end{minipage}
\begin{minipage}{.45\linewidth}
\centering
{\includegraphics[width=1.05\linewidth,page=4]{figures/same_iso_gft}}
\end{minipage}
\caption{Graph Fourier transforms of the
two signals from Fig. \ref{Fig:same_iso}, on Cayley graphs of the symmetric group induced by two different choices of the generating set. On the Cayley graph $\Upgamma_4$ of Fig. \ref{Fig:all_transp_network}, the graph Fourier transforms of the two signals ${\bf f^a}$ and ${\bf f^b}$ are the same, as they have the same energy decomposition across isotypic components (c.f., Fig. \ref{Fig:same_iso}). On the other hand, the graph Fourier transforms in the top row show that ${\bf f^a}$ is smoother than ${\bf f^b}$ with respect to the permutahedron $\PP_4$.}\label{Fig:same_iso_gft}
\end{wrapfigure}
linear order to the trajectories (slots), and therefore it may be likely that objects jump from one trajectory to a crossing trajectory whose label is not adjacent or even similar. This information is typically captured by a corresponding noise model 
(e.g., \cite[Sec. 3.1]{kondor2007multi}). However, in most ranked data applications, the ranking 
positions represent a linear ordering,
and therefore permutations that swap the candidates in the first and last ranking positions, e.g., are not close from a voter's viewpoint. To illustrate this with a specific example, 1234 and 4231 are adjacent in the Cayley graph of Fig. \ref{Fig:all_transp_network}, but they are far apart from a voter's perspective 
and they are far apart in the permutahedron $\PP_4$, whereas 1234 and 2134 are adjacent in both graphs. 
Kondor \cite[Sec. 3]{kondor2010ranking} distinguishes these cases in terms of invariance, with our notion being only right-invariant and the other notion being bi-invariant (right-invariant and left-invariant). 
As shown in Fig. \ref{Fig:same_iso_gft}, the graph structures induced by these two different distance metrics (generating sets) yield different 
notions of signal smoothness, 
as captured by the respective graph Fourier transforms.
To summarize the key takeaways, when we choose an underlying graph data domain, we are defining a notion of distance between permutations; and the distance induced by the permutahedron structure is most appropriate in ranking applications where permutations should be considered closest if the candidate swap occurs across neighboring ranking slots.

\vspace{-.1in}

\subsection{Other Related Transforms for Ranked Data}

We briefly mention other related linear transforms for ranked data. Nested orthogonal contrasts \cite{marden1992use,marden2014analyzing} compare the rankings of two or more groups of candidates, ignoring the relative ranks within each group of candidates. Inversions \cite{mccullagh1993permutations,mccullagh1993matched,marden2014analyzing,grossman2009inversions} project the data onto linear subspaces based on the relative rankings of subsets (pairs, triplets, etc.) of the candidates, net of the effects of lower order subsets.

Other ranked data transforms consider the underlying graph to be a quasi-Abelian Cayley graph; i.e., the set of generators of the Cayley graph is the union of conjugacy classes. Rockmore \cite{rockmore2002fast} investigates fast Fourier transforms for data on quasi-Abelian Cayley graphs.  Ghandehari et al. \cite{mahya_sampta} extend Rockmore's work by developing tight windowed Fourier frames for data residing on quasi-Abelian Cayley graphs. Since the full set of transpositions is a conjugacy class in $\SS_n$, the Cayley graph $\Upgamma_n$ induced by the generating set of all transpositions is quasi-Abelian, but the permutahedron $\PP_n$ does not fall into this class (except for $\PP_2$). Kondor's left-invariant coset-based multiresolution analysis \cite{Kondor-multiresolution} yields an orthogonal basis of wavelet and scaling atoms that are localized in the vertex and spectral domains of $\Upgamma_n$, but not necessarily in either domain when the underlying graph is taken to be the permutahedron $\PP_n$. 

Finally, taking inspiration more from one-dimensional signal processing than from the literature on signal processing on graphs, Kakarala \cite{kakarala2011signal} further decomposes the coefficient matrices $\hat{f}(\gamma)$ of the Fourier transform on the symmetric group into the product of a positive semidefinite ``magnitude'' matrix and an orthogonal ``phase'' matrix.

%%%%%%%%%%%%%%%%%%%%%%%%%%%%%%%%%%%%%%%%%%%%%%%%%%%%%%%%%%%%%%%%%%%%%%%%
\vspace{-.1in}

\section{Tight Spectral Frames for Ranked Data}\label{Se:proposed}
\vspace{-.1in}

In this section, we present a new approach to generate dictionaries for ranked data by combining the symmetry decomposition method \eqref{eq:irreducible_decomposition} from combinatorial representation theory with the spectral graph decomposition method \eqref{eq:eigendecomp} from graph signal processing, in order to capture two different kinds of information about the data. 

These two approaches are connected in the following way, which is at the crux of our method. The graph Laplacian matrix is in the regular representation of the symmetric group algebra, because it can be written as the linear combination 
\begin{equation}\label{eq:Laplacian_representation}
\L = (n-1) \rho_R({\bf 1}) - \sum_{s \in S} \rho_R(s), 
\end{equation}
where $\rho_R(s)$ is the matrix of $s \in S$ acting as a linear transformation on the right of $\Rbb[\SS_n]$ and ${\bf 1} \in \SS_n$ is the identity element so that $\rho_R({\bf 1}) = \mathbf{I}_{n!}$. Since our generators are involutions ($s^2 = 1$), the Laplacian is symmetric, and \eqref{eq:Laplacian_representation} is also true if the right regular representation is replaced with the left regular representation and $\L$ acts on the left. As a result of \eqref{eq:Laplacian_representation}, the isotypic components in \eqref{eq:irreducible_decomposition} decompose into Laplacian eigenspaces, and therefore the group algebra decomposes as follows.
\begin{proposition}\label{prop:decomp_eigenvalues_by_shape}
\vspace{-.1in}
\begin{equation}\label{eq:decomp_Lap_iso}
\Rbb[\SS_n] \cong \bigoplus_{\gamma \vdash n} \bigoplus_{\lambda \in \Lambda_\gamma} Z_{\gamma, \lambda}, \qquad\hbox{where}\quad Z_{\gamma, \lambda} = W_\gamma \cap U_\lambda,
 \end{equation}
and $\Lambda_\gamma$ is the set of Laplacian eigenvectors $\lambda$ such that $W_\gamma \cap U_\lambda \not= \emptyset$.
\end{proposition}
\begin{proof} The group algebra decomposes into a direct sum of vector spaces in two ways,
$$
\Rbb[\SS_n] \cong \bigoplus_{\lambda \in \Lambda} U_{\lambda} \qquad \hbox{and} \qquad \Rbb[\SS_n] \cong \bigoplus_{\gamma \vdash n} W_{\gamma}.
$$
On the left we use the fact that $\L$ is symmetric to decompose $\Rbb[\SS_n]$  into a direct sum of Laplacian eigenspaces $U_\lambda, \lambda \in \Lambda$, where $\Lambda$ is the set of eigenvalues of $\L$. On the right is the decomposition \eqref{eq:irreducible_decomposition} into isotypic components for $\SS_n$.  The isotypic components are closed under both the left and right action of $\SS_n$ and $\L$ is a linear combination  \eqref{eq:Laplacian_representation} of group elements, so the isotypic components $W_\gamma$ are closed under the action of $\L$. 

If ${\bf f} \in \Rbb[\SS_n]$, then there is a unique decomposition ${\bf f} = \sum_{\lambda \in \Lambda} {\bf f}_\lambda$ with ${\bf f}_\lambda \in U_\lambda$; namely, ${\bf f}_\lambda$ is the orthogonal projection of ${\bf f}$ onto $U_\lambda$.  For each $\lambda \in \Lambda$, there is a unique decomposition, ${\bf f}_\lambda = \sum_{\{\gamma \vdash n: \lambda \in \Lambda_\gamma\}} {\bf f}_{\gamma,\lambda}$, of ${\bf f}_\lambda$ into vectors $\{{\bf f}_{\gamma,\lambda}\}$ in the corresponding isotypic components $\{W_\gamma\}$.
Multiplying ${\bf f}_\lambda$ by $\lambda$ gives 
 \begin{equation} \label{eq:decomp_proof2}
  \lambda {\bf f}_\lambda= \sum_{\gamma \vdash n: \lambda \in \Lambda_\gamma} \lambda {\bf f}_{\gamma, \lambda},
 \end{equation}
and multiplying ${\bf f}_\lambda$ by $\L$ gives
 \begin{equation} \label{eq:decomp_proof1}
 \lambda {\bf f}_\lambda= \L {\bf f}_\lambda=\sum_{\gamma \vdash n: \lambda \in \Lambda_\gamma} \L {\bf f}_{\gamma,\lambda}.
  \end{equation}
Since $W_\gamma$ is closed under multiplication by $\L$, each term $\L {\bf f}_{\gamma,\lambda}$ in the summation in \eqref{eq:decomp_proof1} is in $W_\gamma$.
By the uniqueness of the decomposition of ${\bf f}$ into isotypic components, we conclude from \eqref{eq:decomp_proof2} and \eqref{eq:decomp_proof1} that $\L {\bf f}_{\gamma,\lambda} = \lambda {\bf f}_{\gamma,\lambda}$ for each $\gamma \vdash n$ such that $\lambda \in \Lambda_\gamma$. Thus, ${\bf f}_{\gamma,\lambda}$ is a Laplacian eigenvector of eigenvalue $\lambda$ and ${\bf f}_{\gamma,\lambda}\in Z_{\gamma,\lambda}=W_\gamma \cap U_\lambda$.  In summary, for any ${\bf f} \in \Rbb[\SS_n]$, there is a unique decomposition ${\bf f}=\sum_{\gamma \vdash n} \sum_{\lambda \in \Lambda_\gamma} {\bf f}_{\gamma,\lambda}$, with ${\bf f}_{\gamma,\lambda}\in Z_{\gamma,\lambda}$.
\qed\end{proof}

If $\lambda \in \Lambda_\gamma$, we say that the eigenvalue $\lambda$ has \emph{symmetry type} or \emph{shape} 
$\gamma$. Eigenvalues may have multiple symmetry types. For example, the Laplacian eigenvalue $\lambda=3$ on $\PP_6$ is repeated 15 times (i.e., $U_3$ is a 15-dimensional space). This eigenvalue appears 5 times in the $\gamma=\begin{array}{c}{\begin{tikzpicture}[scale=.15,line width=1.0pt] 
\draw (0,0) rectangle (1,1); \draw (1,0) rectangle (2,1); \draw (2,0) rectangle (3,1); \draw (3,0) rectangle (4,1); \draw (4,0) rectangle (5,1);
\draw (0,-1) rectangle (1,0); 
\end{tikzpicture}}\end{array}$ 
component and 10 times in the $\gamma=\begin{array}{c}{\begin{tikzpicture}[scale=.15,line width=1.0pt] 
\draw (0,0) rectangle (1,1); \draw (1,0) rectangle (2,1); \draw (2,0) rectangle (3,1); \draw (3,0) rectangle (4,1); 
\draw (0,-1) rectangle (1,0); 
\draw (0,-2) rectangle (1,-1); 
\end{tikzpicture}}\end{array}$  component (e.g., $W_{\begin{tikzpicture}[scale=.15,line width=1.0pt] 
\draw (0,0) rectangle (1,1); \draw (1,0) rectangle (2,1); \draw (2,0) rectangle (3,1); \draw (3,0) rectangle (4,1); 
\draw (0,-1) rectangle (1,0); 
\draw (0,-2) rectangle (1,-1); 
\end{tikzpicture}} \cap U_3$ is a 10-dimensional space), so $\lambda=3$ has two symmetry types. On the other hand, $\lambda=2$ has a single symmetry type as it only appears in the  $\gamma=\begin{array}{c}{\begin{tikzpicture}[scale=.15,line width=1.0pt] 
\draw (0,0) rectangle (1,1); \draw (1,0) rectangle (2,1); \draw (2,0) rectangle (3,1); \draw (3,0) rectangle (4,1); \draw (4,0) rectangle (5,1);
\draw (0,-1) rectangle (1,0); 
\end{tikzpicture}}\end{array}$ component.

Our objective is to find a spanning set $\{{\boldsymbol \varphi}_{\gamma,\lambda,k,\pi}\}$ of dictionary atoms for each space $Z_{\gamma, \lambda}$ such that 
\begin{enumerate}
\item[(i)] the overall dictionary analysis operator preserves the energy in the signal; that is,  
$\sum_{\gamma,\lambda,k,\pi} |\langle  {\bf f},  {\boldsymbol \varphi}_{\gamma,\lambda,k,\pi} \rangle|^2 = \lVert{\bf f}\rVert_2^2,$ or equivalently, $\lVert{\boldsymbol \Phi}^{\top}{\bf f}\rVert_2=\lVert{\bf f}\rVert_2$, where the atoms $\{{\boldsymbol \varphi}_{\gamma,\lambda,k,\pi}\}$ comprise the columns of the matrix ${\boldsymbol \Phi}$; \item[(ii)] the atoms are interpretable (i.e., they have a particular structure that makes the inner products useful in identifying structure in the data); and 
\item[(iii)] we can efficiently compute the inner products between these dictionary atoms and the signal on the permutahedron.
\end{enumerate}

\subsection{Preliminaries: Schreier Graphs and Equitable Partitions} \label{Se:prelim}
 \newcommand\shapethreetwo{\tikz[baseline]{\draw[xscale=.25,yscale=.27,line width=0.8pt] (0,0) rectangle (3,1); 
\path[xscale=.25,yscale=.27,line width=0.8pt] (.5,0.5) node {$\scriptstyle{1}$}; \path[xscale=.25,yscale=.27,line width=0.8pt] (1.5,0.5) node {$\scriptstyle{2}$}; \path[xscale=.25,yscale=.27,line width=0.8pt] (2.5,0.5) node {$\scriptstyle{4}$}; 
\draw[xscale=.25,yscale=.27,line width=0.8pt] (0,-1) rectangle (2,0); 
\path[xscale=.25,yscale=.27,line width=0.8pt] (.5,-0.5) node {$\scriptstyle{3}$};  \path[xscale=.25,yscale=.27,line width=0.8pt] (1.5,-0.5) node {$\scriptstyle{5}$};}}
 \newcommand\shapethreetwoB{\tikz[baseline]{\draw[xscale=.25,yscale=.27,line width=0.8pt] (0,0) rectangle (3,1); 
\path[xscale=.25,yscale=.27,line width=0.8pt] (.5,0.5) node {$\scriptstyle{1}$}; \path[xscale=.25,yscale=.27,line width=0.8pt] (1.5,0.5) node {$\scriptstyle{2}$}; \path[xscale=.25,yscale=.27,line width=0.8pt] (2.5,0.5) node {$\scriptstyle{3}$}; 
\draw[xscale=.25,yscale=.27,line width=0.8pt] (0,-1) rectangle (2,0); 
\path[xscale=.25,yscale=.27,line width=0.8pt] (.5,-0.5) node {$\scriptstyle{4}$};  \path[xscale=.25,yscale=.27,line width=0.8pt] (1.5,-0.5) node {$\scriptstyle{5}$};}}
\begin{wrapfigure}[18]{r}{0.47\textwidth}
  \begin{center}
  \vspace{-.7in}
\includegraphics[width=0.37\textwidth,page=2]{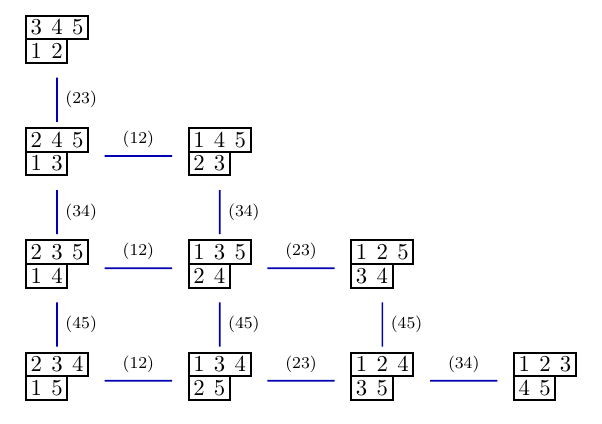} 
  \end{center}
  \vspace{-.13in}
  \caption{The Schreier graph $\PP_{[3,2]}$. In this graph, the vertices are the $m_{[3,2]} = 5!/(3!2!) = 10$ ordered set partitions of $\{1,2,3,4,5\}$ into subsets of size 3 and 2. We use row-strict tableaux to denote set partitions, and interpret the labels as groupings of rankings. For example, $\{\{1,2,4\},\{3,5\}\}$ is denoted by 
{\protect\shapethreetwo} and groups the third and fifth place rankings together, and the other three places together. The edge labels correspond to adjacent transpositions. For example, the (34) label on the edge between {\protect\shapethreetwo} and {\protect\shapethreetwoB}  represents a swapping of the third and fourth place rankings in the partition blocks. We omit self loops when drawing Schreier graphs, as they do not affect the corresponding graph Laplacian operators.}\label{Fig:P32}
\end{wrapfigure}
We start by detailing (i) how to construct, for each integer partition $\gamma$ of $n$,  
a graph $\PP_\gamma$, called a Schreier graph, which is isomorphic to a quotient of the permutahedron $\PP_n$ \cite{friedman2000cayley,schreier1927untergruppen}; and (ii) the relation between the spectral decompositions of the Schreier graphs and the spectral decomposition of the permutahedron.   

If $\gamma = [\gamma_1, \ldots, \gamma_\ell]$ is an integer partition of $n$, then a set partition $\pi=\{C_1, \ldots, C_\ell \}$  has shape $\gamma$ if its blocks have size $|C_i| = \gamma_i$. There are $m_\gamma=\frac{n!}{\prod_{i=1}^{\ell} \gamma_i !}$ different ordered set partitions $\pi$ of $\{1,2,\ldots,n\}$ of shape $\gamma$, and we refer to this collection of ordered partitions as $\Pi_\gamma$.\footnote{By ordered set partitions, we mean that changing the ordering of the blocks results in a different partition, but changing the ordering within blocks does not. For example, $\pi=\{\{1,2\},\{3,4\}\}$ and $\pi^{\prime}=\{\{3,4\},\{1,2\}\}$ are distinct elements of $\Pi_{[2,2]}$. However, $\pi^{\prime\prime}=\{\{2,1\},\{3,4\}\}$ is equivalent to $\pi$.} 
A permutation $\sigma \in \SS_n$ acts on an ordered set partition $\pi \in \Pi_\gamma$ by permuting the entries of $\pi$. For example if $\sigma = \left(\begin{smallmatrix} 1 & 2 & 3 & 4 & 5 \\  2 & 5 & 4 & 3 & 1 \end{smallmatrix}\right)$ and $\pi = \{ \{2,4,5\}, \{1,3\}\}$ then $\sigma(\pi) = \{ \{1,3,5\}, \{2,4\}\}$. 
\begin{definition}\label{Def:Schreier}
The \emph{Schreier graph} $\PP_\gamma$ is the graph with vertex set $\Pi_\gamma$ and edge set $\{(\pi,s \pi) \mid \pi \in \Pi_\gamma, s\in S\}$, where $S \subseteq \SS_n$ is the subset of adjacent transpositions defined in Sec.~\ref{Se:intro}.
\end{definition}

Each Schreier graph (i) is undirected since adjacent transpositions are involutions; (ii) is $(n-1)$-regular since $|S| = n-1$; and (iii) has a self-loop at vertex $\pi$ for each pair $i,i+1$ that is in the same block of $\pi$, since then the transposition $s= (i,i+1)$ fixes $\pi$. Two ordered set partitions $\pi_1 \not= \pi_2$ are connected by the edge labeled by $s = (i,i+1)$ if and only if $\pi_1$ and $\pi_2$ are identical except with $i$ and $i+1$ switched. Thus, non-loop edge weights are equal to 1.  We view the vertices of $\PP_\gamma$ as representing groupings of rankings (first place, second place, etc.), not specific candidates. Fig. \ref{Fig:P32} shows the example of  $\PP_\gamma$ with $\gamma = [3,2]$.

Next, we show that 
each ordered set partition $\pi$ induces an equivalence relation on $\SS_n$, the vertices of the permutahedron $\PP_n$, 
and under this equivalence relation, the Schreier graphs are isomorphic to quotient graphs of the permutahedron $\PP_n$. 

\begin{definition}\label{Def:equivalencerel}
Let $\pi 
\in \Pi_\gamma$ be an ordered set partition of shape $\gamma$, and let $\sigma, \tau \in \SS_n$. The
equivalence relation $\sim_\pi$ is given by identifying $\sigma \sim_\pi \tau$ if and only if 
$\sigma^{-1}( \pi) =\tau^{-1}(\pi)$. The equivalence classes under $\sim_\pi$ are the sets ${\cal V}_{\pi,\mu} = \{ \sigma \in \SS_n \mid \sigma(\mu) = \pi\}$ for each  $\mu \in \Pi_\gamma.$ These are the permutations that place candidates $\pi$ in the positions given by $\mu.$
\end{definition}
For example, if $\pi = \{\{2,4,5\},\{1,3\}\}$ (or $\{245|13\}$ for shorter notation), then  $34512 \sim_\pi 12435$ because each group of candidates, $\{2,4,5\}$ and $\{1,3\}$, is in the same set of positions in the two permutations:  $3\underline{45}1\underline{2}$ and $1\underline{24}3\underline{5}$. Furthermore, the equivalence class containing these two permutations is the set ${\cal V}_{\{245|13\}, \{235|14\}}$ consisting of all permutations with \{2,4,5\} and $\{1,3\}$ in positions \{2,3,5\} and $\{1,4\}$, respectively.

\newcommand{\shapePtwotwo}{$\PP_{\begin{tikzpicture}[scale=.15,line width=1.0pt] 
\draw (0,0) rectangle (1,1); \draw (1,0) rectangle (2,1);   
\draw (0,-1) rectangle (1,0);  \draw (1,-1) rectangle (2,0); 
\end{tikzpicture}}$}
\begin{figure}[]
\centering
\begin{minipage}{.2\linewidth}
\centering
\includegraphics[width=\linewidth,page=1]{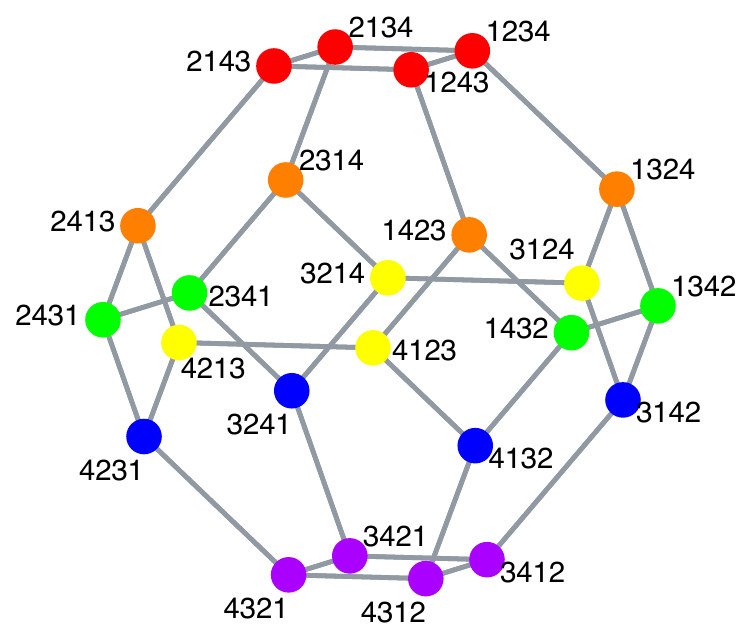}
\end{minipage}
\hspace{.05in}
\begin{minipage}{.2\linewidth}
\centering
\includegraphics[width=\linewidth,page=2]{figures/eq_parts_v2}
\end{minipage}
\hspace{.05in}
\begin{minipage}{.2\linewidth}
\centering
\includegraphics[width=\linewidth,page=3]{figures/eq_parts_v2}
\end{minipage}
\hspace{.1in}
\begin{minipage}{.22\linewidth}
\centering
\includegraphics[width=\linewidth,trim=0 0 25 0, clip]{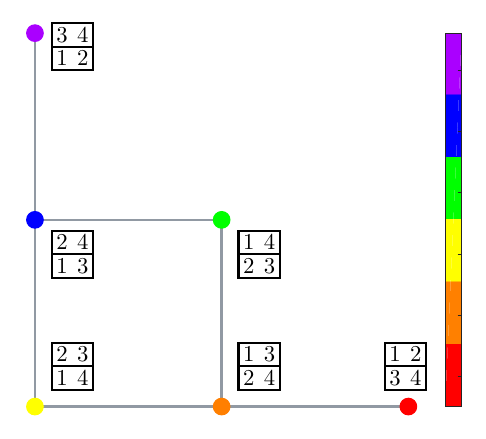}
\end{minipage}
\vspace{-.1in}
\caption{Right: The Schreier graph {\protect\shapePtwotwo}$=\PP_{[2,2]}$ (shown without self loops). Recall from Fig. \ref{Fig:P32} that we view the vertices of the Schreier graph as ordered set partitions of \emph{ranking slots}. Left: Three different equitable partitions of $\PP_4$ that yield this Schreier graph, corresponding to the ordered set partitions $\pi_1=\{\{1,2\},\{3,4\}\}$, $\pi_2=\{\{1,3\},\{2,4\}\}$, and $\pi_3=\{\{1,4\},\{2,3\}\}$, respectively, from left to right. We view the three equitable partitions on the left as set partitions of \emph{candidates}. 
For each equitable partition, all vertices of the permutahedron with the same color are mapped  
to the node of the corresponding color on the Schreier graph; this mapping is encoded in the rows of the characteristic matrix  ${\bf B}_{\pi_i}$. For each equitable partition, the two pairs of candidates defined by the ordered set partition $\pi_i$ are in the ranking slots defined by the two rows of the tableaux with the corresponding color in the Schreier graph on the right. For example, in the middle equitable partition $\pi_2$, the red vertices correspond to rankings where candidates 1 and 3 are in ranking positions 1 and 2, and candidates 2 and 4 are in positions 3 and 4.  }\label{Fig:eq_part}
\vspace{0.1in}
\hrule height 1.5pt
\vspace{-.1in}
\end{figure}

\begin{proposition}\label{Pr:equitable}
The partition of the vertices of $\PP_n$ into equivalence classes $\{{\cal V}_{\pi,\mu} \mid \mu \in \Pi_\gamma\}$ induced by $\sim_\pi$ is an  \emph{equitable partition}, meaning that for every pair of (not necessarily distinct) ordered set partitions $\mu, \nu \in \Pi_\gamma$, there is a nonegative integer ${\bf K}_\pi(\mu,\nu)$ such that each vertex $\sigma \in {\cal V}_{\pi,\mu}$ has exactly ${\bf K}_\pi(\mu,\nu)$ neighbors in ${\cal V}_{\pi,\nu}$ \cite[Sec. 9.3]{godsil2013algebraic}. In fact, when $\mu \not= \nu$, ${\bf K}_\pi(\mu,\nu) = 1$ if there exists $s \in S$ such that $s(\mu) = \nu$ and equals 0 otherwise, and ${\bf K}_\pi(\mu,\mu)$ equals the number of $s \in S$ such that $s(\mu) = \mu$.
\end{proposition}
\begin{proof} When the equivalence class ${\cal V}_{\pi,\mu} = \{ \sigma \in \SS_n \mid \sigma(\mu) = \pi\}$ is right multiplied by $\tau \in \SS_n$, we have 
${\cal V}_{\pi,\mu} \tau  =  \{ \sigma \tau \in \SS_n \mid \sigma(\mu) = \pi\} =  \{ \sigma \tau \in \SS_n \mid \sigma\tau(\tau^{-1}(\mu)) = \pi\} = {\cal V}_{\pi,\tau^{-1}(\mu)}$. It follows that if $s \in S$ and $\sigma \in {\cal V}_{\pi,\mu}$ then $\sigma s \in {\cal V}_{\pi,s(\mu)}$. Multiplication by a group element $s$ is a bijection, so each element in ${\cal V}_{\pi,\mu}$ is connected by an edge in $\PP_n$ to exactly one  element in ${\cal V}_{\pi,s(\mu)}$. If there is not an adjacent transposition $s$ such that $s(\mu) = \nu$, then there are no edges in $\PP_n$ between the vertices in ${\cal V}_{\pi,\mu}$ and those in ${\cal V}_{\pi,\mu}$.
\qed\end{proof}

\setlength{\kbrowsep}{15pt}
\begin{wrapfigure}[23]{l}{.35\textwidth}
\vspace{-.5in}
$$
\kbordermatrix{
&
\begin{tikzpicture}[xscale=.25,yscale=.3,line width=0.8pt] 
\draw (0,0) rectangle (2,1); 
\path (.5,0.5) node {{\scriptsize $1$}}; \path (1.5,0.5) node {{\scriptsize $2$}}; 
\draw (0,-1) rectangle (2,0); 
\path (.5,-0.5) node {{\scriptsize $3$}};  \path (1.5,-0.5) node {{\scriptsize $4$}};  
\end{tikzpicture}
&
\begin{tikzpicture}[xscale=.25,yscale=.3,line width=0.8pt]  
\draw (0,0) rectangle (2,1); 
\path (.5,0.5) node {{\scriptsize $1$}}; \path (1.5,0.5) node {{\scriptsize $3$}}; 
\draw (0,-1) rectangle (2,0); 
\path (.5,-0.5) node {{\scriptsize $2$}};  \path (1.5,-0.5) node {{\scriptsize $4$}};
\end{tikzpicture}
&
\begin{tikzpicture}[xscale=.25,yscale=.3,line width=0.8pt] 
\draw (0,0) rectangle (2,1); 
\path (.5,0.5) node {{\scriptsize $2$}}; \path (1.5,0.5) node {{\scriptsize $3$}}; 
\draw (0,-1) rectangle (2,0); 
\path (.5,-0.5) node {{\scriptsize $1$}};  \path (1.5,-0.5) node {{\scriptsize $4$}};  
\end{tikzpicture}
&
\begin{tikzpicture}[xscale=.25,yscale=.3,line width=0.8pt] 
\draw (0,0) rectangle (2,1); 
\path (.5,0.5) node {{\scriptsize $1$}}; \path (1.5,0.5) node {{\scriptsize $4$}}; 
\draw (0,-1) rectangle (2,0); 
\path (.5,-0.5) node {{\scriptsize $2$}};  \path (1.5,-0.5) node {{\scriptsize $3$}};  
\end{tikzpicture}
&
\begin{tikzpicture}[xscale=.25,yscale=.3,line width=0.8pt] 
\draw (0,0) rectangle (2,1); 
\path (.5,0.5) node {{\scriptsize $2$}}; \path (1.5,0.5) node {{\scriptsize $4$}}; 
\draw (0,-1) rectangle (2,0); 
\path (.5,-0.5) node {{\scriptsize $1$}};  \path (1.5,-0.5) node {{\scriptsize $3$}};  
\end{tikzpicture}
&
\begin{tikzpicture}[xscale=.25,yscale=.3,line width=0.8pt] 
\draw (0,0) rectangle (2,1); 
\path (.5,0.5) node {{\scriptsize $3$}}; \path (1.5,0.5) node {{\scriptsize $4$}}; 
\draw (0,-1) rectangle (2,0); 
\path (.5,-0.5) node {{\scriptsize $1$}};  \path (1.5,-0.5) node {{\scriptsize $2$}};  
\end{tikzpicture}  \vspace{-.2in}
\\
1234&\cdot&1&\cdot&\cdot&\cdot&\cdot\\
1243&\cdot&\cdot&\cdot&1&\cdot&\cdot\\
1324&1&\cdot&\cdot&\cdot&\cdot&\cdot\\
1342&1&\cdot&\cdot&\cdot&\cdot&\cdot\\
1423&\cdot&\cdot&\cdot&1&\cdot&\cdot\\
1432&\cdot&1&\cdot&\cdot&\cdot&\cdot\\
2134&\cdot&\cdot&1&\cdot&\cdot&\cdot\\
2143&\cdot&\cdot&\cdot&\cdot&1&\cdot\\
2314&\cdot&\cdot&1&\cdot&\cdot&\cdot\\
2341&\cdot&\cdot&\cdot&\cdot&1&\cdot\\
2413&\cdot&\cdot&\cdot&\cdot&\cdot&1\\
2431&\cdot&\cdot&\cdot&\cdot&\cdot&1\\
3124&1&\cdot&\cdot&\cdot&\cdot&\cdot\\
3142&1&\cdot&\cdot&\cdot&\cdot&\cdot\\
3214&\cdot&1&\cdot&\cdot&\cdot&\cdot\\
3241&\cdot&\cdot&\cdot&1&\cdot&\cdot\\
3412&\cdot&1&\cdot&\cdot&\cdot&\cdot\\
3421&\cdot&\cdot&\cdot&1&\cdot&\cdot\\
4123&\cdot&\cdot&\cdot&\cdot&1&\cdot\\
4132&\cdot&\cdot&1&\cdot&\cdot&\cdot\\
4213&\cdot&\cdot&\cdot&\cdot&\cdot&1\\
4231&\cdot&\cdot&\cdot&\cdot&\cdot&1\\
4312&\cdot&\cdot&1&\cdot&\cdot&\cdot\\
4321&\cdot&\cdot&\cdot&\cdot&1&\cdot
}
$$\caption{The characteristic matrix ${\bf B}_{\pi_2}$ of the ordered set 
partition $\pi_2=\{\{1,3\},\{2,4\}\}$ from Fig.~\ref{Fig:eq_part}. In the second row, e.g., candidates 1, 3 are in ranking slots 1, 4.}  \label{Fig:char_mat}
\end{wrapfigure}
~
\vspace{-.4in}

\begin{proposition}\label{Pr:quotient}
For each ordered set partition $\pi
\in \Pi_\gamma$, the quotient graph $\PP_n /{\sim_\pi}$ is isomorphic to $\PP_\gamma$.
\end{proposition}

\begin{proof} Since $\sim_\pi$ induces an equitable partition of $\PP_n$ with equivalence classes ${\cal V}_{\pi,\mu}$, the quotient $\PP_n/\sim_\pi$ is a well-defined graph with vertices $\{{\cal V}_{\pi,\mu} \mid \mu \in \Pi_\gamma\}$ and edges $\{ ({\cal V}_{\pi,\mu}, {\cal V}_{\pi,s(\mu)}) \mid \mu \in \Pi_\gamma, s \in S\}$. Comparing Prop. \ref{Pr:equitable} with  Def. \ref{Def:Schreier} shows that these graphs are isomorphic by identifying the vertex ${\cal V}_{\pi,\mu}$ in  $\PP_n/\sim_\pi$  with $\mu$ in $\PP_\gamma$. \qed\end{proof}

\begin{definition} 
The \emph{characteristic matrix of an equitable partition} $\pi$ of shape $\gamma$, denoted by ${\bf B}_{\pi}$, is the $n! \times m_\gamma$ $(0,1)$-matrix whose $(\sigma,\mu)$th element, for $\sigma \in \SS_n$ and $\mu \in \Pi_\gamma$, is equal to 1 if  and only if $\sigma(\mu) = \pi$; that is 
$\sigma$ places the candidates from the $j$th block of $\pi$ into the positions corresponding to the $j$th block of $\mu$, for each $j$.
\end{definition}

Each row of ${\bf B}_{\pi}$ contains exactly one 1, and each column contains exactly $\frac{n!}{m_\gamma}$ 1s. Therefore ${\bf B}_{\pi}^\top{\bf B}_{\pi} = (\frac{n!}{m_\gamma}) {\bf I}_{m_\gamma}$.  This is illustrated in Fig. \ref{Fig:char_mat}.  Furthermore, ${\bf B}_{\pi}$ has the following symmetry property, which we use in our frame construction   \eqref{Eq:frame_elements} and Thm.~\ref{Th:tight_frame}.

 \begin{proposition}\label{Pr:lift_symmetry} For each $\sigma \in \SS_n$ and $\pi \in \Pi_\gamma$ we have $\rho_L(\sigma)  {\bf B}_{\pi} = {\bf B}_{\sigma(\pi)}$, where $\rho_L(\sigma)$ is the matrix of $\sigma$ in the left regular representation of $\SS_n$ on $\Rbb[\SS_n]$.
 \end{proposition}
 
 \begin{proof} For $\sigma,\tau \in \SS_n$ and $\mu\in \Pi_\gamma$, the $(\tau,\mu)$ entry of $\rho_L(\sigma)  {\bf B}_{\pi}$ equals 1 if and only if $\sigma^{-1} \tau (\mu) = \pi$, which is true if and only if $\tau(\mu) = \sigma(\pi)$, and this is exactly the condition for the $(\tau,\mu)$ entry of ${\bf B}_{\sigma(\pi)}$.
\qed
\end{proof}

A key property of characteristic matrices of equitable partitions is that they can be used to lift eigenvectors of the Schreier graphs $\PP_\gamma$ to eigenvectors of the permutahedron $\PP_n$.

 \begin{proposition}\label{Pr:evec_lift}
If $\pi \in \Pi_\gamma$ and ${\bf v}_{\gamma, \lambda}$ is a graph Laplacian eigenvector of the Schreier graph $\PP_\gamma$ with eigenvalue $\lambda$, then ${\bf w} = {\bf B}_{\pi} {\bf v}_{\gamma, \lambda}$ is a graph Laplacian eigenvector of the permutahedron  $\PP_n$ with the same eigenvalue $\lambda$.
\end{proposition}

\begin{proof}
Let ${\bf A}_{\PP_n}$ be the adjacency matrix of the permutahedron, and ${\bf A}_{\PP_\gamma} = {\bf A}_{\PP_n/{\sim_\pi}}$ be the adjacency matrix of the Schreier graph $\PP_\gamma$. Since $\pi$ induces  
an equitable partition $\sim_\pi$ of $\PP_n$,  ${\bf A}_{\PP_n}{\bf B}_{\pi} = {\bf B}_{\pi} {\bf A}_{\PP_\gamma}$ \cite[Lem. 9.3.1]{godsil2013algebraic}, and thus,
\begin{align}\label{Eq:eq_part_lap}
\L_{\PP_n}{\bf B}_{\pi} &= (n-1){\bf I}_{n!}{\bf B}_{\pi} - {\bf A}_{\PP_n}{\bf B}_{\pi} \nonumber \\
&= (n-1){\bf B}_{\pi}{\bf I}_{m_\gamma}-{\bf B}_{\pi}{\bf A}_{\PP_\gamma}  ={\bf B}_{\pi} \L_{\PP_\gamma},
\end{align}
as both $\PP_n$ and $\PP_\gamma$ are (n-1)-regular graphs.
Thus,
\begin{align*}
\L_{\PP_n} {\bf B}_{\pi} {\bf v}_{\gamma, \lambda}={\bf B}_{\pi} \L_{\PP_\gamma} {\bf v}_{\gamma, \lambda}={\bf B}_{\pi} (\lambda {\bf v}_{\gamma, \lambda}) =
\lambda ({\bf B}_{\pi} {\bf v}_{\gamma, \lambda}),
\end{align*}
where the first equality follows from \eqref{Eq:eq_part_lap}. \qed
\end{proof}  

For each shape $\gamma$ we view signals on the Schreier graph $\PP_\gamma$ as vectors in the vector space $\Rbb[\Pi_\gamma]$  with canonical basis  $\{ \e_\pi\}_{\pi \in \Pi_\gamma}$. This vector space has a natural right $\SS_n$-action defined on a basis element $\e_\pi$ and a permutation $\sigma \in \SS_n$ by $\e_\pi \sigma = \e_{\sigma^{-1}(\pi)}$, where permutations act on set partitions in $\Pi_\gamma$ by permuting the entries as described just before Def \ref{Def:equivalencerel}. The $\SS_n$-module $\Rbb[\Pi_\gamma]$ is known as the (right) ``permutation module" for $\SS_n$ and is often denoted by $M^\gamma$ (see for example \cite[Ch. 2]{sagan2013symmetric}).  

For each shape $\gamma \vdash n$, the permutation module $\Rbb[\Pi_\gamma]$ decomposes into irreducible (right) $\SS_n$-submodules according to 
\begin{equation}\label{eq:SchreierDecomposition}
\Rbb [\Pi_\gamma] \cong V_\gamma^\ast \oplus \bigoplus_{\nu  \vartriangleright \gamma} (V_\nu^\ast)^{\oplus K_{\gamma,\nu}},
\end{equation}
where again
$\nu \vartriangleright \gamma$ means that $\nu$ strictly dominates $\gamma$.
Thus $\Rbb[\Pi_\gamma]$ contains exactly one copy of the irreducible $V_\gamma^\ast$ and $K_{\gamma,\nu}$ copies of each irreducible module that comes before it in dominance order on partitions. The multiplicities  $K_{\gamma,\nu}$ are known as Kostka numbers (see \cite[Sec.~2.11]{sagan2013symmetric} and Fig. \ref{Fig:Kostka}). 
Furthermore, $V_\gamma^\ast$ does not appear as a submodule of  $\Rbb [\Pi_\pi]$ for partitions $\pi$ whose shapes come before $\gamma$ in dominance order, which is beneficial for the computational algorithms that we explore in Sec. \ref{Se:comp}.

The following proposition says that when a vector in $\Rbb[\Pi_\gamma]$ that lives entirely in the submodule $V_\gamma^\ast$ is lifted to the permutahedron by the characteristic matrix of an equitable partition of shape $\gamma$, the resulting vector resides in the single isotypic component $W_\gamma \subseteq \Rbb[\SS_n]$ defined in \eqref{eq:isotypic}.

\begin{proposition}\label{Pr:lift_isotypic}
If $\pi \in \Pi_\gamma$ and ${\bf x} \in V_\gamma^\ast \subseteq \Rbb[\Pi_\gamma]$, then ${\bf B}_{\pi} {\bf x} \in W_\gamma$.
\end{proposition}

\begin{proof}
The map $\Rbb[\Pi_\gamma] \to \Rbb[\SS_n]$ given by ${\bf x} \mapsto {\bf B}_{\pi} {\bf x}$ is an injective, right $\SS_n$-module homomorphism (i.e., it commutes with the right $\SS_n$ action on the the two spaces). It is injective, since $ {\bf B}_{\pi}$ has rank $m_\gamma$, and is an $\SS_n$-module homomorphism, since for $\tau \in \SS_n$ and $\pi,\mu \in \Pi_\gamma$ we have  
$$
{\bf B}_{\pi} ( \e_\mu \tau) = {\bf B}_{\pi} ( \e_{\tau^{-1} (\mu )}) = \sum_{\sigma, \sigma\tau^{-1}(\mu)=\pi} \e_\sigma= \sum_{\eta, \eta(\mu)=\pi} \e_{\eta \tau} =  \sum_{\eta, \eta(\mu)=\pi} \e_{\eta} \tau = ({\bf B}_{\pi}  \e_\mu) \tau.
$$  Upon restriction to the irreducible submodule $V_\gamma^\ast$,  by Schur's lemma, the map must be an isomorphism. Thus the image of $V_\gamma^\ast$ under  ${\bf x} \mapsto {\bf B}_{\pi} {\bf x}$ is an isomorphic copy of $V_\gamma^\ast$. The isotypic component $W_\gamma$ contains all copies of $V_\gamma^\ast$ in $\Rbb[\SS_n]$ so ${\bf B}_{\pi} {\bf x} \in W_\gamma$.
\qed\end{proof}

An important implication of Prop. \ref{Pr:evec_lift} and Prop. \ref{Pr:lift_isotypic} is that we can compute, visualize, and interpret Laplacian eigenvectors on the lower-dimensional Schreier graphs, and then lift them  up to the higher-dimensional permutahedron graph in different manners -- assigning different groups of candidates to groups of ranking slots -- in order to generate vectors that reside in specific spaces $Z_{\gamma,\lambda}$ (i.e., have certain symmetry types and smoothness levels). Next, we show that scaled versions of vectors generated in this manner constitute a tight frame for the space $\Rbb [\SS_n]$ of all possible signals on the permutahedron $\PP_n$.

\subsection{Tight Frame Construction}
Our strategy is to construct a tight Parseval frame for each nonempty space $Z_{\gamma, \lambda}$, and then let the dictionary ${\boldsymbol \Phi}$ be the union of these tight frames, so that ${\boldsymbol \Phi}$ is a tight Parseval frame for $\Rbb[\SS_n]$. A set of vectors $\{{\boldsymbol \varphi}_j\}$ is a \emph{frame} for a Hilbert space ${\cal H}$ if there exists frame bounds (constants) $A, B > 0$ such that    
$ A \lVert{\bf f}\rVert_2^2 \leq \sum_j | \langle {\bf f},  {\boldsymbol \varphi}_j \rangle|^2 \leq B \lVert{\bf f}\rVert_2^2,~\forall {\bf f} \in {\cal H}.$
A frame is said to be \emph{tight} if $A=B$, and a \emph{Parseval frame} if $A=B=1$. For finite dimensional Hilbert spaces (such as our $Z_{\gamma, \lambda}$ spaces), a (finite) frame is simply a set of spanning vectors for the space.  A frame is a \emph{group frame} if there exits a finite group $G$ that acts as linear transformations on ${\cal H}$ such that $\{{\boldsymbol \varphi}_j\} = \{ g \boldsymbol\varphi_1\}_{g \in G}$; that is, the frame is generated by rotating a single frame vector ${\boldsymbol\varphi}_1$ by the group $G$.
For more background on frames and group frames, and their use in signal processing and machine learning, see \cite{frames,kovacevic_frames1,kovacevic_frames2,waldron2018introduction}.  

Analogous to  equation \eqref{eq:Laplacian_representation}, the Laplacian matrix $\L_\gamma$ on the Schreier graph $\PP_\gamma$ equals 
\begin{equation}\label{eq:Laplacian_Schreier_representation}
\L_\gamma = (n-1) \rho_\gamma({\bf 1}) - \sum_{s \in S} \rho_\gamma(s), 
\end{equation}
where $\rho_\gamma$ is the representation of $\SS_n$ on $\Rbb[\Pi_\gamma]$, which decomposes as in \eqref{eq:SchreierDecomposition}. It follows that $V_\gamma^\ast$ is closed under multiplication by $\L_\gamma$, and therefore, from the identical argument as in the proof of Prop.\ref{prop:decomp_eigenvalues_by_shape}, $V_\gamma^\ast$ decomposes into eigenspaces for $\L_\gamma$.

For each shape $\gamma \vdash n$ and eigenvalue $\lambda \in \Lambda_\gamma$, to generate a tight Parseval frame $\varPhi_{\gamma,\lambda}$ 
for  $Z_{\gamma, \lambda}$, we 
\begin{enumerate}
\item[(i)] construct an orthonormal basis $\left\{{\bf v}_{\gamma,\lambda,k}\right\}_{k=1}^{\kappa_{\gamma,\lambda}}$ for the graph Laplacian eigenspace of the Schreier graph $\PP_\gamma$ that is associated with the eigenvalue $\lambda$ in $V_\gamma^\ast$, and then 
\item[(ii)] lift each eigenvector ${\bf v}_{\gamma,\lambda,k}$ in the basis back to the permutahedron $\PP_n$ in 
$m_\gamma$ different ways. 
\end{enumerate}
Specifically, we define ${\boldsymbol \varphi}_{\gamma,\lambda,k,\pi}:=c_{\gamma}{\bf B}_{\pi} {\bf v}_{\gamma, \lambda, k}$, and 
\begin{align}\label{Eq:frame_elements}
\varPhi_{\gamma,\lambda}:=\bigcup_{k=1}^{\kappa_{\gamma,\lambda}}
 \bigcup_{\pi \in \Pi_\gamma} 
 {\boldsymbol \varphi}_{\gamma,\lambda,k,\pi}=\bigcup_{k=1}^{\kappa_{\gamma,\lambda}} \bigcup_{\pi \in \Pi_\gamma} c_{\gamma}{\bf B}_{\pi} {\bf v}_{\gamma, \lambda, k},
\end{align}
where the constant $c_\gamma:= \sqrt{\frac{d_\gamma}{n!}}$ with $d_\gamma = \dim(V_\gamma)$. It may be the case that some of the liftings ${\bf B}_\pi \vb_{\gamma,\lambda,k}$ are equal for different $\pi \in \Pi_\gamma$, but we keep these multiple copies in $\Phi_{\gamma,\lambda}$ (viewing it as a multiset), so that $|\Phi_{\gamma,\lambda}| = \kappa_{\gamma,\lambda} m_\gamma$. We remove some of these redundancies in the reduced frame $\bar{\Phi}_{\gamma,\lambda}$ in \eqref{Eq:frame_elements2} below.

\begin{remark}
It is often but not always the case that $\kappa_{\gamma,\lambda}=1$, so that the basis $\left\{{\bf v}_{\gamma,\lambda,k}\right\}_{k=1}^{\kappa_{\gamma,\lambda}}$ consists of a single vector. When $\kappa_{\gamma,\lambda}=1$, we shorten the notation from ${\bf v}_{\gamma,\lambda,k}$ to ${\bf v}_{\gamma,\lambda}$, and from ${\boldsymbol \varphi}_{\gamma,\lambda,k,\pi}$ to ${\boldsymbol \varphi}_{\gamma,\lambda,\pi}$.
To make interpretations more consistent, in our implementations, we always choose the eigenvectors of the Schreier graphs to have norm 1 and a positive coefficient on the vertex associated with ordered set partition $\mu$ that is last in lexicographic order (e.g.,
{\tikz[baseline]{\draw[xscale=.25,yscale=.27,line width=0.8pt] (0,0) rectangle (2,1); 
\path[xscale=.25,yscale=.27,line width=0.8pt] (.5,0.5) node {$\scriptstyle{3}$}; \path[xscale=.25,yscale=.27,line width=0.8pt] (1.5,0.5) node {$\scriptstyle{4}$}; 
\draw[xscale=.25,yscale=.27,line width=0.8pt] (0,-1) rectangle (2,0); 
\path[xscale=.25,yscale=.27,line width=0.8pt] (.5,-0.5) node {$\scriptstyle{1}$};  \path[xscale=.25,yscale=.27,line width=0.8pt] (1.5,-0.5) node {$\scriptstyle{2}$};}} in Fig. \ref{Fig:atoms}).
\end{remark}

\begin{theorem}\label{Th:tight_frame}
For $\gamma \vdash n$ and $\lambda \in \Lambda_\gamma$, 
the collection of atoms $\varPhi_{\gamma,\lambda}$ defined in \eqref{Eq:frame_elements} is the union of of $\kappa_{\gamma,\lambda}$ orthogonal tight Parseval frames and, as such, is a  tight Parseval frame for $Z_{\gamma, \lambda}$. The set of atoms  $\D:=\bigcup_{\gamma \vdash n} \bigcup_{\lambda\in \Lambda_\gamma} \varPhi_{\gamma,\lambda}$ is a tight Parseval frame for $\Rbb[\SS_n]$.
\end{theorem}

\begin{proof}
Let ${\bf v} \in V_\gamma^\ast \subseteq \Rbb[\Pi_\gamma]$ be a unit  Laplacian eigenvector of $\PP_\gamma$ of eigenvalue $\lambda$ and symmetry type $\gamma$. For $\pi \in\Pi_\gamma$, the lifted vector ${\bf w}={\bf B}_\pi {\bf v} \in \Rbb[\SS_n]$ is a Laplacian eigenvector of $\PP_n$ of eigenvalue $\lambda$ by Prop.~\ref{Pr:evec_lift}, and 
${\bf w} \in W_\gamma$ (i.e., it has symmetry type $\gamma$) by Prop.~\ref{Pr:lift_isotypic}. Define the left $\SS_n$-module,
\begin{equation}\label{eq:group_frame}
V_{\gamma,\vb} := \Rbb[\SS_n]\w = \Rbb[\SS_n] B_\pi \vb, \qquad \text{which is spanned by $\left\{\rho_L(\sigma) \w\right\}_{\sigma \in \SS_n}$},
\end{equation}
where $\rho_L(\sigma)$ is the matrix of $\sigma$ in the left regular representation of $\SS_n$. Therefore, $\left\{\rho_L(\sigma) \w\right\}_{\sigma \in \SS_n}$ is a group frame for $V_{\gamma,\vb}$ if we view $\left\{\rho_L(\sigma) \w\right\}_{\sigma \in \SS_n}$ as a multiset of size $n!$ with (potentially many) repetitions. 

Since the left and right actions of $\SS_n$ commute and the Laplacian $\L$ is constructed \eqref{eq:Laplacian_representation}  using the right representation, $V_{\gamma,\vb}$ is a space of Laplacian eigenvectors of $\PP_n$ of eigenvalue $\lambda$.  Moreover, by the double centralizer theorem (see, for example, \cite[Thm.~5.18.1]{etingof2011introduction}), $W_\gamma \cong V_\gamma \otimes V_\gamma^\ast$, as an $(\SS_n,\SS_n)$ bimodule, where $V_\gamma$ and $V_\gamma^\ast$ are  irreducible left and right $\SS_n$ modules indexed by $\gamma$, respectively.  Using this isomorphism, we write ${\bf w}= w_1 \otimes w_2 \in V_\gamma \otimes V_\gamma^\ast$, and then $V_{\gamma,\vb} =  \Rbb[\SS_n] \w = (\Rbb[\SS_n] w_1) \otimes w_2 \cong V_\gamma \otimes w_2  \cong V_\gamma$,  isomorphic as left $\SS_n$-modules. Therefore, $V_{\gamma,\vb}$ is an irreducible $\SS_n$-module and, by \cite[Thm~10.5]{waldron2018introduction}, $\left\{\rho_L(\sigma) \w\right\}_{\sigma \in \SS_n}$ is a tight group frame.

We remove some of the repetition in  $\left\{\rho_L(\sigma) \w\right\}_{\sigma \in \SS_n}$ by using Prop \ref{Pr:lift_symmetry}:
\begin{equation}\label{eq:tgf}
 \{ \rho(\sigma) \w \mid \sigma \in \SS_n\} = \{ \rho(\sigma) {\bf B}_\pi {\bf v} \mid \sigma \in \SS_n\} = \{ {\bf B}_{\sigma(\pi)} {\bf v} \mid \sigma \in \SS_n\} = \{ {\bf B}_\mu {\bf v} \mid \mu \in \Pi_\gamma\},
\end{equation}
where the third equality comes from the fact that $\SS_n$ acts transitively on $\Pi_\gamma$. Eqn. \eqref{eq:tgf} tells us that $V_{\gamma,\vb}$ is independent of the specific lifting $\pi$ that we use.  Let $\SS_\pi = \{\sigma \in \SS_n \mid \sigma(\pi) = \pi\}$ be the stabilizer subgroup of $\pi \in \Pi_\gamma$, and for $\mu \in \Pi_\gamma$, let $\tau_\mu \in \SS_n$ be a permutation such that $\tau_\mu(\pi) = \mu$.  Then the left coset $\tau_\mu \SS_\pi$ consists of all permutations that send $\pi$ to $\mu$. By  \cite[Thm~10.5]{waldron2018introduction}, for any ${\bf f} \in \Rbb[\SS_n]$,  we have
\begin{align*}
{\bf f} & = \frac{d_\gamma}{n!} \frac{1}{ \langle B_\pi \vb, B_\pi \vb \rangle} \sum_{\sigma \in \SS_n} \langle {\bf f}, \rho_L(\sigma) B_\pi \vb \rangle \rho_L(\sigma) B_\pi \vb
=
 \frac{d_\gamma}{n!} \frac{1}{\langle B_\pi \vb, B_\pi \vb \rangle} \sum_{\sigma \in \SS_n} \langle {\bf f}, B_{\sigma(\pi)} \vb \rangle B_{\sigma(\pi)} \vb \\
&  =
 \frac{d_\gamma}{n!} \frac{1}{ \langle B_\pi \vb, B_\pi \vb \rangle} \sum_{\mu \in \Pi_\gamma}  \sum_{\sigma \in \SS_\pi} \langle {\bf f}, B_{\tau_\mu \sigma(\pi)} \vb \rangle B_{\tau_\mu \sigma(\pi)} \vb
 =
  \frac{d_\gamma}{n!} \sum_{\mu \in \Pi_\gamma}  \langle {\bf f}, B_{\mu} \vb \rangle B_{\mu} \vb,
\end{align*}
where the last equality follows from $|\SS_\pi| = n!/m_\gamma$ and $\langle B_\pi \vb, B_\pi \vb \rangle = \vb^\top B_\pi^\top B_\pi \vb  = n!/m_\gamma \vb^\top \vb = n!/m_\gamma$. It follows that $\Phi_{\gamma,\vb} := \{ \sqrt{d_\gamma/n!} B_\mu \vb \mid \mu \in \Pi_\gamma\}$ is a tight Parseval frame for $V_{\gamma,\vb}$, again viewing $\Phi_{\gamma,\vb}$ as a multiset that can have repetition (as seen in Lem. \ref{lemma:reduce} below).

Now suppose that  $\{\vb_i^\ast\}_{i=1}^{d_\gamma}$ is an orthonormal basis for $V_\gamma^\ast \subseteq \Rbb[\Pi_\gamma]$. Then $\Phi_{\gamma,\vb_i}$ is a tight Parseval frame for $V_{\gamma,\vb_i^\ast}$ for each $1 \le i \le d_\gamma$.   Moreover,  $V_{\gamma,\vb_i}$ and $V_{\gamma,\vb_j}$ are orthogonal subspaces for $i \not= j$. To see this, identify $V_{\gamma,\vb_i^\ast}$ with $V_{\gamma} \otimes \vb_i^\ast$ in $W_\gamma \cong V_\gamma \otimes V_\gamma^\ast$. 
Since $V_\gamma$ and $V_\gamma^\ast$ are irreducible, they carry a unique (up to scalar multiple) $\SS_n$-invariant inner product, and therefore   an $\SS_n$-invariant inner product on $W_\gamma \subseteq \Rbb[\SS_n]$ equals $\langle \w_1 \otimes \vb_i, \w_2 \otimes \vb_j \rangle = \langle \w_1,  \w_2\rangle \langle  \vb_i,  \vb_j \rangle$, up to a scalar. The orthogonality of the spaces $\{V_{\gamma,\vb_i^\ast}\}_{i=1}^{d_\gamma}$ follows from the orthogonality of $\{\vb_i^\ast\}_{i=1}^{d_\gamma}$.

Finally, if $\lambda \in \Lambda_\gamma$ and $\left\{{\bf v}_{\gamma,\lambda,k}\right\}_{k=1}^{\kappa_{\gamma,\lambda}}$ is an orthonormal basis  for the graph Laplacian eigenspace of $V_{\gamma}^\ast \subseteq \Rbb[\Pi_\gamma]$ of eigenvalue $\lambda$, then $\Phi_{\gamma,\lambda} = \cup_{k = 1}^{\kappa_{\gamma,\lambda}} \Phi_{\gamma,{\bf v}_{\gamma,\lambda,k}}$ is a union of orthogonal tight Parseval frames, and therefore is a tight Parseval frame for $Z_{\gamma,\lambda}$. Furthermore, $\Phi_\gamma = \cup_{\lambda \in \Lambda_\gamma} \Phi_{\gamma,\lambda}$ a union of orthogonal tight Parseval frames, and therefore is a tight Parseval  frame for the isotypic component $W_\gamma$.  Since isotypic components are orthogonal (e.g., \cite[Thm.~10.7]{waldron2018introduction}),   $\mathcal{D} = \cup_{\gamma \vdash n} \Phi_\gamma$ is a tight Parseval frame for $\Rbb[\SS_n]$.
\qed
\end{proof}

\begin{remark} ({\it Frames for subspaces of $\Rbb[\SS_n]$})
From  the proof of Thm. \ref{Th:tight_frame} we see that the proposed method can be leveraged to construct a tight Parseval frame for any union of the $Z_{\gamma,\lambda}$ subspaces, not just $\Rbb[\SS_n]$. This property can be beneficial for computational reasons, and is explored further in Sec. \ref{Se:subsampling}.  In particular, for $\gamma \vdash n, \lambda \in \Lambda_\gamma$, and $1 \le k \le \kappa_{\gamma,\lambda}$, we have:

\vspace{-.1in}
\begin{enumerate}
\item    $\Phi_{\gamma,\lambda,k} = \{ \boldsymbol{\varphi}_{\gamma,\lambda,k,\pi}\}_{\pi \in \Pi_\gamma}$ 
is a tight Parseval  frame for $\Rbb[\SS_n]  \boldsymbol{\varphi}_{\gamma,\lambda,k,\pi} \cong V_\gamma$ (for any $\pi \in \Pi_\gamma$).
\item  $\Phi_{\gamma,\lambda} = \{ \boldsymbol{\varphi}_{\gamma,\lambda,k,\pi}\}_{1 \le k \le \kappa_{\gamma,\lambda}, \pi \in \Pi_\gamma}$ 
is a tight Parseval frame for the subspace $Z_{\gamma, \lambda}$. 
\item  $\Phi_{\gamma} = \{ \boldsymbol{\varphi}_{\gamma,\lambda,k,\pi}\}_{\lambda \in \Lambda_\gamma,1 \le k \le \kappa_{\gamma,\lambda}, \pi \in \Pi_\gamma}$ 
is a tight Parseval frame for the isotypic component $W_{\gamma}$.
\end{enumerate}
\end{remark}

\begin{remark} ({\it Equal norms})
Since the frame $\Phi_{\gamma,\lambda,k} = \{ \boldsymbol{\varphi}_{\gamma,\lambda,k,\pi}\}_{\pi \in \Pi_\gamma}$ is generated by a group action, the frame vectors have equal norms. In fact, for any $\boldsymbol{\varphi}_{\gamma,\lambda,k,\pi} \in \Phi_{\gamma,\lambda,k}$, 
$$
\langle \boldsymbol{\varphi}_{\gamma,\lambda,k,\pi}, \boldsymbol{\varphi}_{\gamma,\lambda,k,\pi}  \rangle = 
c_{\gamma}^2 \langle {\bf B}_{\pi} {\bf v}_{\gamma, \lambda, k},  {\bf B}_{\pi} {\bf v}_{\gamma, \lambda, k} \rangle
=  {\bf v}_{\gamma, \lambda, k}^\top {\bf B}_{\pi}^\top  {\bf B}_{\pi} {\bf v}_{\gamma, \lambda, k} = c_\gamma^2 \frac{n!}{m_\gamma}  {\bf v}_{\gamma, \lambda, k}^\top {\bf v}_{\gamma, \lambda, k}  = c_\gamma^2 \frac{n!}{m_\gamma} = \frac{d_\gamma}{m_\gamma}.
$$

\end{remark}

\begin{remark} ({\it Frame angles})
The Gram matrix (or Gramian) of $\Phi_{\gamma,\lambda,k}$ is the matrix of inner products,
\begin{equation*}\label{eq:gram}
{\bf G}_{\Phi_{\gamma,\lambda,k}} = [ \langle \varphi, \psi \rangle ]_{\varphi,\psi \in \Phi_{\gamma,\lambda,k}}.
\end{equation*}
Since $\Phi_{\gamma,\lambda,k}$ is a tight Parseval frame for the irreducible module $\Rbb[\SS_n]  {\boldsymbol \varphi}_{\gamma,\lambda,k,\pi} \cong V_\gamma$, the Gram matrix ${\bf G}_{\Phi_{\gamma,\lambda,k}}$ equals the $m_\gamma \times m_\gamma$ matrix that projects $\Rbb[\Pi_\gamma]$ onto $V_\gamma$ (see 
\cite[Cor. 10.2, Thm. 13.1]{waldron2018introduction}). This projection has a well-known description (e.g., \cite[(13.19)]{waldron2018introduction}) as the matrix of the following operator in the center of the group algebra $\Rbb[\SS_n]$:
\begin{equation}\label{eq:isotypic_projector}
p_\gamma = \frac{d_\gamma}{n!} \sum_{\sigma \in \SS_n} \chi_\gamma(\sigma^{-1}) \sigma,
\end{equation}
where the  coefficients $\chi_\gamma(\sigma^{-1}) $ are given by the irreducible character $\chi_\gamma$ corresponding to $\gamma$.
Applying $p_\gamma$ to the basis $\{\e_\mu\}_{\mu \in \Pi_\gamma}$ of $\Rbb[\Pi_\gamma]$ gives
$$
p_\gamma(\e_\mu) = \frac{d_\gamma}{n!} \sum_{\sigma \in \SS_n} \chi_\gamma(\sigma^{-1}) e_\mu \sigma
= \frac{d_\gamma}{n!} \sum_{\sigma \in \SS_n} \chi_\gamma(\sigma^{-1}) e_{\sigma^{-1} (\mu)} 
= \frac{d_\gamma}{n!} \sum_{\sigma \in \SS_n} \chi_\gamma(\sigma) e_{\sigma(\mu)}.
$$
For $\pi,\mu \in \Pi_\gamma$,  the entry $\langle {\boldsymbol \varphi}_{\gamma,\lambda,k,\pi}, {\boldsymbol \varphi}_{\gamma,\lambda,k,\mu}\rangle$ of ${\bf G}_{\Phi_{\gamma,\lambda,k}}$ is the coefficient of $\e_\pi$ in $p_\gamma(\e_{\mu})$; namely,
\begin{equation}\label{eq:gram}
\langle{ \boldsymbol \varphi}_{\gamma,\lambda,k,\pi}, {\boldsymbol \varphi}_{\gamma,\lambda,k,\mu}\rangle = \frac{d_\gamma}{n!} \sum_{\sigma \in {\cal V}_{\pi,\mu} } \chi_\gamma(\sigma),
\end{equation}
where ${\cal V}_{\pi,\mu} = \{ \sigma \in \SS_n \mid \sigma(\mu) = \pi\}$ is one of the equivalence classes of Def.~\ref{Def:equivalencerel} and is a right coset of the stabilizer ${\cal V}_{\pi,\pi} = \{ \sigma \in \SS_n \mid \sigma(\pi) = \pi\}$. The characters of the symmetric group are integers, so, $ \sum_{\sigma \in {\cal V}_{\pi,\mu} } \chi_\gamma(\sigma) \in \Zbb$. 
Frame vectors corresponding to different values of $\lambda$ and $k$ are orthogonal (as seen in the proof of Thm.~\ref{Th:tight_frame}). Therefore the Gram matrix ${\bf G}_{\Phi_\gamma} = \bigoplus_{\lambda \in \Lambda_\gamma} \bigoplus_{k = 1}^{\kappa_{\gamma,\lambda}} {\bf G}_{\Phi_{\gamma,\lambda,k}}$ for the isotypic component $W_\gamma$ is the direct sum of $d_\gamma$ matrices, each of the form \eqref{eq:gram}.  \end{remark}

For shapes $\gamma$ with multiple blocks of the same size, there is redundancy in the frame $\varPhi_{\gamma,\lambda}$ that can be removed. Let $\gamma = [\gamma_1, \ldots, \gamma_\ell]$ and suppose that two parts of $\gamma$ are equal; that is, $\gamma_i = \gamma_j$.  For $\pi \in \Pi_\gamma$, let $\pi' \in \Pi_\gamma$ be the ordered set partition obtained from $\pi$ by swapping row $i$ and row $j$.  For example, if $\gamma = [4,3,3,2]$, then
\begin{equation}\label{eq:rowswap}
\pi = \begin{array}{c}
\begin{tikzpicture}[xscale=.33,yscale=.29,line width=0.8pt] 
\draw (0,0) rectangle (4,1); 
\path (.5,0.5) node {{\scriptsize $1$}}; \path (1.5,0.5) node {{\scriptsize $4$}}; \path (2.5,0.5) node {{\scriptsize $7$}}; \path (3.5,0.5) node {{\scriptsize $12$}}; 
\draw (0,-1) rectangle (3,0); 
\path (.5,-0.5) node {{\scriptsize $2$}}; \path (1.5,-0.5) node {{\scriptsize $5$}}; \path (2.5,-0.5) node {{\scriptsize $8$}}; 
\draw (0,-2) rectangle (3,-1); 
\path (.5,-1.5) node {{\scriptsize $3$}};  \path (1.5,-1.5) node {{\scriptsize $6$}};  \path (2.5,-1.5) node {{\scriptsize $11$}};  
\draw (0,-3) rectangle (2,-2); 
\path (.5,-2.5) node {{\scriptsize $9$}};  \path (1.5,-2.5) node {{\scriptsize $10$}}; 
\end{tikzpicture}
\end{array}
\qquad \hbox{and} ~ \qquad
\pi' = \begin{array}{c}
\begin{tikzpicture}[xscale=.33,yscale=.29,line width=0.8pt] 
\draw (0,0) rectangle (4,1); 
\path (.5,0.5) node {{\scriptsize $1$}}; \path (1.5,0.5) node {{\scriptsize $4$}}; \path (2.5,0.5) node {{\scriptsize $7$}}; \path (3.5,0.5) node {{\scriptsize $12$}}; 
\draw (0,-1) rectangle (3,0); 
\path (.5,-1.5) node {{\scriptsize $2$}}; \path (1.5,-1.5) node {{\scriptsize $5$}}; \path (2.5,-1.5) node {{\scriptsize $8$}}; 
\draw (0,-2) rectangle (3,-1); 
\path (.5,-0.5) node {{\scriptsize $3$}};  \path (1.5,-0.5) node {{\scriptsize $6$}};  \path (2.5,-0.5) node {{\scriptsize $11$}};  
\draw (0,-3) rectangle (2,-2); 
\path (.5,-2.5) node {{\scriptsize $9$}};  \path (1.5,-2.5) node {{\scriptsize $10$}}; 
\end{tikzpicture}
\end{array}
\end{equation}
satisfy this condition.
The liftings from $V_\gamma \subseteq \Rbb[\Pi_\gamma]$ to $\Rbb[\SS_n]$ via $\pi$ and $\pi'$ are related according to the following lemma.

\begin{lemma}\label{lemma:reduce}
 Let $\gamma = [\gamma_1, \ldots, \gamma_\ell] \vdash n$ with $t= \gamma_i = \gamma_j$ and let $\pi, \pi' \in \Pi_\gamma$ be equal after swapping rows $i$ and $j$ in $\pi$. If $\vb \in V_\gamma \subseteq \Rbb[\Pi_\gamma]$, then 
${\bf B}_\pi \vb = (-1)^t {\bf B}_{\pi'} \vb$.
\end{lemma}

\begin{proof}
Let $\gamma = [\gamma_1, \ldots, \gamma_\ell]$ with $t =\gamma_i = \gamma_j$. For $\rho \in \Pi_\gamma$, let $\rho'$ be the same set partition as $\rho$ except with rows $i$ and $i+1$ swapped, as illustrated in \eqref{eq:rowswap}.  Suppose that $\vb \in \Rbb[\Pi_\gamma]$ is expressed in the canonical basis as $\vb = \sum_{\rho \in \Pi_\gamma} c_\rho \e_\rho$. Suppose further that for each $\rho$ we have $c_\rho = (-1)^t c_{\rho'}$, which we call the symmetry property.  Then, 
$$
{\bf B}_{\pi} \vb= \sum_{\rho \in \Pi_\gamma} c_\rho {\bf B}_\pi \e_\rho 
=  \sum_{\rho \in \Pi_\gamma} c_\rho \sum_{\sigma \in\SS_n  \atop \sigma(\rho) = \pi }   \e_\sigma
= (-1)^t \sum_{\rho \in \Pi_\gamma} c_{\rho'} \sum_{\sigma \in\SS_n  \atop \sigma(\rho) = \pi }   \e_\sigma
= (-1)^t \sum_{\rho \in \Pi_\gamma} c_{\rho'} \sum_{\sigma \in\SS_n  \atop \sigma(\rho') = \pi' }   \e_\sigma
= (-1)^t {\bf B}_{\pi'} {\bf v},
$$
since $\sigma(\rho) = \pi$ if and only if $\sigma(\rho') = \pi'$. Thus, the proposition is proved if we show that every $\vb$ in the submodule $V_\gamma \subseteq \Rbb[\Pi_\gamma]$ has the symmetry property  $c_\rho = (-1)^t c_{\rho'}$.

The submodule $V_\gamma^\ast \subseteq \Rbb[\Pi_\gamma]$ is spanned by the following set of vectors, called polytabloids (see \cite[2.3]{sagan2013symmetric}),
\begin{equation}\label{eq:polytabloids}
q_\pi = \sum_{\beta \in C_\pi} \mathsf{sign}(\beta) \e_\pi \beta = \sum_{\beta \in C_\pi} \mathsf{sign}(\beta) \e_{\beta^{-1} (\pi)}, \qquad \pi \in \Pi_\gamma,
\end{equation}
where $C_\pi \subseteq \SS_n$ is the column group of $\pi$, that is, the permutations that stabilize the columns of $\pi$, and $\mathsf{sign}(\beta)$ is the sign of the permutation $\beta$. Let $\tau_\pi \in C_\pi$ be the permutation that is the product of the $t$ disjoint transpositions (not necessarily adjacent) that swap an entry in row $i$ of $\pi$ with the corresponding entry in row $j$.  For example, in \eqref{eq:rowswap}, $\tau_\pi = (2,3)(5,6)(8,11)$.  Then  $\mathsf{sign}(\tau_\pi) = (-1)^t$ and $\pi' = \tau_\pi(\pi)$.
 
Since $\tau_\pi \in C_\pi$, we have $\tau_\pi^{-1} C_\pi \tau_\pi = C_\pi$ and $\mathsf{sign} (\tau_\pi^{-1} \beta \tau_\pi ) = \mathsf{sign} (\beta)$. Moreover, $\tau_\pi^{-1} = \tau_\pi$, so we have
\begin{align*}
q_{\pi} \tau_\pi 
& = \sum_{\beta \in C_\pi} \mathsf{sign}(\beta) \e_\pi \beta \tau_\pi 
= \sum_{\beta \in C_\pi} \mathsf{sign}(\beta) \e_\pi \tau_\pi \tau_\pi^{-1} \beta \tau_\pi 
= \sum_{\beta \in C_\pi} \mathsf{sign}(\beta) \e_{\tau_\pi(\pi)} \tau_\pi^{-1} \beta \tau_\pi 
= q_{\tau_\pi(\pi)}
= q_{\pi'},
\\
q_{\pi} \tau_\pi 
&= \sum_{\beta \in C_\pi} \mathsf{sign}(\beta) \e_\pi \beta \tau_\pi 
= \sum_{\eta \in C_\pi} \mathsf{sign}(\tau_\pi  \eta) \e_\pi \eta
= \mathsf{sign}(\tau) \sum_{\eta \in C_\pi} \mathsf{sign}( \eta) \e_\pi \eta
=  (-1)^t q_{\pi},
\end{align*}
and thus $q_{\pi} = (-1)^t q_{\pi'}$.  It follows that $q_{\pi}$ has the symmetry property $c_\rho = (-1)^t c_{\rho'}$.  Since this is true of each vector of the spanning set $q_\pi, \pi \in \Pi_\gamma$, it is true for all of $V_\gamma^\ast$, and the result is proved.
\qed
\end{proof}

Lem. \ref{lemma:reduce} tells us that, in the case where $\gamma$ has repeated parts, many of the atoms in \eqref{Eq:frame_elements} are identical or are the negatives of others. We then can lift fewer vectors to generate a tight Parseval frame for $Z_{\gamma, \lambda}$, which leads to a more computationally efficient implementation without sacrificing any interpretability. Define
$z_{\gamma}:=\frac{m_\gamma}{\prod_i k_i!}$, where  if $\gamma=[\gamma_1,\ldots,\gamma_{\ell}]$, $k_i$ is the multiplicity of $i$ in $\gamma$. For example, if  $\gamma=[4,2,2,2,1]$, $m_\gamma=\frac{11!}{4!2!2!2!1!}$ and $z_\gamma=\frac{m_\gamma}{3!}$,  
as $i=2$ appears three times in $\gamma$. Identifying ordered set partitions in $\Pi_{\gamma}$ that feature the same groupings of candidates yields a smaller set of $z_\gamma$ (unordered) set partitions, which we denote by $\bar{\Pi}_\gamma$. For example, the ordered set partitions $\{\{1,2,3,4\},\{5,6\},\{7,8\},\{9,10\},\{11\}\}$, $\{\{1,2,3,4\},\{7,8\},\{5,6\},\{9,10\},\{11\}\}$, and four others are all identified to a single set partition in $\bar{\Pi}_{[4,2,2,2,1]}$.  For $\bpi \in \bar{\Pi}_\gamma$,
define ${\boldsymbol \varphi}_{\gamma,\lambda,k,\bpi} := \bar{c}_{\gamma}{\bf B}_{\bpi} {\bf v}_{\gamma, \lambda, k}$, and define the reduced frame 
\begin{align}\label{Eq:frame_elements2}
\bar{\varPhi}_{\gamma,\lambda}:=\bigcup_{k=1}^{\kappa_{\gamma,\lambda}} \bigcup_{\bpi \in \overline{\Pi}_\gamma} {\boldsymbol \varphi}_{\gamma,\lambda,k,\bpi}:=\bigcup_{k=1}^{\kappa_{\gamma,\lambda}} \bigcup_{\bpi \in \bar{\Pi}_\gamma} \bar{c}_{\gamma}{\bf B}_{\bpi} {\bf v}_{\gamma, \lambda, k},
\end{align}
where the constant $\bar{c}_\gamma:=\sqrt{\frac{d_\gamma m_\gamma}{n! z_\gamma}}$.  In Thm. \ref{Th:reduced_tight_frame} we show that $\bar{\varPhi}_{\gamma,\lambda}$ remains a tight Parseval frame for $Z_{\gamma,\lambda}$.

\begin{theorem}\label{Th:reduced_tight_frame}
For $\gamma \vdash n$ and $\lambda \in \Lambda_\gamma$, 
the collection of atoms $\bar{\varPhi}_{\gamma,\lambda}$ defined in \eqref{Eq:frame_elements2} is a tight Parseval frame for $Z_{\gamma, \lambda}$, and the set of atoms  $\bar{\D}:=\bigcup_{\gamma \vdash n} \bigcup_{\lambda\in \Lambda_\gamma} \bar{\varPhi}_{\gamma,\lambda}$ is a tight Parseval frame for $\Rbb[\SS_n]$.
\end{theorem}

\begin{proof} By  Lem.  \ref{lemma:reduce}, $\Phi_{\gamma,\lambda} = \bigcup_{i = 1}^{m_\gamma/z_\gamma} \pm \bar{\Phi}_{\gamma,\lambda}$.
Therefore,  from Thm. \ref{Th:tight_frame}, for any ${\bf f} \in \Rbb[\SS_n]$, we have
$$
{\bf f} = \sum_{\phi \in \Phi_{\gamma,\lambda}} \langle {\bf f}, \phi \rangle \phi 
=  \sum_{i = 1}^{m_\gamma/z_\gamma} \sum_{\pm\phi \in \pm\bar{\Phi}_{\gamma,\lambda}} \langle {\bf f}, \pm\phi \rangle (\pm\phi)
=  \sum_{i = 1}^{m_\gamma/z_\gamma} \sum_{\phi \in \bar{\Phi}_{\gamma,\lambda}} \langle {\bf f}, \phi \rangle \phi
= \frac{m_\gamma}{z_\gamma}\sum_{\phi \in \bar{\Phi}_{\gamma,\lambda}} \langle {\bf f}, \phi \rangle \phi.
$$
By observing that $\bar{c}_\gamma=\sqrt{\frac{d_\gamma m_\gamma}{n! z_\gamma}} = \sqrt{\frac{m_\gamma}{z_\gamma}} c_\gamma,$ we see that $\bar{\Phi}_{\gamma,\lambda}$ is a tight Parseval frame for $Z_{\gamma,\lambda}$. 
\qed
\end{proof}

\clearpage

\begin{figure}
\hspace{.2in}
\begin{minipage}{.2\linewidth}
\centering
\includegraphics[width=\linewidth,page=1]{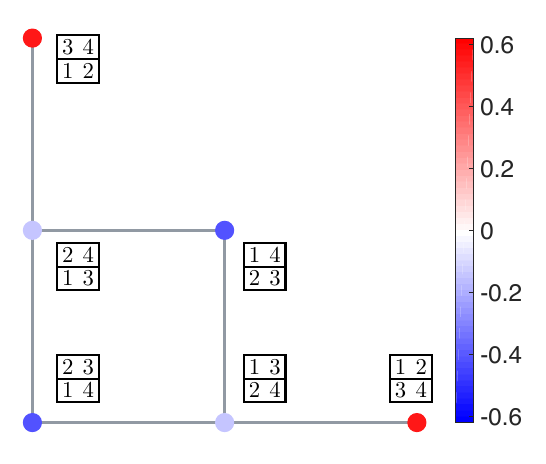}
\vspace{.005in}
\end{minipage}
\hspace{.2in}
\begin{minipage}{.6\linewidth}
\centering
\includegraphics[width=.315\linewidth,page=1]{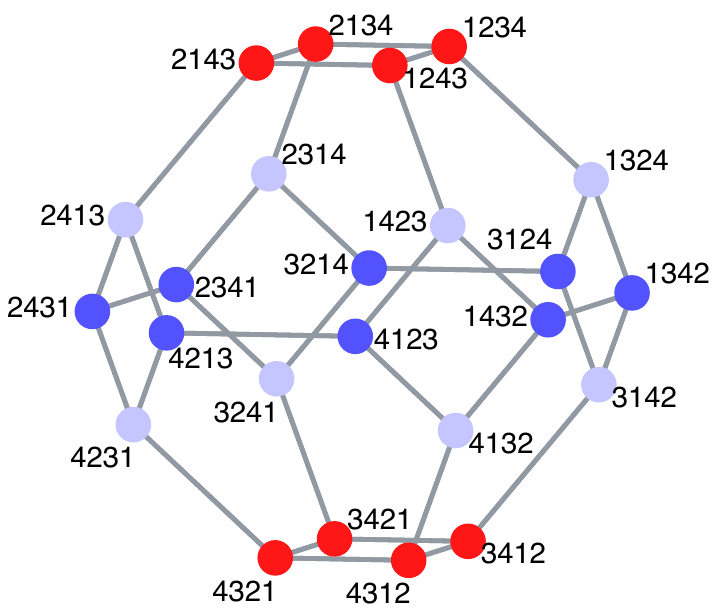}
\includegraphics[width=.315\linewidth,page=2]{figures/atomsb} 
\includegraphics[width=.315\linewidth,page=3]{figures/atomsb}
\\
\vspace{-.15in}
${\small \underbrace{\hspace{3.4in}}_{\hbox{Tight frame for }{U}_{1.2679}}}$
\end{minipage}
\begin{minipage}{.08\linewidth}
\centering
\includegraphics[height=1in]{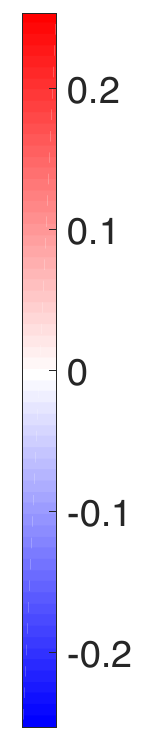}
\vspace{.2in}
\end{minipage} \\
\vspace{.2in}

\hspace{.2in}
\begin{minipage}{.2\linewidth}
\centering
\includegraphics[width=\linewidth,page=2]{diagrams/schreiereigswithtableaux}
\vspace{.005in}
\end{minipage}
\hspace{.2in}
\begin{minipage}{.6\linewidth}
\centering
\includegraphics[width=.315\linewidth,page=4]{figures/atomsb}
\includegraphics[width=.315\linewidth,page=5]{figures/atomsb} 
\includegraphics[width=.315\linewidth,page=6]{figures/atomsb}
\\
\vspace{-.15in}
${\small \underbrace{\hspace{3.4in}}_{\hbox{Tight frame for }{U}_{4.7321}}}$
\end{minipage}
\begin{minipage}{.08\linewidth}
\centering
\includegraphics[height=1in]{figures/atoms_colorbar}
\vspace{.2in}
\end{minipage}
\caption{Left: The graph Laplacian eigenvectors ${\bf v}_{[2,2],\lambda}$ of the Schreier graph $\PP_{[2,2]}$ associated with eigenvalues $\lambda=1.2679$ (top) and $\lambda=4.7321$ (bottom). Right: The dictionary atoms $\{{\boldsymbol \varphi}_{\gamma,\lambda,z}\} =\{\bar{c}_\gamma{\bf B}_{\pi_z} {\bf v}_{[2,2],\lambda}\}$ are generated by lifting the Schreier Laplacian eigenvectors according to the three different ordered set
partitions $\pi_z \in \bar{\Pi}_{[2,2]}$ shown in Fig. \ref{Fig:eq_part}. The resulting sets of three atoms form tight frames for the corresponding spaces $Z_{[2,2],\lambda}$. In this case, 
$Z_{[2,2],1.2679}=U_{1.2679}$ and $Z_{[2,2],4.7321}=U_{4.7321}$, since the eigenvalues $\lambda=1.2679$ and $\lambda=4.7321$ only appear in the [2,2] irreducible.}\label{Fig:atoms}
\vspace{0.1in}
\hrule height 1.5pt
\vspace{-.1in}
\end{figure}

In Fig. \ref{Fig:atoms}, for $\gamma=[2,2]$, we show two different eigenvectors of the Schreier graph $\PP_\gamma$ lifted back to the permutahedron according to the $z_\gamma=3$ different ordered set
partitions in Fig. \ref{Fig:eq_part}, yielding tight frames $\bar{\varPhi}_{[2,2],1.2679}$ and $\bar{\varPhi}_{[2,2],4.7321}$ with three vectors each for the two-dimensional spaces $Z_{[2,2],1.2679}$ and $Z_{[2,2],4.7321}$, respectively.  \emph{An important point of emphasis here is that compared to the orthonormal bases for the same spaces in Fig. \ref{Fig:ortho_basis}, the frame vectors in Fig. \ref{Fig:atoms} maintain interpretable symmetry properties.} 

In this case, were we to include all $m_\gamma=6$ liftings generated from the ordered set partitions in $\Pi_\gamma$, the resulting frames $\varPhi_{\gamma,\lambda}$ would have two copies of each of these three atoms, all scaled by the constant factor $\frac{c_\gamma}{\bar{c}_\gamma}=\sqrt{\frac{z_\gamma}{m_\gamma}}$. In cases where the repeated parts have odd length (e.g., tight frames for $Z_{[8,1,1],\lambda}$ generated by lifting the eigenvector in Fig. \ref{fig:811_schreier}(b)), the removed atoms would have the opposite sign in each entry. Unless specified, our default in the remainder of the paper is to use the tight frames with fewer elements defined in \eqref{Eq:frame_elements2}.

Returning to the three objectives outlined at the beginning of this section, the proposed dictionary atoms comprise a tight Parseval frame, as shown in Thm. \ref{Th:tight_frame} and Thm. \ref{Th:reduced_tight_frame}, and therefore satisfy the first objective of preserving the energy of the signal. Each atom ${\boldsymbol \varphi}_{\gamma,\lambda,k,\bpi}$ belongs to the space $Z_{\gamma, \lambda} = W_\gamma \cap U_\lambda$ and therefore inherits known symmetry and smoothness properties from $W_\gamma$ and $U_\lambda$, respectively. In the next section, we investigate the interpretation of specific frame analysis coefficients 
$\langle{\bf f},{\boldsymbol \varphi}_{\gamma,\lambda,k,\bpi}\rangle$ in the context of several example data sets. The third objective of identifying methods to efficiently compute these inner products is the focus of Sec. \ref{Se:comp}, and this computational question actually gives rise to additional interesting theoretical questions.

\clearpage
%%%%%%%%%%%%%%%%%%%%%%%%%%%%%%%%%%%%%%%%%%%%%%%%%%%%%%%%%%%%%%%%%%%%%%%% 
\section{Interpretation of the Analysis Coefficients} \label{Se:interpretation}

The inner products between a signal on the permutahedron and the atoms of the tight spectral frame $\bar{\D}$ from Thm. \ref{Th:reduced_tight_frame} (or $\D$ from Thm. \ref{Th:tight_frame}), referred to as \emph{analysis coefficients}, 
are useful in identifying structure in voting data, such as popular candidates, polarizing candidates, and clusters of candidates commonly ranked similarly by subgroups of voters (e.g., political parties). 
The analysis coefficients have the form:
\begin{align}\label{Eq:anal_coeff_two}
\alpha_{\gamma,\lambda,k,\bpi}:=\langle {\bf f}, {\boldsymbol \varphi}_{\gamma,\lambda,k,\bpi} \rangle = \langle {\bf f}, \bar{c}_{\gamma}{\bf B}_{\bpi} {\bf v}_{\gamma, \lambda, k}\rangle = \bar{c}_{\gamma} \langle  {\bf B}_{\bpi}^{\top}{\bf f},{\bf v}_{\gamma, \lambda, k}\rangle.
\end{align}
In some instances, it is beneficial for interpretation purposes to view these analysis coefficients as inner products between the signal ${\bf f}$ and the atoms (the first two terms in \eqref{Eq:anal_coeff_two}). All of these signals reside on the permutahedron $\PP_n$ (see Fig. \ref{Fig:atoms} or the second row from the bottom in Fig. \ref{Fig:frame_coeffs} for illustrations of example atoms). In other instances, it is helpful to view the same quantity via the last term in \eqref{Eq:anal_coeff_two}: a constant times an inner product between the signal projected down to the Schreier graph $\PP_\gamma$ in a specific manner, and a Laplacian eigenvector of that Schreier graph.

Recalling that the energy of the signal is equal to the energy of the analysis coefficients, i.e., 
\begin{align}\label{Eq:energy_pres}
\lVert{\bf f}\rVert^2=\sum_{\gamma,\lambda,k,\bpi}  |\langle {\bf f}, {\boldsymbol \varphi}_{\gamma,\lambda,k,\bpi} \rangle|^2,
\end{align}
we  
first 
investigate what information can be garnered from the decomposition of the energy of the analysis coefficients on the right-hand side of \eqref{Eq:energy_pres} (i) across shapes $\gamma$, (ii) across eigenvalues $\lambda$ within a fixed shape $\gamma$, and (iii) across atoms within a shape-eigenvalue pair $\gamma, \lambda$. We then examine the interpretation of specific analysis coefficients.

\subsection{Energy Decomposition}

The squared magnitudes $|\langle {\bf f}, {\boldsymbol \varphi}_{\gamma,\lambda,k,\bpi} \rangle|^2$ in the summand of \eqref{Eq:energy_pres} can be aggregated and plotted in different ways to identify structural patterns in the ranking tallies, ${\bf f}$. First, for each shape $\gamma$, the sum  $\sum_{\lambda,k,\bpi}  |\langle {\bf f}, {\boldsymbol \varphi}_{\gamma,\lambda,k,\bpi} \rangle|^2$ is equal to $\lVert{\bf f}_\gamma \rVert^2$, the energy of the projection of the signal onto the corresponding isotypic component (compare, e.g., the bottom image in Fig. \ref{Fig:isodecomp} to the top table in Fig. \ref{Fig:frame_coeffs}). Each of these quantities can be further decomposed across the eigenspaces associated with the eigenvalues in $\Lambda_\gamma$ via the sums $\lVert{\bf f}_{\gamma,\lambda}\rVert^2=\sum_{k,\bpi}  |\langle {\bf f}, {\boldsymbol \varphi}_{\gamma,\lambda,k,\bpi} \rangle|^2$. For example, as shown in Fig. \ref{Fig:frame_coeffs}, for the 2017 Minneapolis City Council Ward 3 election data ${\bf g}$,  
\begin{align*}
355201.6=\lVert{\bf g}_{\begin{tikzpicture}[scale=.15,line width=1.0pt] 
\draw (0,0) rectangle (1,1); \draw (1,0) rectangle (2,1); \draw (2,0) rectangle (3,1); \draw (0,-1) rectangle (1,0); 
\end{tikzpicture}}\rVert^2
&=\lVert{\bf g}_{\begin{tikzpicture}[scale=.15,line width=1.0pt] 
\draw (0,0) rectangle (1,1); \draw (1,0) rectangle (2,1); \draw (2,0) rectangle (3,1); \draw (0,-1) rectangle (1,0); 
\end{tikzpicture},0.586}\rVert^2
+\lVert{\bf g}_{\begin{tikzpicture}[scale=.15,line width=1.0pt] 
\draw (0,0) rectangle (1,1); \draw (1,0) rectangle (2,1); \draw (2,0) rectangle (3,1); \draw (0,-1) rectangle (1,0); 
\end{tikzpicture},2}\rVert^2
+\lVert{\bf g}_{\begin{tikzpicture}[scale=.15,line width=1.0pt] 
\draw (0,0) rectangle (1,1); \draw (1,0) rectangle (2,1); \draw (2,0) rectangle (3,1); \draw (0,-1) rectangle (1,0); 
\end{tikzpicture},3.414}\rVert^2 \\
&=\sum_{\bpi}  \left|\left\langle {\bf g}, {\boldsymbol \varphi}_{{\begin{tikzpicture}[scale=.15,line width=1.0pt] 
\draw (0,0) rectangle (1,1); \draw (1,0) rectangle (2,1); \draw (2,0) rectangle (3,1); \draw (0,-1) rectangle (1,0); 
\end{tikzpicture}},0.586,\bpi} \right\rangle\right|^2
+\sum_{\bpi}  \left|\left\langle {\bf g}, {\boldsymbol \varphi}_{{\begin{tikzpicture}[scale=.15,line width=1.0pt] 
\draw (0,0) rectangle (1,1); \draw (1,0) rectangle (2,1); \draw (2,0) rectangle (3,1); \draw (0,-1) rectangle (1,0); 
\end{tikzpicture}},2,\bpi} \right\rangle\right|^2
+\sum_{\bpi}  \left|\left\langle {\bf g}, {\boldsymbol \varphi}_{{\begin{tikzpicture}[scale=.15,line width=1.0pt] 
\draw (0,0) rectangle (1,1); \draw (1,0) rectangle (2,1); \draw (2,0) rectangle (3,1); \draw (0,-1) rectangle (1,0); 
\end{tikzpicture}},3.414,\bpi} \right\rangle\right|^2 \\
&=147617.5+192845.1+14739.0.
\end{align*}

Through plots such as those shown in Fig. \ref{Fig:mpls_gftplus} and Fig. \ref{Fig:sushi_gftplus}, we can visualize the full decomposition of energy into shape-eigenvalue pairs. For example, in Fig. \ref{Fig:mpls_gftplus} (or the top row of Fig. \ref{Fig:frame_coeffs}), we see that most of the energy from the 
2017 Minneapolis City Council Ward 3 election data ${\bf g}$ 
shown in Fig. \ref{Fig:signals} 
falls in the spaces $Z_{[4],0}$ (which just conveys information about the total number of voters), $Z_{[3,1],2}$, $Z_{[3,1],0.586}$, and $Z_{[2,2],1.268}$. 
As discussed in Sec. \ref{Se:gsp}, typically occurring rankings are smooth with respect to the underlying permutahedron structure, and we therefore expect to see a decay in the energies $\{\lVert{\bf f}_{\gamma,\lambda}\rVert^2\}$ as $\lambda$ increases, as is the case, e.g., in the sushi data in Fig. \ref{Fig:sushi_gftplus}. While in the examples shown Fig. \ref{Fig:mpls_gftplus} and Fig. \ref{Fig:sushi_gftplus} the bar at each eigenvalue is comprised of a single color representing the corresponding shape, this need not be the case. For example, when $n=6$, the eigenvalue $\lambda=3$ appears in two shapes (
$\begin{tikzpicture}[scale=.15,line width=1.0pt] 
\draw (0,0) rectangle (1,1); \draw (1,0) rectangle (2,1); \draw (2,0) rectangle (3,1); \draw (3,0) rectangle (4,1);  \draw (4,0) rectangle (5,1); \draw (0,-1) rectangle (1,0); 
\end{tikzpicture}$
and 
$\begin{tikzpicture}[scale=.15,line width=1.0pt] 
\draw (0,0) rectangle (1,1); \draw (1,0) rectangle (2,1); \draw (2,0) rectangle (3,1); \draw (3,0) rectangle (4,1);  \draw (0,-1) rectangle (1,0);  \draw (0,-2) rectangle (1,-1);
\end{tikzpicture}$
), and therefore the energy of a signal at this eigenvalue would be comprised of two bars of different colors stacked on top of each other.  Additionally, when the graph Laplacian eigenspace of the Schreier graph $\PP_\gamma$ that is associated with $\lambda$ in $V_{\gamma}$ has dimension greater than one (i.e., $\kappa_{\gamma,\lambda}>1$), the energies of the analysis coefficients resulting from all atoms generated from the basis $\left\{{\bf v}_{\gamma,\lambda,k}\right\}_{k=1}^{\kappa_{\gamma,\lambda}}$ are stacked on top of one another and shown as a single bar. For example, the apparent outliers at $\lambda=8$ and $\lambda=10$ in Fig. \ref{Fig:sushi_gftplus} are only due to the fact that $\kappa_{\gamma,\lambda}>1$ for these shape-eigenvalue pairs, as opposed to some structure of interest in the sushi data; a plot of the energies of the analysis coefficients of a pure noise signal yields similar outliers.
  
We can further decompose the energy in any shape-eigenvalue pair by examining the sequence of energies, $\{|\langle {\bf f}, {\boldsymbol \varphi}_{\gamma,\lambda,k,\bpi} \rangle|^2\}_{1 \le k \le \kappa_{\gamma,\lambda}, \pi \in \bar{\Pi}_\gamma},$ associated with each atom in the frame. In particular, for any shape-eigenvalue pair with a significant amount of energy, we can identify the ordered set
partitions $\bar{\pi}$ used to lift the associated eigenvectors to generate the atoms that yield the inner products with the highest magnitudes. For example, for the sushi data signal ${\bf h}$, Fig. \ref{Fig:sushi_analysis} shows the shape-eigenvalue-lifting triplets associated with largest magnitude analysis coefficients.

 \subsection{Interpretation of Specific Analysis Coefficients}
 
\emph{How can we use the magnitudes of the analysis coefficients to identify structure in the ranked data?} The general methodology for each analysis coefficient is to look at the structure of the corresponding eigenvector on the Schreier graph of the identified shape and the projection of the signal onto that Schreier graph via ${\bf B}_{\bar{\pi}}^{\top}$, as the analysis coefficient is the inner product of those two signals on $P_{\gamma}$ (see \eqref{Eq:anal_coeff_two}).
  One key advantage of this method is that it allows us to create visualizations of high-dimensional data on much lower-dimensional graphs. Some of these eigenvectors are more easily interpretable than others, but we are fortunate that the most interpretable eigenvectors are often the ones associated with the largest magnitude analysis coefficients, particularly for the smooth signals that commonly arise in ranking applications.

\begin{figure}
\centering
\includegraphics[width = 6in]{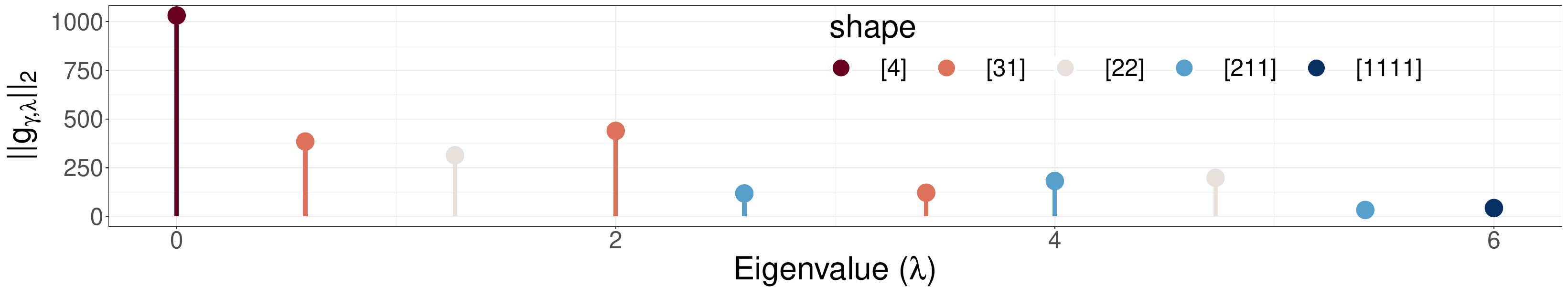}
\caption{Decomposition of the energy of the 2017 Minneapolis City Council Ward 3 election data ${\bf g}$ into the shape-eigenvalue spaces $\{Z_{\gamma,\lambda}\}$.}
\label{Fig:mpls_gftplus}
\end{figure}

\begin{figure}
\centering
\includegraphics[width = 6in]{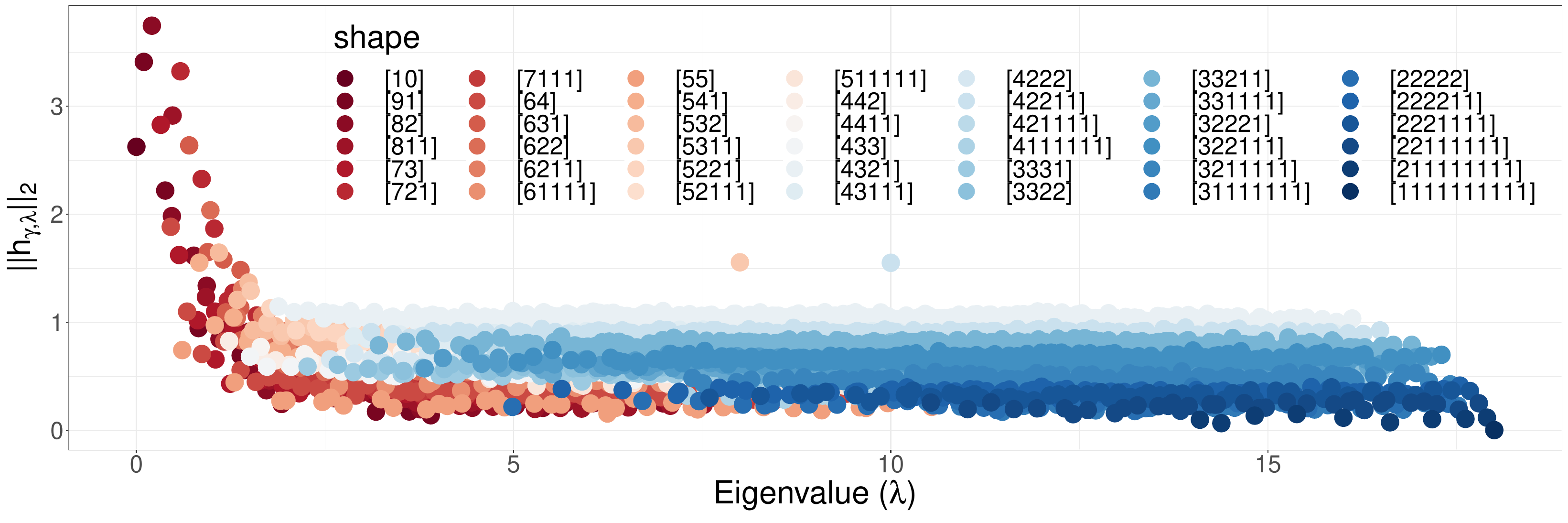}
\caption{The sushi data ${\bf h}$ is a smooth signal on the permutahedron $\PP_{10}$, and the energies of ${\bf h}$ in each shape-eigenvalue space generally decay as the eigenvalues increase, with the exception of two values for which $\kappa_{\gamma,\lambda}>1$ (these outliers are expected and do not reflect any structure of interest in this particular data).}
\label{Fig:sushi_gftplus}
\vspace{0.1in}
\hrule height 1.5pt
\vspace{-.1in}
\end{figure}

\clearpage

\begin{figure}[t]
\centering
\includegraphics[width=.995\linewidth,page=1,trim=35 625 32 100, clip]{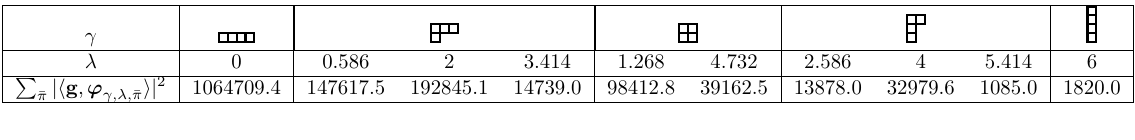}
\includegraphics[width=\linewidth,page=3,trim=27 488 31 100, clip]{figures/MinneapolisDataTable}
\caption{\label{Fig:frame_coeffs}Top: The energies of the projections of ${\bf g}$ onto the spaces $Z_{\gamma,\lambda}$ associated with each 
shape-Laplacian eigenvalue pair. Bottom: Inner products between ${\bf g}$ and a subset of the dictionary atoms of the form ${\boldsymbol \varphi}_{\gamma,\lambda,\bpi}$, each of which is created by lifting the Schreier eigenvector ${\bf v}_\lambda$ back to the permutahedron according to the  partition $\bpi$.} 
\vspace{0.3in}
\hrule height 1.5pt
\vspace{.2in}
\end{figure}

\begin{figure}[b]
\centering
\includegraphics[width=.95\linewidth,page=3,trim=60 420 60 90, clip]{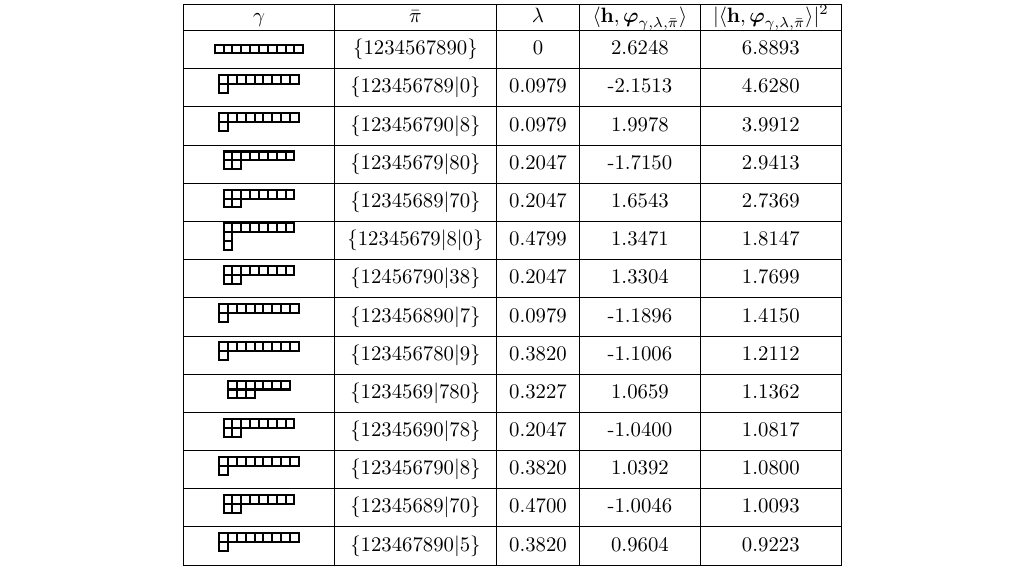}
\caption{\label{Fig:sushi_analysis} The 14 tight spectral frame analysis coefficients with the largest magnitudes for the sushi preference data.} 
\end{figure}

\clearpage

\subsubsection{The Shape $\gamma=[n]$: Number of Votes}

The only eigenvalue in $\Lambda_{[n]}$ is $\lambda=0$, and there is just a single atom associated with the shape $\gamma=[n]$: ${\boldsymbol \varphi}_{[n],0,\{\{1,2,\ldots,n\}\}}=\frac{1}{\sqrt{n!}} {\bf 1}_{n!},$ where  ${\bf 1}_{n!}$ is a constant vector of ones of length $n!$. This atom is a basis for the isotypic component $W_{[n]}$, and the analysis coefficient $\langle {\bf f},{\boldsymbol \varphi}_{[n],0,\{\{1,2,\ldots,n\}\}}\rangle = \hat{f}(0) = \lVert{\bf f}_{[n]}\rVert$ is just equal to the total number of votes (rankings) divided by $n!$ (c.f., Fig. \ref{Fig:isodecomp}).

\subsubsection{The Shape $\gamma=[n-1,1]$: Individual Popularity and Polarization} \label{Se:path_schr}
The Laplacian eigenvalues $\Lambda_{[n-1,1]}$ and associated eigenvectors of the Schreier graph $\PP_{[n-1,1]}$ are the same as the 
$n-1$ nonzero Laplacian eigenvalues and associated eigenvectors of a path graph with $n$ vertices. These are known in closed form.

\begin{lemma}[see, e.g., \cite{strang1999discrete}]
Let $Q_n$ denote the path graph on $n$ vertices. The Laplacian eigenvalues of $Q_n$ are
\begin{align*}
\lambda_\ell^{Q_n} = 2-2\cos\left(\frac{\pi\ell}{n}\right),~\ell=0,1,\ldots,n-1,
\end{align*}
and the associated Laplacian eigenvectors are
\begin{align*}
{\bf v}_0^{Q_n}(i) &= \frac{1}{\sqrt{n}},~i=1,2,\ldots,n,\hbox{ and}\\
{\bf v}_\ell^{Q_n}(i) &= \frac{2}{\sqrt{n}}\cos\left(\frac{\pi\ell(i-0.5)}{n}\right),~\ell=1,2,\ldots,n-1, i=1,2,\ldots,n.
\end{align*}
\end{lemma}
In Fig. \ref{fig:91_schreier}, we show the first two eigenvalues $\Lambda_{[9,1]}$ and their associated eigenvectors on $\PP_{[9,1]}$.

\newcommand\shapenineone{\tikz[baseline]{\draw[xscale=.25,yscale=.27,line width=0.8pt] (0,0) rectangle (9,1); 
\path[xscale=.25,yscale=.27,line width=0.8pt] (.5,0.5) node {$\scriptstyle{1}$}; \path[xscale=.25,yscale=.27,line width=0.8pt] (1.5,0.5) node {$\scriptstyle{2}$}; \path[xscale=.25,yscale=.27,line width=0.8pt] (2.5,0.5) node {$\scriptstyle{4}$}; 
\path[xscale=.25,yscale=.27,line width=0.8pt] (3.5,0.5) node {$\scriptstyle{5}$}; 
\path[xscale=.25,yscale=.27,line width=0.8pt] (4.5,0.5) node {$\scriptstyle{6}$}; 
\path[xscale=.25,yscale=.27,line width=0.8pt] (5.5,0.5) node {$\scriptstyle{7}$}; 
\path[xscale=.25,yscale=.27,line width=0.8pt] (6.5,0.5) node {$\scriptstyle{8}$}; 
\path[xscale=.25,yscale=.27,line width=0.8pt] (7.5,0.5) node {$\scriptstyle{9}$}; 
\path[xscale=.25,yscale=.27,line width=0.8pt] (8.5,0.5) node {$\scriptstyle{0}$}; 
\draw[xscale=.25,yscale=.27,line width=0.8pt] (0,-1) rectangle (1,0); 
\path[xscale=.25,yscale=.27,line width=0.8pt] (.5,-0.5) node {$\scriptstyle{3}$};}}
\begin{wrapfigure}[14]{r}{.6\linewidth}
\vspace{-.2in}
 \begin{minipage}{.48\linewidth}
 \centering
 \centerline{\small{${\bf v}_{[9,1],0.0979}$~~~}} 
 \includegraphics[width = 1.8in,page=1]{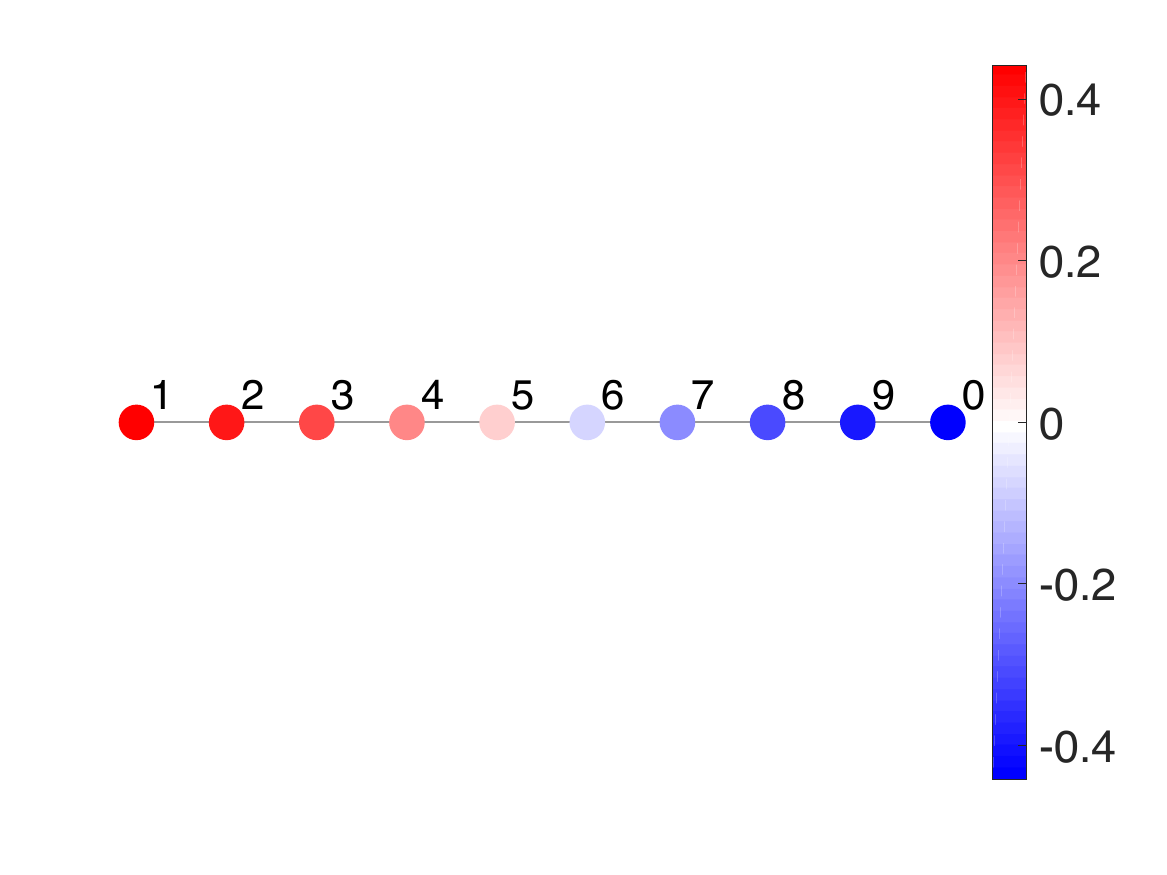}
 \end{minipage} 
  \begin{minipage}{.48\linewidth}
   \centering
     \centerline{\small{${\bf v}_{[9,1],0.3820}$~~~}} 
    \includegraphics[width = 1.8in,page=2]{figures/91_v12}
    \end{minipage} 
    \caption{The Laplacian eigenvectors associated with the first two eigenvalues in $\Lambda_{[n-1,1]}$ on the Schreier graph  $\PP_{[9,1]}$. The index on each vertex corresponds to the element in the second block of the set partition associated with the vertex; e.g., ``3'' is shorthand notation for {\protect\shapenineone}.}
    \label{fig:91_schreier}
\end{wrapfigure}
\emph{Atoms Generated from the Laplacian Eigenvector ${\bf v}_{[n-1,1],2-2\cos(\frac{\pi}{n})}$: How Popular Is Each Candidate?}

For $z=1,2,\ldots,n$, let $\bpi_z$ be the set partition that places candidate $z$ in one block and all other candidates in the other block (e.g. with $n=4$, $\bpi_2=\{\{1,3,4\},\{2\}\}$). Then the atom ${\boldsymbol \varphi}_{[n-1,1],2-2\cos(\frac{\pi}{n}),\bpi_z}=\bar{c}_{[n-1,1]}{\bf B}_{\bpi_z} {\bf v}_{[n-1,1],2-2\cos(\frac{\pi}{n})}$ is equal to $\bar{c}_{[n-1,1]}\frac{2}{\sqrt{n}}\cos(\frac{\pi(i-0.5)}{n})$ on each vertex of the permutahedron $\PP_n$ associated with a ranking in which candidate $z$ is ranked in place $i$ (see the second row from the bottom in Fig. \ref{Fig:frame_coeffs} for illustrations of such atoms). Since the  eigenvector ${\bf v}_{[n-1,1],2-2\cos(\frac{\pi}{n})}=\frac{2}{\sqrt{n}}\cos(\frac{\pi(i-0.5)}{n})$ decreases from $i=1$ to $i=n$ (see, Fig. \ref{fig:91_schreier}), the analysis coefficient ${\alpha}_{[n-1,1],2-2\cos(\frac{\pi}{n}),{\bar{\pi}}_z}$ conveys a notion of general favorability of candidate $z$. As shown in Fig. \ref{Fig:frame_coeffs}, for the 2017 Minneapolis City Council Ward 3 election data, the largest analysis coefficient in this shape-eigenvalue pair is the 290.8 associated with $\bpi_3$, followed by $\bpi_1$, $\bpi_4$, and $\bpi_2$, indicating that candidate 3 is generally popular whereas candidate 2 is generally not popular. In the sushi preference data, the order of the candidate popularities according to this metric, from most popular to least popular, is 8, 3, 1, 6, 2, 5, 4, 9, 7, 10. This ordering is the same as the Condorcet ranking listed in Sec. \ref{Se:sushi}, except with shrimp and salmon roe swapped in third and fourth place.

\begin{remark}[Relation with rank aggregation and the Borda count]
The problem of mapping ranking data into a single consensus ranking of the candidates is referred to as \emph{rank aggregation} and dates back to social  choice theory in the 18th century \cite{borda1784memoire} (see \cite{lin2010rank}, \cite[Sec. 5]{yu2019analysis} for more recent surveys). Borda's original method \cite{borda1784memoire} is to assign $n$ points to each first place vote, $n-1$ points to each second place vote, and so forth, until 1 point for each last place vote. The consensus ranking is then formed according to the total number of points each candidates receives. The ranking of the candidates according to the analysis coefficients  ${\alpha}_{[n-1,1],2-2\cos(\frac{\pi}{n}),{\bar{\pi}}_z}$ demonstrated above is equivalent to a Borda count aggregate ranking that uses point values 
\begin{align*}
\frac{(n-1)\cos\left(\frac{\pi(i-0.5)}{n}\right)}{2\cos\left(\frac{\pi}{2n}\right)}+\frac{n+1}{2},~i=1,2,\ldots,n
\end{align*}
for $i$th place instead of the more common linearly spaced point values from $n$ to $1$.
For example, with 10 candidates, the ranking of the analysis coefficients generated from this eigenvector is equivalent to a weighted Borda count where the points assigned for each vote, from first place to last place, are 10.00, 9.56, 8.72, 7.57, 6.21, 4.79, 3.43, 2.28, 1.44, and 1.00.
\end{remark}

\emph{Atoms Generated from the Laplacian Eigenvector ${\bf v}_{[n-1,1],2-2\cos(\frac{2\pi}{n})}$: How Polarizing Is Each Candidate?}

The eigenvector elements ${\bf v}_{[n-1,1],2-2\cos(\frac{2\pi}{n})}(i)=\frac{2}{\sqrt{n}}\cos(\frac{2\pi(i-0.5)}{n})$ decrease from $i=1$ to $i=\ceil{\frac{n}{2}}$, and then increase again from $i=\floor{\frac{n}{2}}+1$ to $i=n$ (see, e.g., Fig. \ref{fig:91_schreier}). The atom 
$${\boldsymbol \varphi}_{[n-1,1],2-2\cos(\frac{2\pi}{n}),\bpi_z}=\bar{c}_{[n-1,1]}{\bf B}_{\bpi_z} {\bf v}_{[n-1,1],2-2\cos(\frac{2\pi}{n})}$$ 
therefore features positive values on the vertices of $\PP_n$ associated with rankings in which candidate $z$ is ranked towards the top or bottom, and negative values on vertices associated with rankings in which candidate $z$ is ranked in the middle. Accordingly, the analysis coefficient ${\alpha}_{[n-1,1],2-2\cos(\frac{2\pi}{n})\bpi_z}$ is large when many voters feel strongly (either positively or negatively) about candidate $z$. For example, looking at the summary of the analysis coefficients for the 2017 Minneapolis City Council Ward 3 election data in the bottom row of the column for eigenvalue $\lambda=2-2\cos(\frac{2\pi}{4})=2$ in Fig. \ref{Fig:frame_coeffs}, the largest analysis coefficient is the 318.7 associated with the set partition $\{\{2,3,4\},\{1\}\}$, indicating that candidate 1 (Ginger Jentzen) is often ranked in either first or last place. The corresponding coefficients for candidates 2 and 3 are negative, indicating they are often ranked in positions 2 and 3, and are less polarizing. For the sushi preference data, the largest positive analysis coefficients in the $\gamma=[9,1]$, $\lambda=0.382$ pair (see Fig. \ref{Fig:sushi_analysis}) indicate items 8 and 5 (fatty tuna and sea urchin) are often ranked quite highly or quite lowly, while item 9 (tuna roll) is often ranked in the middle. The key takeaway is that the combination of a near-zero value of $\alpha_{[n-1,1],2-2\cos(\frac{\pi}{n}),\bpi_z}$ and a relatively large value of $\alpha_{[n-1,1],2-2\cos(\frac{2\pi}{n}),\bpi_z}$ indicates that the candidate identified by the singleton in the set partition $\bpi_z$ is highly polarizing. This is the case in the sushi preference data for the sea urchin item, for which $\alpha_{[n-1,1],2-2\cos(\frac{2\pi}{n}),\bpi_z}=-0.016$ (lowest magnitude of any item) and $\alpha_{[n-1,1],2-2\cos(\frac{2\pi}{n}),\bpi_z}=0.960$ (second highest of the items), indicating that the voting population is roughly split between strongly liking and strongly disliking sea urchin.

\subsubsection{The Shapes $\gamma=[n-2,2]$ and $\gamma=[n-2,1,1]$: Pairwise Co-Occurence}

\emph{Net of their individual popularities, when are two candidates likely to be ranked similarly (either positively or negatively) by voters? Given a voter's first choice candidate, are there other candidates the voter is likely to feel positively, negatively, or neutral about?} These are the types of pairwise co-occurence questions that can be answered with the second-order marginal information found in the shapes $\gamma=[n-2,2]$ and $\gamma=[n-2,1,1]$.

For the 2017 Minneapolis City Council Ward 3 election data, the largest analysis coefficient in this shape, shown in bottom-right of Fig. \ref{Fig:frame_coeffs}, is the 239.0 associated with eigenvalue $\lambda=1.268$ and the set partition $\pi_z=\{12|34\}$, indicating candidates 3 and 4 are often ranked together in the first two positions or last two positions. This is not surprising as these candidates belong to the same political party. For a small number of candidates, such as $n=4$ in this case, it is possible to visually inspect the atoms of the form ${\boldsymbol \varphi}_{\gamma,\lambda,\bpi}$ shown in Fig. \ref{Fig:frame_coeffs} in order to interpret the analysis coefficients.
 
For larger values of $n$, however, it is more convenient to think about the inner products defined in \eqref{Eq:anal_coeff_two} as 
$\alpha_{\gamma,\lambda,\bpi}= \bar{c}_{\gamma} \langle  {\bf B}_{\bpi}^{\top}{\bf f},{\bf v}_{\gamma, \lambda}\rangle$. Here, ${\bf B}_{\bpi}^{\top}{\bf f}$ is a projection of the signal from $\PP_n$ down to the Schreier graph $\PP_\gamma$, the same structure on which the eigenvector ${\bf v}_{\gamma, \lambda}$ resides. In Fig. \ref{fig:82p}, we show four such projections of ${\bf h}$ onto $\PP_{[8,2]}$ that capture the joint placement of four different pairs of items: 3 and 8 (two favorites), 8 and 10 (one favorite and one of the least preferable items), 5 and 6 (two more polarizing items that are often ranked together near the top or together at the bottom), and 1 and 2 (two generally well liked items that are often ranked towards the top but not at the top). In each of the projections shown in Fig. \ref{fig:82p}, the value at each vertex is equal to the number of voters who ranked the selected pair of items in the ranking positions contained in the vertex labels; for example, Fig. \ref{fig:82p}(a) shows that 637 of the 5000 voters placed items 3 and 8 (tuna and fatty tuna) in their top two ranking slots.

 \begin{figure}[t]
\centering
\hfill
\begin{minipage}{.22\linewidth}
 \centering
  \centerline{\small{${\bf B}_{\{12456790|38\}}^{\top}{\bf h}$~~~}} 
    \includegraphics[width = \linewidth,page=1]{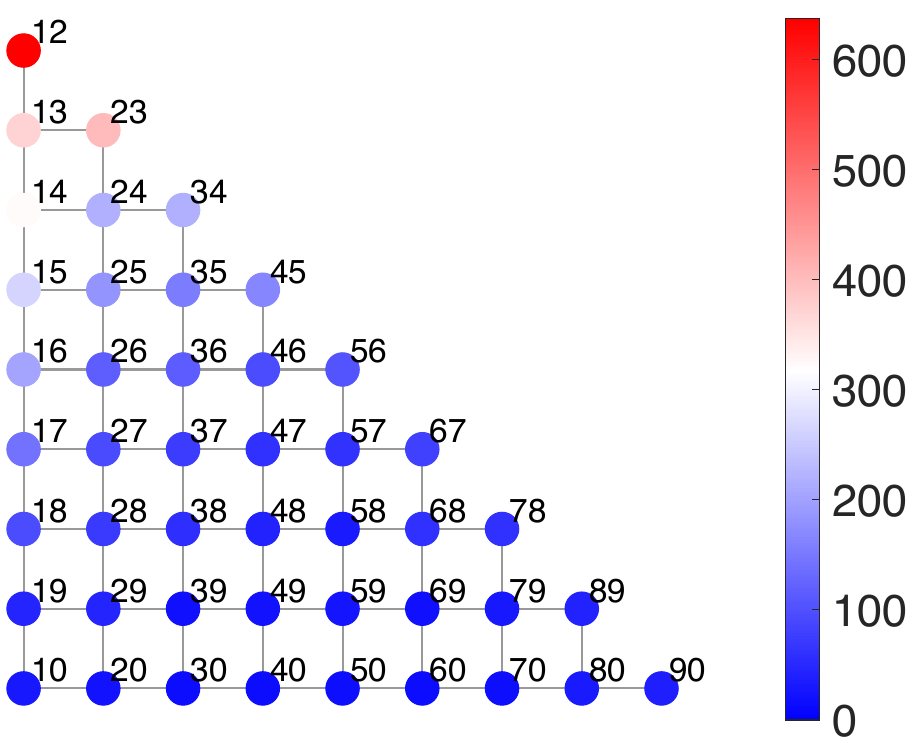} 
    \centerline{\small{(a)~~~}}
\end{minipage}
\hfill
\begin{minipage}{.22\linewidth}
 \centering
  \centerline{\small{${\bf B}_{\{12345679|80\}}^{\top}{\bf h}$~~~}} 
    \includegraphics[width = \linewidth,page=2]{figures/82p} 
    \centerline{\small{(b)~~~}}
    \end{minipage}
    \hfill
    \begin{minipage}{.22\linewidth}
     \centering
  \centerline{\small{${\bf B}_{\{12347890|56\}}^{\top}{\bf h}$~~~}} 
    \includegraphics[width = \linewidth,page=3]{figures/82p} 
    \centerline{\small{(c)~~~}}
    \end{minipage}
    \hfill
    \begin{minipage}{.22\linewidth}
     \centering
  \centerline{\small{${\bf B}_{\{34567890|12\}}^{\top}{\bf h}$~~~}} 
        \includegraphics[width = \linewidth,page=4]{figures/82p} 
        \centerline{\small{(d)~~~}}
        \end{minipage}
        \hfill
    \caption{The sushi preference signal ${\bf h}$ projected down from the permutahedron $\PP_{10}$ to the Schreier graph $\PP_{[8,2]}$ via four different projection matrices that correspond to four different ordered 
     set partitions in $\bar{\Pi}_{[8,2]}$.}
    \label{fig:82p}
          \vspace{0.1in}
\hrule height 1.5pt
\vspace{-.1in}
\end{figure}

 \begin{figure}[b]
    \vspace{-0.1in}
\hrule height 1.5pt
\vspace{.1in}
\centering
\hfill
\begin{minipage}{.22\linewidth}
 \centering
   \centerline{\small{${\bf T}_{[9,1],[8,2]}{\bf v}_{[9,1], 0.0979}$~~~}} 
    \includegraphics[width = .95\linewidth,page=1]{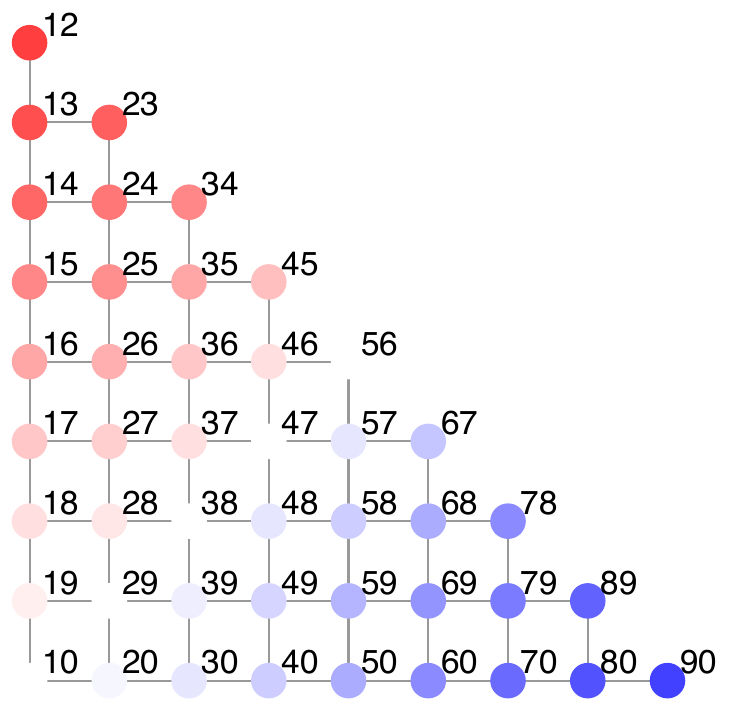}
    \centerline{\small{(a)}}
\end{minipage}
\hfill
\begin{minipage}{.22\linewidth}
 \centering
   \centerline{\small{${\bf T}_{[9,1],[8,2]}{\bf v}_{[9,1], 0.3820}$~~~}} 
    \includegraphics[width = .95\linewidth,page=2]{figures/82}
    \centerline{\small{(b)}}
    \end{minipage}
    \hfill
    \begin{minipage}{.22\linewidth}
     \centering
       \centerline{\small{${\bf v}_{[8,2], 0.2047}$~~~}} 
    \includegraphics[width = .95\linewidth,page=3]{figures/82}
    \centerline{\small{(c)}}
    \end{minipage}
    \hfill
    \begin{minipage}{.22\linewidth}
     \centering
       \centerline{\small{${\bf v}_{[8,2], 0.4700}$~~~}}
        \includegraphics[width = .95\linewidth,page=4]{figures/82}
        \centerline{\small{(d)}}
        \end{minipage}
        \hfill
            \begin{minipage}{.02\linewidth}
     \centering
        \includegraphics[height=1.3in,page=5]{figures/82}
        \end{minipage}
        \hfill
\hfill
    \caption{(a)-(b) The Laplacian eigenvectors ${\bf v}_{[9,1], \lambda}$, for $\lambda=0.0979$ and $\lambda=0.3820$, lifted from $\PP_{[9,1]}$ to $\PP_{[8,2]}$. (c)-(d) The Laplacian eigenvectors ${\bf v}_{[8,2], \lambda}$, for $\lambda=0.2047$ and $\lambda=0.4700$, are orthogonal to each other and to the two vectors in (a) and (b) on the Schreier graph $\PP_{[8,2]}$.}
    \label{fig:82_schreier}
\end{figure}

The other halves of the inner products $\alpha_{[8,2],\lambda,\bpi}= \bar{c}_{[8,2]} \langle  {\bf B}_{\bpi}^{\top}{\bf f},{\bf v}_{[8,2], \lambda}\rangle$ are the Laplacian eigenvectors of the module $V_{[8,2]}$, the first two of which we show in Fig. \ref{fig:82_schreier}(c)-(d). The coefficients $\{\alpha_{[8,2],0.2047,\bpi}\}$ associated with the eigenvector in Fig. \ref{fig:82_schreier}(c) capture a notion of pairwise proximity. Specifically, a large positive coefficient indicates the two items/candidates grouped by the set partition $\bpi_z$ are likely to be close to the first two or last two ranking positions, suggesting they may have similar features or belong to the same political party in the case of an election. For the sushi preference data, as shown in Fig. \ref{Fig:sushi_analysis}, the two largest positive coefficients $\bar{c}_{[8,2]} \langle {\bf B}_{\bpi}^{\top}{\bf f},{\bf v}_{[8,2], \lambda}\rangle$ are the 1.3304 associated with the set partition $\bpi=\{12456790|38\}$, which groups the two overall favorites fatty tuna and tuna together, and the 1.6543 associated with $\bpi=\{12345689|70\}$, which groups the two overall least favorites -- and the only two vegetarian options -- egg and cucumber roll together. In ranked voting elections where there are more than two candidates in the same political party, we might expect to see large positive coefficients on these $[n-2,2]$ analysis coefficients associated with each of the pairs of candidates from the same party. A negative coefficient with large magnitude, on the other hand, indicates the two items are frequently ranked at opposite ends. In the sushi data, the most negative coefficient $\bar{c}_{[8,2]} \langle {\bf B}_{\bpi}^{\top}{\bf f},{\bf v}_{[8,2], \lambda}\rangle$ is the -1.7150 associated with the set partition $\bpi=\{12345679|80\}$, which groups the overall favorite (fatty tuna) and overall least favorite (cucumber roll) items together.

Let us now consider the set partition $\bpi=\{34567890|12\}$, for which the corresponding projection onto $\PP_{[8,2]}$ is shown in Fig. \ref{fig:82p}(d). Of the 45 analysis coefficients $\{\alpha_{[8,2],0.2047,\bpi}\}$, $\alpha_{[8,2],0.2047,\{34567890|12\}}=-0.0231$ is the third smallest in magnitude, and of the 45 analysis coefficients $\{\alpha_{[8,2],0.4700,\bpi}\}$, $\alpha_{[8,2],0.4700,\{34567890|12\}}=-0.0088$ is the smallest in magnitude.  There are two reasons these coefficients are closer to 0. First, the energy of the projection ${\bf B}_{\{34567890|12\}}^{\top}{\bf h}$ in Fig. \ref{fig:82p}(d) is more evenly spread across the vertices of $\PP_{[8,2]}$ than the other projections in Fig. \ref{fig:82p}(a)-(c) even though they all have the same sum; thus, the norm $\lVert{\bf B}_{\{34567890|12\}}^{\top}{\bf h}\rVert$ is smaller. Second, from a graph signal processing viewpoint, the energy $\lVert{\bf B}_{\{34567890|12\}}^{\top}{\bf h}\rVert^2$ decomposes across the Laplacian eigenspaces of the Schreier graph $\PP_{[8,2]}$. In this case, most of that energy falls into the one-dimensional spaces spanned by the eigenvectors shown in Fig. \ref{fig:82_schreier}(a)-(b). In addition to being  Laplacian eigenvectors of $\PP_{[8,2]}$, these vectors can be viewed as liftings of eigenvectors from the module $V_{[9,1]}$ to $\PP_{[8,2]}$. More formally and generally, 
if $\nu \vartriangleright \gamma$, then for any $\xi \in \bar{\Pi}_\nu$ and $\pi \in \bar{\Pi}_\gamma$, we define the linear mapping ${\bf T}_{\xi,\pi}: \Rbb [\Pi_\nu] \rightarrow \Rbb [\Pi_\gamma]$ by ${\bf T}_{\xi,\pi}{\bf x}:={\bf B}_{\pi}^{\top}{\bf B}_{\xi}{\bf x}$ (see Prop. ~\ref{prop:Young}). If $K_{\gamma,\nu}=1$, then $\frac{{\bf T}_{\xi,\pi}{\bf x}}{\lVert{\bf T}_{\xi,\pi}{\bf x}\rVert}= \pm \frac{{\bf T}_{\xi^\prime,\pi^\prime}{\bf x}}{\lVert{\bf T}_{\xi^\prime,\pi^\prime}{\bf x}\rVert}$ for all $\xi,\xi^\prime \in \bar{\Pi}_\nu$, $\pi,\pi^\prime \in \bar{\Pi}_\gamma$, and ${\bf x} \in V_\nu$, and we identify all of these transformations under the notation ${\bf T}_{\nu,\gamma}$.

  \begin{wrapfigure}[15]{r}{.5\textwidth}
  \vspace{-.3in}
\centering
    \begin{minipage}{.47\linewidth}
     \centering
         \centerline{\small{${\bf T}_{[8,2],[8,1,1]}{\bf v}_{[8,2], 0.2047}$}} 
    \includegraphics[width = .95\linewidth,page=3]{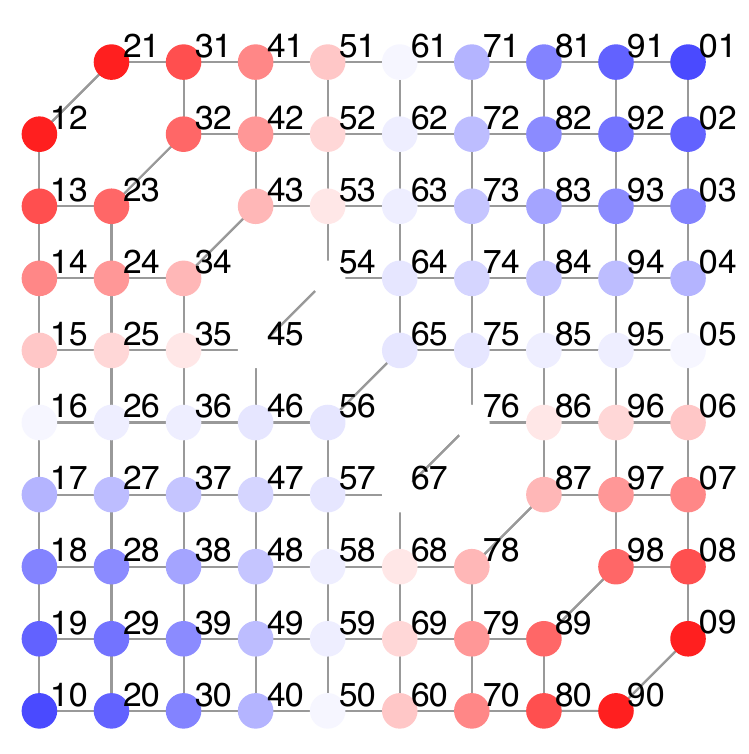}
    \centerline{\small{(a)}}
    \end{minipage}
    \hfill
    \begin{minipage}{.47\linewidth}
     \centering
              \centerline{\small{${\bf v}_{[8,1,1], 0.4799}$}} 
        \includegraphics[width = .95\linewidth,page=4]{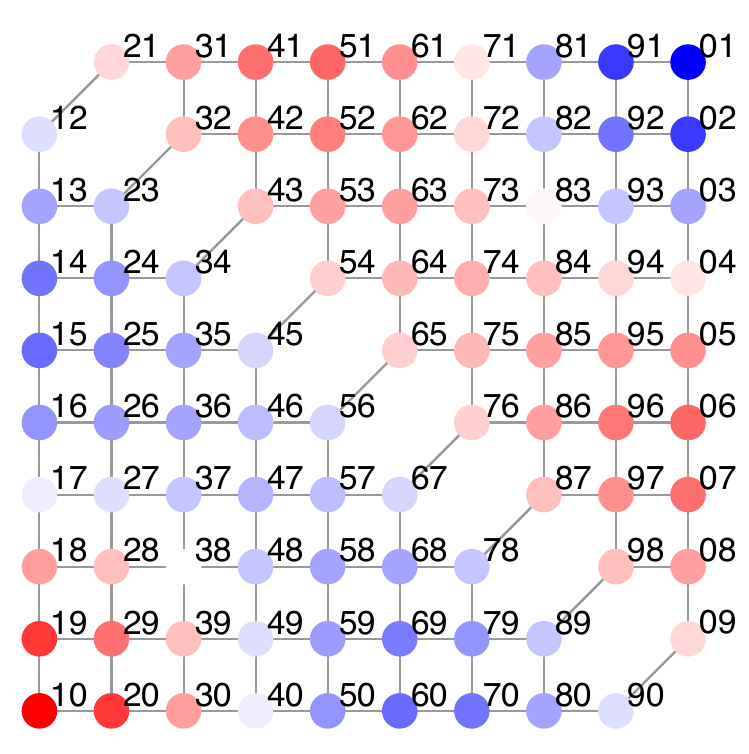}
        \centerline{\small{(b)}}
        \end{minipage}
        \hfill
            \begin{minipage}{.02\linewidth}
     \centering
        \includegraphics[height=1.3in,page=5]{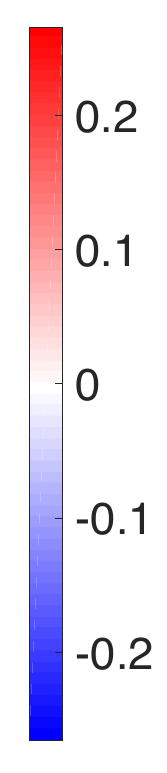}
        \end{minipage}
        \hfill
    \caption{(a) The Laplacian eigenvector ${\bf v}_{[8,2], 0.2047}$, lifted from $\PP_{[8,2]}$ to $\PP_{[8,1,1]}$. (b) The Laplacian eigenvector ${\bf v}_{[8,1,1], 0.4799}$ is orthogonal to the vector in (a) on the Schreier graph $\PP_{[8,1,1]}$.}
    \label{fig:811_schreier}
\end{wrapfigure}
Next, we examine the ordered second-order marginals, discussed in Sec. \ref{Se:Fourier} and captured in the $[n-2,1,1]$ isotypic component. Returning to the negative coefficient $\alpha_{[8,2],0.2047,\{12345679|80\}}=-1.7150$, potential causes for the large magnitude could be that (i) the first selected candidate (8) is commonly ranked first and the second selected candidate (10) is commonly ranked last, (ii) the first selected candidate is commonly ranked last and the second selected candidate is commonly ranked first, or (iii) both candidates are polarizing, but ranked at opposite ends of the spectrum by different sets of voters. Which other coefficients provide structural information that can be combined with $\alpha_{[8,2],0.2047,\{12345679|80\}}$ to inform which of these scenarios might be occurring? In this case, we know from the coefficients $\alpha_{[9,1],0.0979,\{123456789|0\}}=-2.1513$ and $\alpha_{[9,1],0.0979,\{123456790|8\}}=1.9978$ that item 8 (fatty tuna) is the popular one and item 10 (cucumber roll) is the unpopular one. More generally, we can examine the coefficient $\alpha_{[8,1,1],0.4799,\{12345679|8|0\}}$, which is equal to 1.3471 in this case. From the eigenvector ${\bf v}_{[8,1,1], 0.4799}$ in Fig. \ref{fig:811_schreier}(b), we see that a positive value implies the first selected candidate is more often ranked towards the top (the case here), a negative value implies the second selected candidate is more often ranked towards the top, and a coefficient of small magnitude indicates that both candidates are roughly evenly located at the two ends (scenario (iii) above).

An analysis coefficient associated with the first eigenvector of the module $V_{[n-2,1,1]}^\ast$ can also provide some insight into the interpretation of a large positive analysis coefficient associated with the first eigenvector of the module $V_{[n-2,2]}^\ast$. Namely, a positive value of the coefficient indicates the second selected candidate is more popular than the first, and vice versa. For example, consider the pairs of items 7/10 and 9/10. From the fact that $\alpha_{[8,2],0.2047,\{12345689|70\}}=1.6543$ and $\alpha_{[8,2],0.2047,\{12345678|90\}}=0.7055$ are positive, we conclude these pairs of items often appear together at the top or bottom of the rankings. We know from the fact that the coefficients $\alpha_{[9,1],0.0979,\{123456789|0\}}=-2.1513$, $\alpha_{[9,1],0.0979,\{123456890|7\}}=-1.1896$, and $\alpha_{[9,1],0.0979,\{123456780|9\}}=-0.3733$ are all negative that all three items are unpopular and these pairs therefore appear together more often at the bottom of the rankings. Is one item in each pair more likely to be ranked last? We see from the projections onto $\PP_{[8,1,1]}$ in Fig. \ref{fig:disliked}(b) and (d) that when ranked in the last two slots, item 10 is only slightly more likely than item 7 to be ranked last, but item 9 is much more frequently ranked higher than item 10 in these pairwise frequencies. This difference is reflected in the coefficients $\alpha_{[8,1,1],0.4799,\{12345689|7|0\}}=-0.4050$ and $\alpha_{[8,1,1],0.4799,\{12345678|9|0\}}=-0.9136$.

 \begin{figure}[t]
\centering
\hfill
\begin{minipage}{.22\linewidth}
 \centering
  \centerline{\small{${\bf B}_{\{12345689|70\}}^{\top}{\bf h}$~~}} 
    \includegraphics[width = \linewidth,page=1]{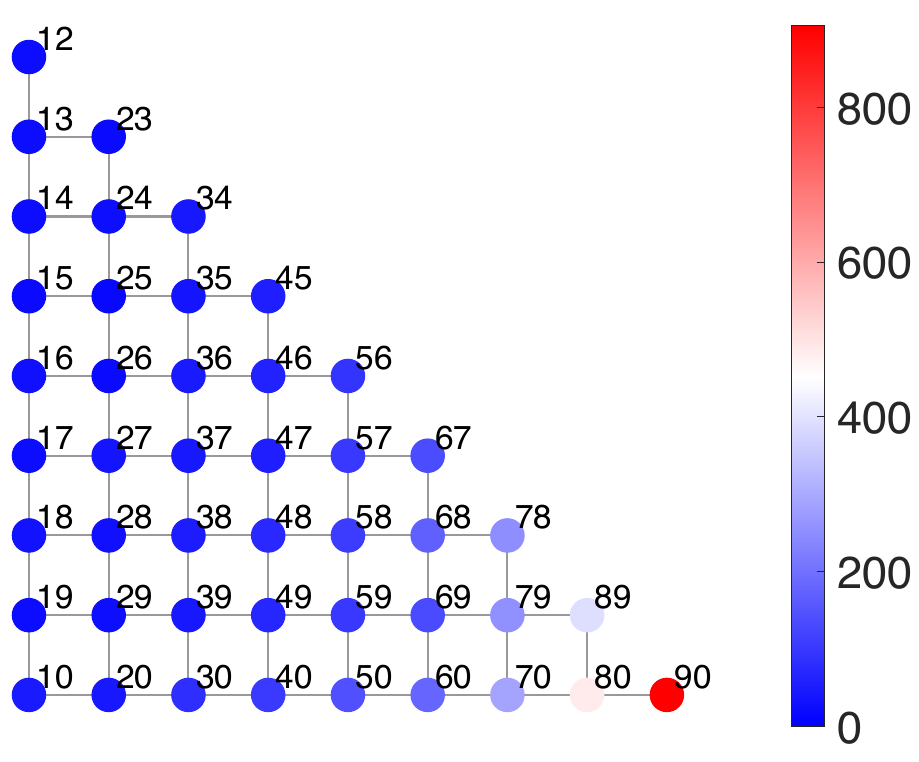} 
    \centerline{\small{(a)~~}}
\end{minipage}
\hfill
\begin{minipage}{.22\linewidth}
 \centering
 \centerline{\small{${\bf B}_{\{12345689|7|0\}}^{\top}{\bf h}$~~}} 
    \includegraphics[width = \linewidth,page=2]{figures/disliked} 
    \centerline{\small{(b)~~}}
    \end{minipage}
    \hfill
    \begin{minipage}{.22\linewidth}
     \centering
       \centerline{\small{${\bf B}_{\{12345678|90\}}^{\top}{\bf h}$~~}} 
    \includegraphics[width = \linewidth,page=3]{figures/disliked} 
    \centerline{\small{(c)~~}}
    \end{minipage}
    \hfill
    \begin{minipage}{.22\linewidth}
     \centering
      \centerline{\small{${\bf B}_{\{12345678|9|0\}}^{\top}{\bf h}$~~}} 
        \includegraphics[width = \linewidth,page=4]{figures/disliked} 
        \centerline{\small{(d)~~}}
        \end{minipage}
        \hfill
    \caption{Projections capturing the second order ordered and unordered marginals for the pairs of sushi items 7/10 and 9/10.}
    \label{fig:disliked}
      \vspace{0.1in}
\hrule height 1.5pt
\vspace{-.1in}

\end{figure}

\subsubsection{Shapes with $\gamma_1< n-2$}

For subsequent shapes, we can interpret the analysis coefficient in a similar manner, as the inner products between projections of the signal down to the Schreier graph and Laplacian eigenvectors ${\bf v}_{\gamma,\lambda}$ in $V_\gamma^\ast$. Once again, the most interpretable eigenvectors and therefore informative analysis coefficients are usually those associated with lower eigenvalues (i.e., the first couple from each new irreducible $V_\gamma^\ast$).

\begin{wrapfigure}[15]{r}{.6\textwidth}
\vspace{-.25in}
\begin{minipage}{.48\linewidth}
\centerline{\small{${\bf B}_{\{1234790|568\}}^{\top}{\bf h}$}} 
\centerline{\includegraphics[width=1\linewidth,page=1]{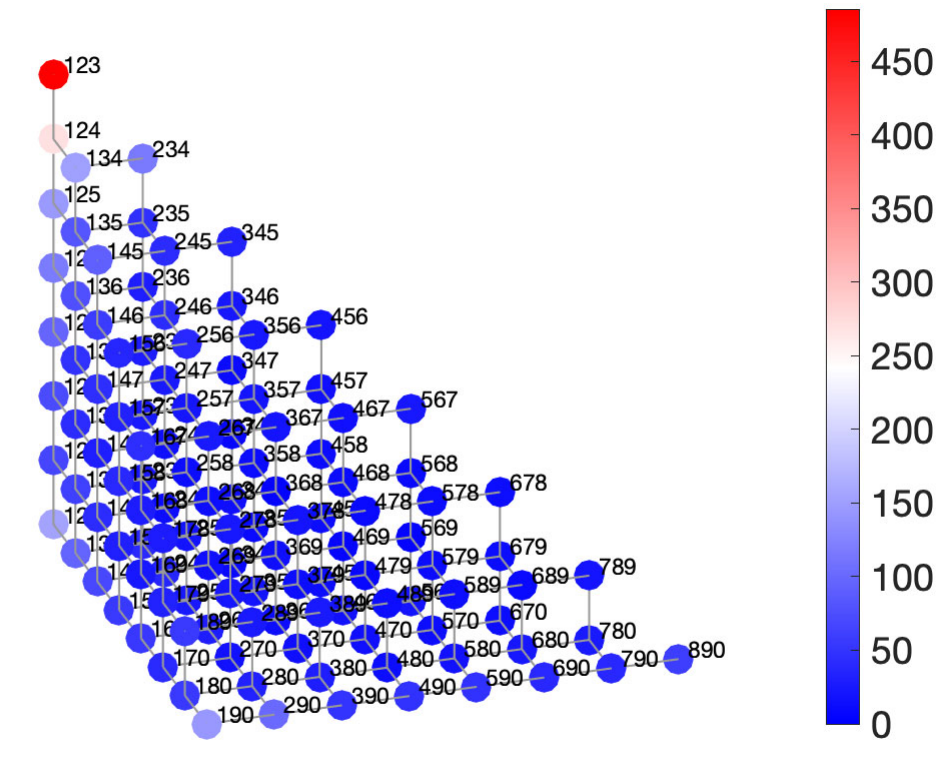}} 
\end{minipage}
\begin{minipage}{.48\linewidth}
\centerline{\small{${\bf B}_{\{3567890|124\}}^{\top}{\bf h}$}} 
\centerline{\includegraphics[width=1\linewidth,page=2]{figures/73p}}
\end{minipage}
\caption{The sushi preference signal ${\bf h}$ projected down from the permutahedron $\PP_{10}$ to the Schreier graph $\PP_{[7,3]}$ via two different projection matrices that correspond to two of the 120 different ordered 
set partitions in $\bar{\Pi}_{[7,3]}$.}
\label{Fig:projections73} 
\end{wrapfigure}
As a final example, we consider the analysis coefficients $\alpha_{[7,3],\lambda,\bpi}$ generated by the values $\lambda$ equal to 0.3227, 0.5660, and 0.8122 (the associated eigenvectors of which are shown in the bottom row of Fig. \ref{fig:73_schreier}) and the set partitions $\bpi$ equal to $\{1234790|568\}$ and $\{3567890|124\}$. The first projection of the data in Fig. \ref{Fig:projections73}, corresponding to $\bpi=\{1234790|568\}$, shows that items 5, 6, and 8 (sea urchin, salmon roe, and fatty tuna) most frequently occur in the top three rankings or rankings 1, 2, and 4, but it is also common that voters ranked two of the three at the top and one at the bottom, or one at the top and two at the bottom. Due to the high concentration of the projection on the vertices of the Schreier $\PP_{[7,3]}$ labeled by 123 and 124, the inner products with the three eigenvectors in Fig. \ref{fig:73_schreier}(d)-(f) have the second, seventh, and tenth largest magnitudes of the 9000 analysis coefficients $\{\alpha_{[7,3],\lambda,\bpi}\}_{\lambda \in \Lambda_{[7,3]},\bpi \in \bPi_{[7,3]}}$, and the second, first, and first largest magnitudes of the 120 analysis coefficients in each of the subsets $\{\alpha_{[7,3],0.3227,\bpi}\}_{\bpi \in \bPi_{[7,3]}}$, $\{\alpha_{[7,3],0.5660,\bpi}\}_{\bpi \in \bPi_{[7,3]}}$, and $\{\alpha_{[7,3],0.8122,\bpi}\}_{\bpi \in \bPi_{[7,3]}}$, respectively. That is to say, there is a strong third order effect of these three items being ranked together and highly, net of the first order and second order effects, which are captured, e.g., in the lifted eigenvectors in Fig. \ref{fig:73_schreier}(a)-(c).

The second projection of the data in Fig. \ref{Fig:projections73} captures the joint ranking positions of three sushi items with cooked fish: 1, 2, and 4 (shrimp, sea eel, and squid). Our suspicion that these three items would be ranked similarly by many voters is confirmed by the projection ${\bf B}_{\{3567890|124\}}^{\top}{\bf h}$, which shows the items are most commonly ranked in slots 4/5/6, 5/6/7, or 6/7/8. Despite the frequent closeness of the rankings of these three items, the magnitudes of the inner products with the three eigenvectors in Fig. \ref{fig:73_schreier}(d)-(f) are quite small (all less than 0.075 and ranking 88th, 72nd, and 52nd out of the magnitudes of the analysis coefficients associated with 120 different set partitions for the respective eigenvalues). The reason for this is that the projection ${\bf B}_{\{3567890|124\}}^{\top}{\bf h}$ is quite close (up to a sign change) to  ${\bf T}_{[9,1],[7,3]}{\bf v}_{[9,1], 0.3820}$ in Fig. \ref{fig:73_schreier}(b), and the eigenvectors in Fig. \ref{fig:73_schreier}(d)-(f) are orthogonal to ${\bf T}_{[9,1],[7,3]}{\bf v}_{[9,1], 0.3820}$. That is, the third order effect of these three cooked fish sushi items being ranked together is weak once the first order effects found in the $[9,1]$ shape have been accounted for.
\clearpage

 \begin{figure}[t]
 \vspace{.1in}
\centering
\hfill
\begin{minipage}{.3\linewidth}
 \centering
     \centerline{\small{${\bf T}_{[9,1],[7,3]}{\bf v}_{[9,1], 0.0979}$}} 
    \includegraphics[width = .9\linewidth,page=1]{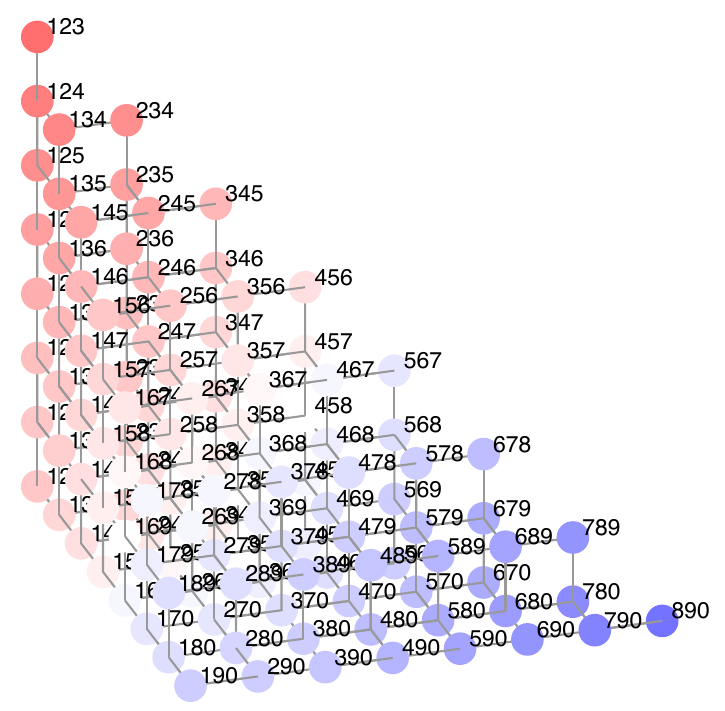}
                    \centerline{\small{(a)}}
                    \vspace{.1in}
                    
              \centerline{\small{${\bf v}_{[7,3], 0.3227}$}} 
        \includegraphics[width = .9\linewidth,page=4]{figures/73}
                        \centerline{\small{(d)}}
\end{minipage}
\hfill
\begin{minipage}{.3\linewidth}
 \centering
      \centerline{\small{${\bf T}_{[9,1],[7,3]}{\bf v}_{[9,1], 0.3820}$}} 
    \includegraphics[width = .9\linewidth,page=2]{figures/73}
                    \centerline{\small{(b)}}
                                        \vspace{.1in}
                                        
                  \centerline{\small{${\bf v}_{[7,3], 0.5660}$}} 
        \includegraphics[width = .9\linewidth,page=5]{figures/73}
                        \centerline{\small{(e)}}
    \end{minipage}
    \hfill
    \begin{minipage}{.3\linewidth}
     \centering
          \centerline{\small{${\bf T}_{[8,2],[7,3]}{\bf v}_{[8,2], 0.2047}$}} 
    \includegraphics[width = .9\linewidth,page=3]{figures/73}
                    \centerline{\small{(c)}}
                                        \vspace{.1in}
                                        
                  \centerline{\small{${\bf v}_{[7,3], 0.8122}$}} 
        \includegraphics[width = .9\linewidth,page=6]{figures/73}
                \centerline{\small{(f)}}
    \end{minipage}
    \hfill
             \begin{minipage}{.02\linewidth}
     \centering
        \includegraphics[height=1.7in,page=7]{figures/73}
        \end{minipage}
        \hfill
\hfill
    \caption{The three Laplacian eigenvectors ${\bf v}_{[7,3], \lambda}$ of $\PP_{[7,3]}$ shown in (d)-(f) (and the 72 others not shown) are orthogonal to each other and to the liftings to $\PP_{[7,3]}$ of all Laplacian eigenvectors of $P_{[10]}$, $P_{[9,1]}$, and $P_{[8,2]}$ (the Schreier graphs corresponding to the shapes that strictly precede $[7,3]$ in dominance ordering), including the three liftings shown in (a)-(c).}
    \label{fig:73_schreier}
     \vspace{0.1in}
\hrule height 1.5pt
\vspace{-.1in}
\end{figure}

%%%%%%%%%%%%%%%%%%%%%%%%%%%%%%%%%%%%%%%%%%%%%%%%%%%%%%%%%%%%%%%%%%%%%%%%
\section{Computationally Efficient Algorithms}\label{Se:comp}

In this section, we develop efficient  algorithms for the proposed transform,  
and discuss 
details of our openly available software that implements the 
transform and helps users find structure in ranked data by visualizing the analysis coefficients and/or projecting the data into lower dimensional spaces to visualize.

\emph{What are the initial computational challenges?} First, naively computing the eigenvectors of  
$\PP_n$ is not feasible for $n$ above 7 or 8, as it has complexity ${\cal O}([n!]^3)$. Second, even naively computing the eigenvectors of the required Schreier graphs is not feasible (complexity ${\cal O}(\left[\frac{n!}{\lceil n/2 \rceil!}\right]^3)$). 
Third, any method that explicitly  computes and stores all 
dictionary atoms of length $n!$ on the permutahedron in order to take inner products with the ranked data will quickly run into memory issues as $n$ grows. Fourth, the 
number of dictionary atoms in $\D$ from Thm. \ref{Th:tight_frame} is $\sum_{\gamma \vdash n} d_{\gamma} m_{\gamma}$, which grows
faster than $n!$; i.e., the redundancy of the dictionary increases as $n$ increases. 

In order to circumvent these issues, we need 
specialized algorithms (i.e., not standard signal processing or numerical linear algebra techniques) that take advantage of the symmetry and structure present in the permutahedron. 

Our general approach is to include as much of the computation as possible into a setup portion that is independent of the data and can therefore be performed just once, offline and ahead of time, for each $n$. This setup portion consists of three phases: the dynamic constructions of the  adjacency matrix and the characteristic matrix ${\bf B}_\gamma$ for each Schreier graph $\PP_\gamma$ (Sec. \ref{Se:sch_adj}), the computation of the Laplacian eigenvalues and eigenvectors of the Schreier graphs (Sec. \ref{Se:sch_eigs}), and the computation of a path from each reading set partition to every other vertex in the Schreier graph (Sec. \ref{Sec:anal_code}). With all of this information stored for a given $n$, the data-dependent analysis portion of the code (Sec. \ref{Sec:anal_code}) needs to be executed for each new ranked data vector ${\bf f}$.

\subsection{Construction of the Schreier Adjacency Matrices and a Lifting Matrix from Each Schreier to the Permutahedron} \label{Se:sch_adj}

\begin{wrapfigure}[24]{r}{0.6\textwidth}
\vspace{-.3in}
\centerline{\includegraphics[width=\linewidth,page=1]{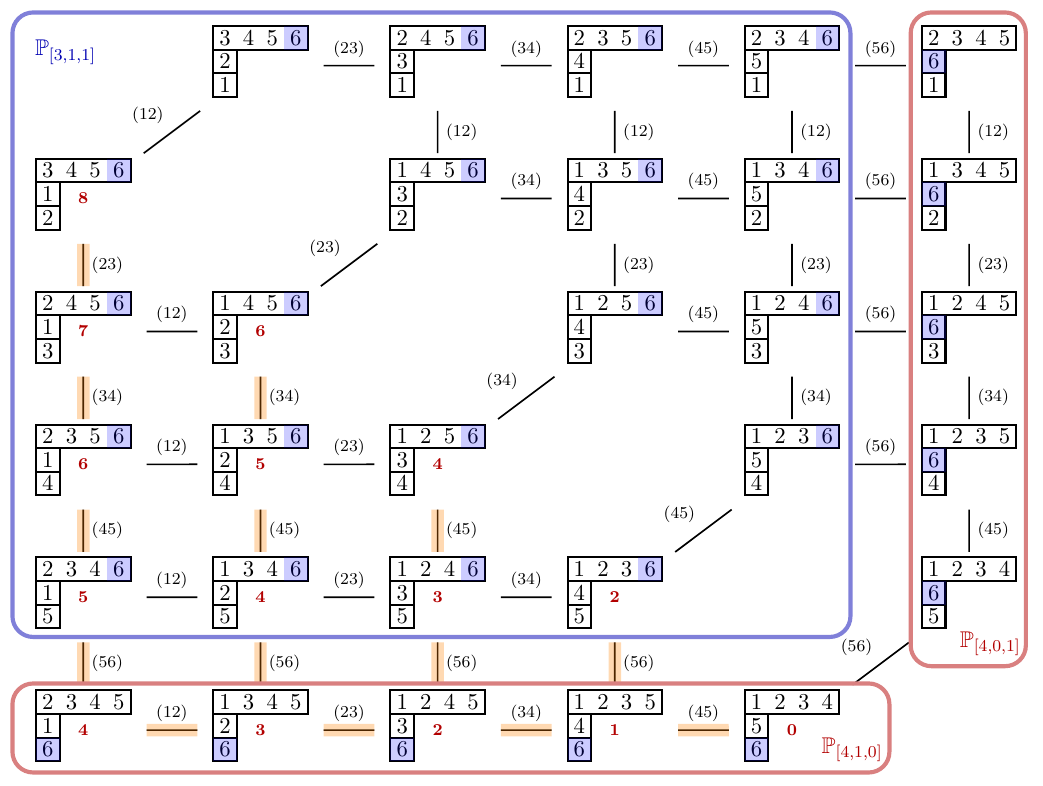}}
\caption{The decomposition of the Schreier graph $\PP_{[4,1,1]}$ into $\PP_{[3,1,1]}, \PP_{[4,1,0]},$ and $\PP_{[4,0,1]}$. Note that $\PP_{[4,1,0]}$ and $\PP_{[4,0,1]}$ are each isomorphic to $\PP_{[4,1]}$. The red numbers indicate the number of adjacent transpositions to get to the reading-order set partition $\pi_1$, and the highlighted paths are 
minimal paths from each $\bpi \in \bPi_{[4,1,1]}$ to $\pi_1$, constructed via a breadth-first search. \label{fig:recursiveSchreier}}
\end{wrapfigure}

The Schreier graphs for $\SS_n$ can be constructed 
in an iterative dynamic manner from those for $\SS_{n-1}$. If $\gamma = [\gamma_1, \gamma_2,\dots, \gamma_\ell]\vdash n$, then for each $1\leq i\leq \ell$, let $\Pi_{\gamma,i} \subseteq \Pi_\gamma$ be the subset of ordered set partitions $\mu$ with $n$ in the $i$th row of $\mu$ and let $\mathbb{P}_{\gamma,i}$ be the subgraph of $\mathbb{P}_\gamma$ induced by $\Pi_{\gamma,i}$. In this way, $\mathbb{P}_\gamma$ partitions into a disjoint union of subgraphs $\{\mathbb{P}_{\gamma,i}\}_{i=1}^\ell$ with edges between vertices in different subgraphs as follows: $\mu\in \mathbb{P}_{\gamma,i}$ is connected to $\nu\in\mathbb{P}_{\gamma,j}$ $(i\neq j)$ by an edge labeled by $(n-1,n)$ if and only if $\mu = (n-1,n)(\nu)$. Note that $\mathbb{P}_{\gamma,i}\cong \mathbb{P}_{\gamma'}$ where $\gamma'$ is the integer partition given by subtracting 1 from $\gamma_i$ and, if necessary, sorting the parts so they are nonincreasing. For example, Fig.\ \ref{fig:recursiveSchreier} gives the decomposition of $\PP_{[4,1,1]}$ into $\PP_{[3,1,1]},$ $\PP_{[4,0,1]},$ and $\PP_{[4,1,0]}$ with $\PP_{[4,0,1]}\cong\PP_{[4,1,0]}\cong \PP_{[4,1]}$. This decomposition allows one to construct the Schreier graphs of $\mathbb{S}_n$ 
dynamically from those for $\mathbb{S}_{n-1}$. 
Since the permutahedron is a Schreier graph, $\mathbb{P}_n = \mathbb{P}_{[1,1,\dots,1]}$, we also construct it  
dynamically using this method.

Let $\B_\gamma := \B_{\pi_1}$ denote the characteristic matrix corresponding to the reading ordered set partition $\pi_1\in\Pi_\gamma$, which is a matrix that lifts Schreier eigenvectors from the Schreier $\PP_\gamma$ to the permutahedron $\PP_n$. Then $\B_\gamma$ has a recursive structure that respects the decomposition of $\mathbb{P}_\gamma$,
and, as we show in Sec. \ref{Sec:anal_code},
this is the only lifting matrix that we need to compute for shape $\gamma$. If $\gamma = [\gamma_1, \ldots, \gamma_\ell] \vdash n$, then for $1 \le i \le n$,
let $\SS_{n,i} := \{\sigma \in \SS_n \mid \sigma(n) = i\}$.
The 
decomposition of $\B_\gamma$ is given in Prop.\ \ref{prop:Bdecomp} and illustrated in Ex.\ \ref{eg:Bdecomp}.

\begin{proposition}\label{prop:Bdecomp} If $\gamma = [\gamma_1, \ldots, \gamma_\ell] \vdash n$, then the characteristic matrix $\B_\gamma$ decomposes into block sub-matrices, 
$$
(\B_\gamma)_{i,j} = \begin{cases} 
\B_{\gamma(j)}, & \text{if $i$ is in the $j$th row of $\pi_1$,} \\
{\bf 0}, & \text{otherwise,}
\end{cases}
\qquad 1 \le i \le n,\quad 1 \le j \le \ell,
$$
where $\gamma(j)$ is obtained from $\gamma$ by subtracting 1 from $\gamma_j$.
\end{proposition}

\begin{proof} For $\sigma \in \SS_n$ and $\mu \in \Pi_\gamma$, the $(\sigma,\mu)$-entry of $\B_\gamma$ equals 1 if and only if $\sigma(\mu) = \pi_1$. If $n$ is in the $j$th block of $\mu$ and $\sigma(n) = i$, then this entry can be nonzero only if $i$ is in the $j$th row of $\pi_1$. The nonzero blocks equal $\B_{\gamma(j)}$ by definition (after ignoring $n$).
\qed\end{proof}

\begin{example}\label{eg:Bdecomp} In Fig. \ref{Fig:char_rec}, we show the recursive structure of the characteristic matrix $\B_{[2,2]}$, which corresponds to the ordered set partition $\pi_1 = \{12\vert34\}$ (compare with Fig. \ref{Fig:char_mat}, which gives this same matrix with the permutations in lexiocgraphic order).
\begin{figure}[t]
$$
\B_{[2,2]}=\B_{\begin{tikzpicture}[xscale=.20,yscale=.20,line width=0.8pt] 
\draw (0,0) rectangle (2,1); 
\path (.5,0.5) node {{\tiny $1$}}; \path (1.5,0.5) node {{\tiny $2$}}; 
\draw (0,-1) rectangle (2,0);  
\path (.5,-0.5) node {{\tiny $3$}};  \path (1.5,-0.5) node {{\tiny $4$}};  
\end{tikzpicture}
}
\setlength{\kbrowsep}{5pt}
= \kbordermatrix{
&\Pi_{[2,2],2}&\Pi_{[2,2],1} \\ 
\SS_{4,4}&\B_{[2,1]}& \cdot\\
\SS_{4,3}&\B_{[2,1]}& \cdot\\
\SS_{4,2}& \cdot &\B_{[1,2]}\\
\SS_{4,1}& \cdot &\B_{[1,2]}
}
\setlength{\kbrowsep}{15pt}
=
\kbordermatrix{
&
\begin{tikzpicture}[xscale=.25,yscale=.3,line width=0.8pt] 
\draw (0,0) rectangle (2,1); 
\path (.5,0.5) node {{\scriptsize $1$}}; \path (1.5,0.5) node {{\scriptsize $2$}}; 
\draw (0,-1) rectangle (2,0); 
\path (.5,-0.5) node {{\scriptsize $3$}};  \path (1.5,-0.5) node {{\scriptsize $4$}};  
\end{tikzpicture}
&
\begin{tikzpicture}[xscale=.25,yscale=.3,line width=0.8pt]  
\draw (0,0) rectangle (2,1); 
\path (.5,0.5) node {{\scriptsize $1$}}; \path (1.5,0.5) node {{\scriptsize $3$}}; 
\draw (0,-1) rectangle (2,0); 
\path (.5,-0.5) node {{\scriptsize $2$}};  \path (1.5,-0.5) node {{\scriptsize $4$}};
\end{tikzpicture}
&
\begin{tikzpicture}[xscale=.25,yscale=.3,line width=0.8pt] 
\draw (0,0) rectangle (2,1); 
\path (.5,0.5) node {{\scriptsize $2$}}; \path (1.5,0.5) node {{\scriptsize $3$}}; 
\draw (0,-1) rectangle (2,0); 
\path (.5,-0.5) node {{\scriptsize $1$}};  \path (1.5,-0.5) node {{\scriptsize $4$}};  
\end{tikzpicture}
&\omit\vrule&
\begin{tikzpicture}[xscale=.25,yscale=.3,line width=0.8pt] 
\draw (0,0) rectangle (2,1); 
\path (.5,0.5) node {{\scriptsize $1$}}; \path (1.5,0.5) node {{\scriptsize $4$}}; 
\draw (0,-1) rectangle (2,0); 
\path (.5,-0.5) node {{\scriptsize $2$}};  \path (1.5,-0.5) node {{\scriptsize $3$}};  
\end{tikzpicture}
&
\begin{tikzpicture}[xscale=.25,yscale=.3,line width=0.8pt] 
\draw (0,0) rectangle (2,1); 
\path (.5,0.5) node {{\scriptsize $2$}}; \path (1.5,0.5) node {{\scriptsize $4$}}; 
\draw (0,-1) rectangle (2,0); 
\path (.5,-0.5) node {{\scriptsize $1$}};  \path (1.5,-0.5) node {{\scriptsize $3$}};  
\end{tikzpicture}
&
\begin{tikzpicture}[xscale=.25,yscale=.3,line width=0.8pt] 
\draw (0,0) rectangle (2,1); 
\path (.5,0.5) node {{\scriptsize $3$}}; \path (1.5,0.5) node {{\scriptsize $4$}}; 
\draw (0,-1) rectangle (2,0); 
\path (.5,-0.5) node {{\scriptsize $1$}};  \path (1.5,-0.5) node {{\scriptsize $2$}};  
\end{tikzpicture}  \vspace{-.2in}
\\
1234&1&\cdot&\cdot&\omit\vrule&\cdot&\cdot&\cdot\\
2134&1&\cdot&\cdot&\omit\vrule&\cdot&\cdot&\cdot\\
1324&\cdot&1&\cdot&\omit\vrule&\cdot&\cdot&\cdot\\
3124&\cdot&\cdot&1&\omit\vrule&\cdot&\cdot&\cdot\\
2314&\cdot&1&\cdot&\omit\vrule&\cdot&\cdot&\cdot\\
3214&\cdot&\cdot&1&\omit\vrule&\cdot&\cdot&\cdot\\ \hline
1243&1&\cdot&\cdot&\omit\vrule&\cdot&\cdot&\cdot\\
2143&1&\cdot&\cdot&\omit\vrule&\cdot&\cdot&\cdot\\
1423&\cdot&1&\cdot&\omit\vrule&\cdot&\cdot&\cdot\\
4123&\cdot&\cdot&1&\omit\vrule&\cdot&\cdot&\cdot\\
2413&\cdot&1&\cdot&\omit\vrule&\cdot&\cdot&\cdot\\
4213&\cdot&\cdot&1&\omit\vrule&\cdot&\cdot&\cdot\\ \hline
1342&\cdot&\cdot&\cdot&\omit\vrule&1&\cdot&\cdot\\
3142&\cdot&\cdot&\cdot&\omit\vrule&\cdot&1&\cdot\\
1432&\cdot&\cdot&\cdot&\omit\vrule&1&\cdot&\cdot\\
4132&\cdot&\cdot&\cdot&\omit\vrule&\cdot&1&\cdot\\
3412&\cdot&\cdot&\cdot&\omit\vrule&\cdot&\cdot&1\\
4312&\cdot&\cdot&\cdot&\omit\vrule&\cdot&\cdot&1\\ \hline
2341&\cdot&\cdot&\cdot&\omit\vrule&1&\cdot&\cdot\\
3241&\cdot&\cdot&\cdot&\omit\vrule&\cdot&1&\cdot\\
2431&\cdot&\cdot&\cdot&\omit\vrule&1&\cdot&\cdot\\
4231&\cdot&\cdot&\cdot&\omit\vrule&\cdot&1&\cdot\\
3421&\cdot&\cdot&\cdot&\omit\vrule&\cdot&\cdot&1\\
4321&\cdot&\cdot&\cdot&\omit\vrule&\cdot&\cdot&1
}.
$$
\caption{The recursive structure of the characteristic matrix $\B_{[2,2]}$.}\label{Fig:char_rec}
     \vspace{0.1in}
\hrule height 1.5pt
\vspace{-.1in}
\end{figure}
Similarly, the recursive structure of $\B_{[4,1,1]}$, which corresponds to the ordered set partition $\pi_1 = \{1234\vert5\vert6\}$, is given by
\[
\B_{[4,1,1]} = \B_{\begin{tikzpicture}[xscale=.20,yscale=.20,line width=0.8pt] 
\draw (0,0) rectangle (4,1); 
\path (.5,0.5) node {{\tiny $1$}}; \path (1.5,0.5) node {{\tiny $2$}}; 
\path (2.5,0.5) node {{\tiny $3$}}; \path (3.5,0.5) node {{\tiny $4$}};
\draw (0,-1) rectangle (1,0); \draw (0,-2) rectangle (1,-1); 
\path (.5,-0.5) node {{\tiny $5$}};  \path (0.5,-1.5) node {{\tiny $6$}};  
\end{tikzpicture}
} 
\setlength{\kbrowsep}{5pt}
= \kbordermatrix{
&\Pi_{[4,1,1],3}&\Pi_{[4,1,1],2}&\Pi_{[4,1,1],1}\\
\SS_{6,6}&\B_{[4,1,0]}&\cdot & \cdot\\
\SS_{6,5}&\cdot &\B_{[4,0,1]}&\cdot\\
\SS_{6,4}&\cdot &\cdot &\B_{[3,1,1]}\\
\SS_{6,3}&\cdot & \cdot&\B_{[3,1,1]}\\
\SS_{6,2}&\cdot &\cdot &\B_{[3,1,1]}\\
\SS_{6,1}&\cdot & \cdot&\B_{[3,1,1]}
}. 
\]
\end{example}

\subsection{Computation of the Schreier Eigenvalues and Eigenvectors Via Deflation}   \label{Se:sch_eigs}

Frame vectors ${\boldsymbol \varphi}_{\gamma,\lambda,k,\bpi}$ of eigenvalue $\lambda$ and shape $\gamma$ are computed as lifts from the $\lambda$-eigenspace of the Schreier graph corresponding to the irreducible component $V_\gamma^\ast$ inside of $\Rbb[\Pi_\gamma]$. 
The multiplicity of $V_\gamma^\ast$
\begin{figure}[b]
    \vspace{-0.1in}
\hrule height 1.5pt
\vspace{.1in}
\centering
\includegraphics[width=.7\linewidth,page=1,trim=120 460 110 135, clip]{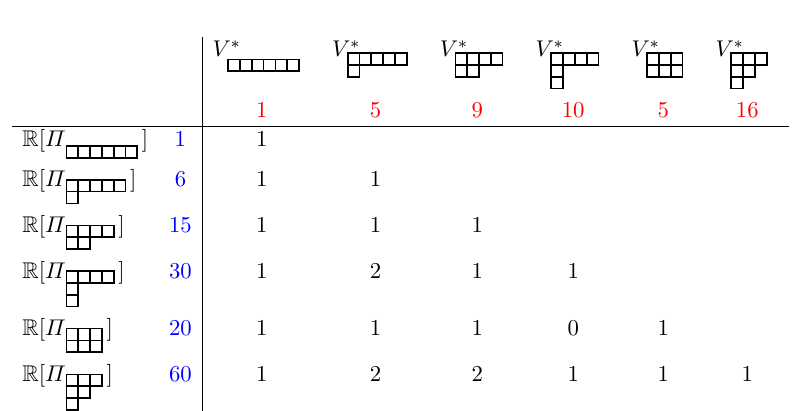}
\caption{The (upper-left quadrant of the) table of Kostka numbers $K_{\gamma,\nu}$ giving the multiplicity of the irreducible module $V_\nu^\ast$ as a component of the module $\Rbb[\Pi_\gamma]$. The dimension $d_\nu$ of $V_\nu^\ast$ is given below it in red and the dimension $m_\gamma$ of $\Rbb[\Pi_\gamma]$ is given to its right in blue. The dimensions sum as $m_\gamma = \sum_{\nu \vdash n} K_{\gamma,\mu} d_\nu$; for example, $60 = 1\cdot 1 + 2 \cdot 5 + 2 \cdot 9 + 1 \cdot 10 + 1 \cdot 5 + 1 \cdot 16$.
\label{Fig:Kostka}}
\end{figure}
in $\Rbb[\Pi_\gamma]$ is equal to one and the other irreducible components $V_\nu^\ast$ that appear in $\Rbb[\Pi_\gamma]$ correspond to shapes $\nu$ with  $\nu \triangleright \gamma$ in dominance order. See \eqref{eq:SchreierDecomposition} and Fig. \ref{Fig:Kostka}. By computing   
the Laplacian eigenvectors in any order that respects dominance, we will have already computed the 
eigenvectors ${\bf v}_{\nu,\varrho,k}$
for $\nu \triangleright \gamma$. To single out $V_\gamma^\ast$ inside $\Rbb[\Pi_\gamma]$, we 
lift the Laplacian eigenvectors from  $\PP_\nu$ up to $\PP_\gamma$, 
and find the eigenvectors associated with the eigenvalues of $V_\gamma^\ast$ as an orthogonal complement of this span. 

Interesting new theory is needed here, as the multiplicity of $V_\nu^\ast$ in $\Rbb[\Pi_\gamma]$ equals the Kostka number, and therefore we need to 
lift the eigenvectors of shape $\nu$ in $K_{\gamma,\nu}$ linearly independent ways to $\Rbb[\Pi_\gamma]$.  In Prop.~\ref{prop:Young} below, we show that lifting and projecting with set partitions constructed from ``column-strict" tableaux is the same as Young's rule for decomposing $\Rbb[\Pi_\gamma]$ into irreducible modules with multiplicity $K_{\gamma,\nu}$. Therefore, using mappings that come from different column-strict tableaux give us linearly independent images in $\Rbb[\Pi_\gamma]$ as needed. 

The Kostka number $K_{\nu,\gamma}$ equals the number of column-strict tableaux of shape $\nu$ and \emph{content} $\gamma$; that is, a tableau constructed by filling the boxes of $\nu$ with $\gamma_1$ ones, $\gamma_2$ twos, and so on, such that the rows weakly increase and the columns strictly increase. For example, the $K_{\gamma,\nu} = 3$ column-strict tableaux $T_1, T_2, T_3$ of shape $\nu=[5,4]$ and content $\gamma=[4,2,2,1]$ are shown in Fig.~\ref{Fig:KostkaFillings}. Note that from this definition,  $K_{\gamma,\nu} =0$ if $\nu \not\trianglerighteq \gamma$.  If $T$ is a column-strict tableaux of shape $\nu$ and content $\gamma$, then define $\xi_T \in \Pi_\gamma$ to be the set partition with $j$ in row $r$ if the  $j$th box of $T_i$, read in reading order, contains $r$. By reading order, we mean left-to-right across the rows from top to bottom.  The set partitions $\{\xi_{T_1}, \xi_{T_2},\xi_{T_3}\}$ are shown in Fig.~\ref{Fig:KostkaFillings}. For example, $\xi_{T_2}$ has 5 in row 3 because the 5th box of $T_2$ contains 3.

\begin{proposition}\label{prop:Young}
If $\pi_1 \in \Pi_\nu$ is the reading-order set partition and $\{\xi_i\}_{i=1}^{K_{\gamma,\nu}} \subseteq \Pi_\gamma$ are the set partitions corresponding to the column-strict tableaux $\{T_i\}_{i=1}^{K_{\gamma,\nu}}$ of shape $\nu$ and content $\gamma$, then the matrices  ${\bf T}_{\xi_i,\pi_1} = {\bf B}_{\xi_i}^\top {\bf B}_{\pi_1}$, $1 \le i \le K_{\gamma,\nu}$, give $\SS_n$-module isomorphisms between $V_\nu^\ast \subseteq \Rbb[\Pi_\nu]$ and $K_{\gamma,\nu}$ linearly independent copies $\{V_{\nu,i}^\ast\}_{i=1}^{K_{\gamma,\nu}}$ of $V_\nu^\ast$ in $\Rbb[\Pi_\gamma]$. 
\end{proposition}

\begin{proof} The transformation ${\bf T}_{\xi_i,\pi_1}$ is a right $\SS_n$-module homomorphism, since both ${\bf B}_{\xi_i}, {\bf B}_{\pi_1}$ are.  We will show that ${\bf T}_{\xi_i,\pi_1} = c_i \Theta_{T_i}$, for a nonzero scalar $c_i$, where $\Theta_{T_i}$ is the $\SS_n$-module isomophism given by Young's rule defined in \cite[Def.~2.9.3] {sagan2013symmetric}. These transformations $\Theta_{T_i}$ are known to give $K_{\gamma,\nu}$ linearly independent isomorphisms from $V_\nu^\ast \in\Rbb[\Pi_\nu]$ to $V_{\nu,i}^\ast\in \Rbb[\Pi_\gamma]$ by \cite[Thm.~2.11.2]{sagan2013symmetric}.  

Since ${\bf T}_{\xi_i,\pi_1}$ is a module homomorphism, and $\SS_n$ acts transitively on $\Pi_\nu$, it is sufficient to show that they are equal at $\e_{\pi_1}$.
Let $\SS_{\pi_1} = \{ \sigma \in \SS_n \mid \sigma(\pi_1) = \pi_1\}$ be the stabilizer subgroup of $\pi_1$. Then
$$
{\bf B}_{\xi_i}^\top {\bf B}_{\pi_1} \e_{\pi_1} = \sum_{\sigma \in \SS_{\pi_1}} {\bf B}_{\xi_i}^\top  \e_\sigma = \sum_{\sigma \in \SS_{\pi_1}} \!\sum_{\mu \in \Pi_\gamma, \sigma(\mu) = \xi_i} \!\! \e_{\mu} = \sum_{\sigma \in \SS_{\pi_1}} \e_{\sigma^{-1} (\xi_i)} = \sum_{\sigma \in \SS_{\pi_1}} \e_{\sigma (\xi_i)}.
$$
Thus the coefficient of $\e_\mu$ in ${\bf B}_{\xi_i}^\top {\bf B}_{\pi_1} \e_{\pi}$ is $c_{\mu,\xi_i,\pi_1} = \# \{\sigma \in \SS_n \mid \sigma(\xi_i) = \mu, \sigma(\pi_1) = \pi_1\}$.  This set is a coset of the stabilizer of $(\xi_i,\pi_1)$ under the action of $\SS_n$ on $\Pi_\gamma \times \Pi_\nu$, so its cardinality is constant for all $\mu$, which we denote by $c_{\xi_i}$ (equal to the size of this stabilizer). By comparing with the action of $\Theta_{T_i}$ in \cite[Def.~2.9.3] {sagan2013symmetric} (identifying ordered set partitions here with ``tabloids" there), we have that ${\bf T}_{\xi_i,\pi_1} = c_{\xi} \Theta_{T_i}$, where $\xi_i$ is the set partition filled according to $T_i$. 
\qed\end{proof}

\begin{figure}[t]
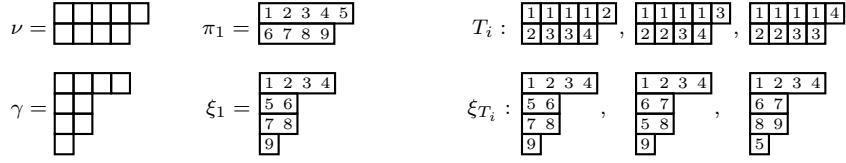

$$
\begin{array}{llcclll}
\nu=\tikz[baseline]{
\foreach \i in {0,...,4}  { \draw[xscale=.25,yscale=.27,line width=0.8pt] (\i,0.5) rectangle (\i+1,1.5); }   
\foreach \i in {0,...,3}  { \draw[xscale=.25,yscale=.27,line width=0.8pt] (\i,-0.5) rectangle (\i+1,0.5); } } \qquad &
\pi_1=\tikz[baseline]{
\draw[xscale=.25,yscale=.27,line width=0.8pt] (0,0.5) rectangle (5,1.5);   
     \path[xscale=.25,yscale=.27,line width=0.8pt] (0.5,1.0) node {$\scriptstyle{1}$};
     \path[xscale=.25,yscale=.27,line width=0.8pt] (1.5,1.0) node {$\scriptstyle{2}$};
     \path[xscale=.25,yscale=.27,line width=0.8pt] (2.5,1.0) node {$\scriptstyle{3}$};
     \path[xscale=.25,yscale=.27,line width=0.8pt] (3.5,1.0) node {$\scriptstyle{4}$};
     \path[xscale=.25,yscale=.27,line width=0.8pt] (4.5,1.0) node {$\scriptstyle{5}$};
\draw[xscale=.25,yscale=.27,line width=0.8pt] (0,-0.5) rectangle (4,0.5);  
     \path[xscale=.25,yscale=.27,line width=0.8pt] (0.5,0.0) node {$\scriptstyle{6}$};
     \path[xscale=.25,yscale=.27,line width=0.8pt] (1.5,0.0) node {$\scriptstyle{7}$};
     \path[xscale=.25,yscale=.27,line width=0.8pt] (2.5,0.0) node {$\scriptstyle{8}$};
     \path[xscale=.25,yscale=.27,line width=0.8pt] (3.5,0.0) node {$\scriptstyle{9}$};} 
&\hskip.5in& T_i: &
\tikz[baseline]{
\foreach \i in {0,...,4}  { \draw[xscale=.25,yscale=.27,line width=0.8pt] (\i,0.5) rectangle (\i+1,1.5); }  
\foreach \i in {0,...,3}  { \path[xscale=.25,yscale=.27,line width=0.8pt] (\i.5,1.0) node {$\scriptstyle{1}$}; }  
\path[xscale=.25,yscale=.27,line width=0.8pt] (4.5,1.0) node {$\scriptstyle{2}$};  
\foreach \i in {0,...,3}  { \draw[xscale=.25,yscale=.27,line width=0.8pt] (\i,-0.5) rectangle (\i+1,0.5); }  
\path[xscale=.25,yscale=.27,line width=0.8pt] (0.5,0.0) node {$\scriptstyle{2}$};  
\path[xscale=.25,yscale=.27,line width=0.8pt] (1.5,0.0) node {$\scriptstyle{3}$};
\path[xscale=.25,yscale=.27,line width=0.8pt] (2.5,0.0) node {$\scriptstyle{3}$};
\path[xscale=.25,yscale=.27,line width=0.8pt] (3.5,0.0) node {$\scriptstyle{4}$};}, &
\tikz[baseline]{
\foreach \i in {0,...,4}  { \draw[xscale=.25,yscale=.27,line width=0.8pt] (\i,0.5) rectangle (\i+1,1.5); }  
\foreach \i in {0,...,3}  { \path[xscale=.25,yscale=.27,line width=0.8pt] (\i.5,1.0) node {$\scriptstyle{1}$}; }  
\path[xscale=.25,yscale=.27,line width=0.8pt] (4.5,1.0) node {$\scriptstyle{3}$};  
\foreach \i in {0,...,3}  { \draw[xscale=.25,yscale=.27,line width=0.8pt] (\i,-0.5) rectangle (\i+1,0.5); }  
\path[xscale=.25,yscale=.27,line width=0.8pt] (0.5,0.0) node {$\scriptstyle{2}$};  
\path[xscale=.25,yscale=.27,line width=0.8pt] (1.5,0.0) node {$\scriptstyle{2}$};
\path[xscale=.25,yscale=.27,line width=0.8pt] (2.5,0.0) node {$\scriptstyle{3}$};
\path[xscale=.25,yscale=.27,line width=0.8pt] (3.5,0.0) node {$\scriptstyle{4}$};},&
\tikz[baseline]{
\foreach \i in {0,...,4}  { \draw[xscale=.25,yscale=.27,line width=0.8pt] (\i,0.5) rectangle (\i+1,1.5); }  
\foreach \i in {0,...,3}  { \path[xscale=.25,yscale=.27,line width=0.8pt] (\i.5,1.0) node {$\scriptstyle{1}$}; }  
\path[xscale=.25,yscale=.27,line width=0.8pt] (4.5,1.0) node {$\scriptstyle{4}$};  
\foreach \i in {0,...,3}  { \draw[xscale=.25,yscale=.27,line width=0.8pt] (\i,-0.5) rectangle (\i+1,0.5); }  
\path[xscale=.25,yscale=.27,line width=0.8pt] (0.5,0.0) node {$\scriptstyle{2}$};  
\path[xscale=.25,yscale=.27,line width=0.8pt] (1.5,0.0) node {$\scriptstyle{2}$};
\path[xscale=.25,yscale=.27,line width=0.8pt] (2.5,0.0) node {$\scriptstyle{3}$};
\path[xscale=.25,yscale=.27,line width=0.8pt] (3.5,0.0) node {$\scriptstyle{3}$};}
\\ \\
\gamma=\tikz[baseline]{
\foreach \i in {0,...,3}  { \draw[xscale=.25,yscale=.27,line width=0.8pt] (\i,1) rectangle (\i+1,2); }   
\foreach \i in {0,...,1}  { \draw[xscale=.25,yscale=.27,line width=0.8pt] (\i,0) rectangle (\i+1,1); }
\foreach \i in {0,...,1}  { \draw[xscale=.25,yscale=.27,line width=0.8pt] (\i,-1) rectangle (\i+1,0); }
\draw[xscale=.25,yscale=.27,line width=0.8pt] (0,-2) rectangle (1,-1);  } 
& 
\,\xi_1=\tikz[baseline]{
\draw[xscale=.25,yscale=.27,line width=0.8pt] (0,1) rectangle (4,2); 
     \path[xscale=.25,yscale=.27,line width=0.8pt] (0.5,1.5) node {$\scriptstyle{1}$};
     \path[xscale=.25,yscale=.27,line width=0.8pt] (1.5,1.5) node {$\scriptstyle{2}$};
     \path[xscale=.25,yscale=.27,line width=0.8pt] (2.5,1.5) node {$\scriptstyle{3}$};
     \path[xscale=.25,yscale=.27,line width=0.8pt] (3.5,1.5) node {$\scriptstyle{4}$};
\draw[xscale=.25,yscale=.27,line width=0.8pt] (0,0) rectangle (2,1); 
     \path[xscale=.25,yscale=.27,line width=0.8pt] (0.5,0.5) node {$\scriptstyle{5}$};
     \path[xscale=.25,yscale=.27,line width=0.8pt] (1.5,0.5) node {$\scriptstyle{6}$};
\draw[xscale=.25,yscale=.27,line width=0.8pt] (0,-1) rectangle (2,0); 
     \path[xscale=.25,yscale=.27,line width=0.8pt] (0.5,-0.5) node {$\scriptstyle{7}$};
     \path[xscale=.25,yscale=.27,line width=0.8pt] (1.5,-0.5) node {$\scriptstyle{8}$};
\draw[xscale=.25,yscale=.27,line width=0.8pt] (0,-2) rectangle (1,-1);  
     \path[xscale=.25,yscale=.27,line width=0.8pt] (0.5,-1.5) node {$\scriptstyle{9}$};} 
&&\xi_{T_i}:&
\tikz[baseline]{
\draw[xscale=.25,yscale=.27,line width=0.8pt] (0,1) rectangle (4,2); 
     \path[xscale=.25,yscale=.27,line width=0.8pt] (0.5,1.5) node {$\scriptstyle{1}$};
     \path[xscale=.25,yscale=.27,line width=0.8pt] (1.5,1.5) node {$\scriptstyle{2}$};
     \path[xscale=.25,yscale=.27,line width=0.8pt] (2.5,1.5) node {$\scriptstyle{3}$};
     \path[xscale=.25,yscale=.27,line width=0.8pt] (3.5,1.5) node {$\scriptstyle{4}$};
\draw[xscale=.25,yscale=.27,line width=0.8pt] (0,0) rectangle (2,1); 
     \path[xscale=.25,yscale=.27,line width=0.8pt] (0.5,0.5) node {$\scriptstyle{5}$};
     \path[xscale=.25,yscale=.27,line width=0.8pt] (1.5,0.5) node {$\scriptstyle{6}$};
\draw[xscale=.25,yscale=.27,line width=0.8pt] (0,-1) rectangle (2,0); 
     \path[xscale=.25,yscale=.27,line width=0.8pt] (0.5,-0.5) node {$\scriptstyle{7}$};
     \path[xscale=.25,yscale=.27,line width=0.8pt] (1.5,-0.5) node {$\scriptstyle{8}$};
\draw[xscale=.25,yscale=.27,line width=0.8pt] (0,-2) rectangle (1,-1);  
     \path[xscale=.25,yscale=.27,line width=0.8pt] (0.5,-1.5) node {$\scriptstyle{9}$};}, &
\tikz[baseline]{
\draw[xscale=.25,yscale=.27,line width=0.8pt] (0,1) rectangle (4,2); 
     \path[xscale=.25,yscale=.27,line width=0.8pt] (0.5,1.5) node {$\scriptstyle{1}$};
     \path[xscale=.25,yscale=.27,line width=0.8pt] (1.5,1.5) node {$\scriptstyle{2}$};
     \path[xscale=.25,yscale=.27,line width=0.8pt] (2.5,1.5) node {$\scriptstyle{3}$};
     \path[xscale=.25,yscale=.27,line width=0.8pt] (3.5,1.5) node {$\scriptstyle{4}$};
\draw[xscale=.25,yscale=.27,line width=0.8pt] (0,0) rectangle (2,1); 
     \path[xscale=.25,yscale=.27,line width=0.8pt] (0.5,0.5) node {$\scriptstyle{6}$};
     \path[xscale=.25,yscale=.27,line width=0.8pt] (1.5,0.5) node {$\scriptstyle{7}$};
\draw[xscale=.25,yscale=.27,line width=0.8pt] (0,-1) rectangle (2,0); 
     \path[xscale=.25,yscale=.27,line width=0.8pt] (0.5,-0.5) node {$\scriptstyle{5}$};
     \path[xscale=.25,yscale=.27,line width=0.8pt] (1.5,-0.5) node {$\scriptstyle{8}$};
\draw[xscale=.25,yscale=.27,line width=0.8pt] (0,-2) rectangle (1,-1);  
     \path[xscale=.25,yscale=.27,line width=0.8pt] (0.5,-1.5) node {$\scriptstyle{9}$};}, &
\tikz[baseline]{
\draw[xscale=.25,yscale=.27,line width=0.8pt] (0,1) rectangle (4,2); 
     \path[xscale=.25,yscale=.27,line width=0.8pt] (0.5,1.5) node {$\scriptstyle{1}$};
     \path[xscale=.25,yscale=.27,line width=0.8pt] (1.5,1.5) node {$\scriptstyle{2}$};
     \path[xscale=.25,yscale=.27,line width=0.8pt] (2.5,1.5) node {$\scriptstyle{3}$};
     \path[xscale=.25,yscale=.27,line width=0.8pt] (3.5,1.5) node {$\scriptstyle{4}$};
\draw[xscale=.25,yscale=.27,line width=0.8pt] (0,0) rectangle (2,1); 
     \path[xscale=.25,yscale=.27,line width=0.8pt] (0.5,0.5) node {$\scriptstyle{6}$};
     \path[xscale=.25,yscale=.27,line width=0.8pt] (1.5,0.5) node {$\scriptstyle{7}$};
\draw[xscale=.25,yscale=.27,line width=0.8pt] (0,-1) rectangle (2,0); 
     \path[xscale=.25,yscale=.27,line width=0.8pt] (0.5,-0.5) node {$\scriptstyle{8}$};
     \path[xscale=.25,yscale=.27,line width=0.8pt] (1.5,-0.5) node {$\scriptstyle{9}$};
\draw[xscale=.25,yscale=.27,line width=0.8pt] (0,-2) rectangle (1,-1);  
     \path[xscale=.25,yscale=.27,line width=0.8pt] (0.5,-1.5) node {$\scriptstyle{5}$};}      
\end{array}
$$
\caption{When $\nu = [5,4]$ and $\gamma = [4,2,2,1]$, the Kostka number is $K_{\gamma,\nu} = 3$. The three column-strict tableaux $T_1, T_2,T_3$ of shape $\nu$ and content $\gamma$ are shown here.  They each have $\gamma_1 = 4$ ones, $\gamma_2 = 2$ twos, $\gamma_3 =2$ threes, and $\gamma_4 = 1$ four. The set partitions $\pi_1 \in \Pi_\nu$ and $\xi_1 \in \Pi_\gamma$ are in reading order, and $\xi_{T_1}, \xi_{T_2},\xi_{T_3} \subseteq \Pi_\gamma$ are the set-partitions with $T_i$ fillings; that is $\xi_{T_i}$ has $j$ is row $r$ if the $j$th box of $T_i$ (read in reading order) contains $r$.
\label{Fig:KostkaFillings}}
     \vspace{0.1in}
\hrule height 1.5pt
\vspace{-.1in}
\end{figure}

To summarize, from the structure of the Schreier graphs, we know that the graph Laplacian of the Schreier graph $\PP_\gamma$ has the following spectral decomposition:
\begin{align*}
\L_{\PP_\gamma}
=\sum_{\lambda \in \Lambda_\gamma} \sum_{k=1}^{\kappa_{\gamma,\lambda}} \lambda  {\bf v}_{\gamma, \lambda,k} {\bf v}_{\gamma, \lambda,k}^{\top}
+\sum_{\nu  \vartriangleright \gamma} \sum_{i=1}^{K_{\gamma,\nu}} \sum_{\lambda \in \Lambda_\nu}\sum_{k=1}^{\kappa_{\nu,\lambda}} \lambda ({\bf T}_{\xi_i,\pi_1}^{\nu,\gamma} {\bf v}_{\nu, \lambda,k})({\bf T}_{\xi_i,\pi_1}^{\nu,\gamma} {\bf v}_{\nu, \lambda,k})^{\top}.
\end{align*}
Thus, when the Laplacian eigenvectors of the Schreier graphs that precede $\PP_\gamma$ in dominance order have already been computed, to find an orthonormal eigenbasis for $V_\gamma^*$, it suffices to diagonalize the rank $d_\gamma$ matrix
\begin{align}\label{Eq:eig_deflation}
\L_{\PP_\gamma}-\sum_{\nu  \vartriangleright \gamma} \sum_{i=1}^{K_{\gamma,\nu}} \sum_{\lambda \in \Lambda_\nu}\sum_{k=1}^{\kappa_{\nu,\lambda}} \lambda ({\bf T}_{\xi_i,\pi_1}^{\nu,\gamma} {\bf v}_{\nu, \lambda,k})({\bf T}_{\xi_i,\pi_1}^{\nu,\gamma} {\bf v}_{\nu, \lambda,k})^{\top},
\end{align}
as opposed to the rank $m_\gamma$ matrix $\L_{\PP_\gamma}$. The complexity of forming and computing the eigendecomposition of the matrix in \eqref{Eq:eig_deflation} is ${\cal O}(m_\gamma d_\gamma^2 + m_\gamma^2(\sum_{\nu  \vartriangleright \gamma} d_\nu K_{\gamma,\nu}))$.

\subsection{Efficient Computation of the Analysis Coefficients} \label{Sec:anal_code}

As detailed in \eqref{Eq:anal_coeff_two}, the analysis coefficients associated with shape $\gamma$ can be computed either by lifting each eigenvector ${\bf v}_{\gamma,\lambda,k}$ up to the permutahedron in $z_\gamma$ different ways and taking the inner product of each with the signal ${\bf f}$, or by projecting the signal down to the Schreier graph in $z_\gamma$ different ways and taking the inner product between each of the projections and each eigenvector  ${\bf v}_{\gamma,\lambda,k}$. Since the characteristic matrix used to lift or project between the Schreier $\PP_\gamma$ and the permutahedron $\PP_n$ is a sparse matrix with $n!$ entries equal to 1 (one per row) and the remainder equal to 0, the respective complexities of these two approaches are ${\cal O}(n!d_\gamma z_\gamma)$ and ${\cal O}(z_\gamma(n!+d_\gamma m_\gamma))$.
However, the more problematic issue with both of these approaches is the memory required to store all $z_\gamma$ characteristic matrices $\B_{\bpi}$ associated with each shape $\gamma$, which is $4n!z_\gamma$ bytes (e.g., storing these $\B_{\bpi}$ matrices for all $\bpi \in \bPi_{[4,3,2,1]}$ alone would require approximately 383GB). 

To reduce the required memory, we use the following method which only requires storing the single characteristic matrix ${\bf B}_\gamma={\bf B}_{\pi_1}$ associated with the reading-ordered set partition $\pi_1$ for each shape $\gamma$; these matrices are computed 
dynamically, as detailed in Sec. \ref{Se:sch_adj}. For each $\bpi \in \bPi_\gamma$, we let  $\sigma \in \SS_n$ be a permutation such that $\sigma(\pi_1) = \bpi$. Then we have
$$
\langle {\bf f},{\boldsymbol \varphi}_{\gamma,\lambda,k,\bpi} \rangle
=\bar{c}_\gamma \langle{\bf f},{\bf B}_{{\bpi}} {\bf v}_{\gamma,\lambda,k}\rangle
=\bar{c}_\gamma \langle{\bf f},{\bf B}_{\sigma(\pi_1)} {\bf v}_{\gamma,\lambda,k}\rangle
=\bar{c}_\gamma \langle{\bf f},\rho_L(\sigma){\bf B}_{\pi_1} {\bf v}_{\gamma,\lambda,k}\rangle
=\bar{c}_\gamma \langle {\bf B}^{\top}_{{\pi}_1} \rho_L(\sigma^{-1}) {\bf f}, {\bf v}_{\gamma,\lambda,k} \rangle,
$$
where the third equality follows from Prop. \ref{Pr:lift_symmetry} and the fourth equality follows from the orthogonality of $\rho_L(\sigma)$. 
Therefore, to compute the analysis coefficient $\langle {\bf f},{\boldsymbol \varphi}_{\gamma,\lambda,k} \rangle$ we can compute instead  $\langle {\bf B}^{\top}_{{\pi}_1} \rho_L(\sigma^{-1}) {\bf f}, {\bf v}_{\gamma,\lambda,k} \rangle$, which amounts to reordering the data vector ${\bf f}$ by $\sigma^{-1}$, projecting it down to the Schreier graph using a single matrix ${\bf B}_{\pi_1}$, and then taking the inner product with ${\bf v}_{\gamma,\lambda,k}$. 

The permutation $\sigma^{-1}$ is recorded as the product of a minimal sequence of adjacent transpositions, 
and the reorderings are computed by sequentially applying these transpositions.  Thus, the third and final phase of the setup portion consists of computing $\sigma^{-1}$ for each shape $\gamma$ and every $\bpi \in \bPi_\gamma$ by constructing a  \emph{path} in $\PP_\gamma$ from the reading-order set partition $\pi_1\in\bPi_\gamma$ to $\bpi\in\bPi_\gamma$, which 
is a sequence ($\pi_1,\pi_2,\dots,\pi_r$) such that $\pi_{i+1} = (j_i,j_i+1)(\pi_i)$ for $i=1,\dots,r-1$ and $\pi_r = \bpi$. Thus, paths give the sequence of adjacent transpositions that transform $\pi_1$ into $\bpi$. A path is \emph{minimal} if there is no path in $\PP_\gamma$ from $\pi_1$ to $\bpi$ with fewer edges (equivalently, no way to transform $\pi_1$ to $\bpi$ with fewer swaps). Minimal paths can be constructed via a breadth-first search \cite{bondy2000graph}. One example of a minimal path constructed in this manner is shown in Fig. \ref{fig:recursiveSchreier}. Moreover, all minimal paths to $\bar\pi \in \bar\Pi_\gamma$ live entirely in $\bar\Pi_\gamma$. To see this, we use the fact that the length of a minimal path to $\pi \in \Pi_\gamma$ equals the number of inversions in $\pi$, where an inversion is a pair $(i,j)$ with $i < j$ and $j$ in a higher row than $i$ in $\pi$  (see \cite{AH} for a proof). Thus, inversion numbers go up at each step in a minimal path from $\pi_1$ to $\pi$. A minimal path never leaves $\bar \Pi_\gamma$ and returns to $\bar \Pi_\gamma$, for that would require the number of inversions to go down. So for each $\bpi\in\Bar\Pi_\gamma$, we pre-compute and save the (length $n!$) permutation that rearranges ${\bf f}$ into $\rho_L(\sigma^{-1}) {\bf f}$. These  permutations are constructed by first 
dynamically
making the permutation corresponding to each of the $n-1$ adjacent swaps. Then we perform a tree traversal on the Schreier graph and multiply adjacent swaps to get the permutation corresponding to each Schreier vertex in $\bPi_\gamma$. For each Schreier graph $\PP_\gamma$, the complexity of computing these permutations is ${\cal O}(z_\gamma n!)$.
 
Finally, in the analysis portion of the code (the only part that is dependent on the data), for each $\gamma \vdash n$ and each $\bpi\in\Bar\Pi_\gamma$,  the vector ${\bf f}$ is reordered into $\rho_L(\sigma^{-1}) {\bf f}$ by the stored permutation and projected by ${\bf B}^{\top}_{{\pi}_1}$ to the Schreier graph $\PP_\gamma$, where its inner product with each of the Laplacian eigenvectors $\{{\bf v}_{\gamma,\lambda,k}\}$ is computed and multiplied by the constant $\bar{c}_\gamma$. This portion has complexity ${\cal O}(z_\gamma(2n!+m_\gamma d_\gamma))$ for each shape $\gamma$.
  
 \begin{remark}
 The implementation as described thus far requires us to store $z_\gamma$ permutation vectors of length $n!$ for each shape, which can lead to memory issues as $n$ grows (e.g., for $n=10$, the permutation vectors associated with the liftings in $\bar{\bPi}_{[4,3,2,1]}$ require approximatley 170 GB of memory). Thus, for larger $n$, we also implement an alternative version of the code that is more memory efficient. In this second version, we save only the permutations corresponding to the adjacent transpositions and the lists of adjacent transpositions leading from reading-ordered set partitions to the other vertices of the Schreier graphs. During the analysis phase, we perform a tree traversal of each Schreier graph, and at each step, the data vector is permuted by the adjacent transposition corresponding to the edge in the tree. The rearranged vectors are then projected down to the corresponding Schreier graph, where the inner products with the Laplacian eigenvectors of $\PP_\gamma$ are computed. While this variant is more efficient from a memory standpoint, the downside is that the work of computing  the permutations that rearrange ${\bf f}$ into each $\rho_L(\sigma^{-1}) {\bf f}$ from phase 3 of the setup portion now needs to be done for each new data vector ${\bf f}$.
 \end{remark}

\subsection{Subsampling of the Dictionary Atoms} \label{Se:subsampling}
The total number of dictionary atoms in $\D$ from Thm. \ref{Th:tight_frame} is $\sum_{\gamma \vdash n} d_{\gamma} m_{\gamma}$. We briefly mention three ways to reduce the number of atoms in order to improve the computational efficiency of the applying the analysis operator. 

First, we always use the less redundant dictionary $\bar{\D}$ from Thm. \ref{Th:reduced_tight_frame}, which provides exactly the same information but avoids identical atoms in order to reduce the total number of atoms to $\sum_{\gamma \vdash n} d_{\gamma} z_{\gamma}$. However, this quantity still grows faster than $n!$, meaning that the redundancy of the dictionary $\bar{\D}$ also increases as $n$ increases. 

Second, in many applications, the most relevant information lies in the isotypic components associated with the first handful of symmetry types and/or the Laplacian eigenspaces of the permutahedron associated with the lowest eigenvalues. In this case, we do not need to compute all of the analysis coefficients, which significantly reduces the overall complexity, as the computational bottlenecks lie in the symmetry types that are later in the dominance order. If the energy decomposition onto each isotypic component is still desired, we can leverage the bipartite nature of the permutahedron by computing the inner products between the atoms associated with the transpose shape with the element-wise product of the signal and the sign vector of the permutahedron; for example, 
$$\sum_{z=1}^{15}\left| \left\langle  {\bf f},  {\boldsymbol \varphi}_{{\begin{tikzpicture}[scale=.15,line width=1.0pt] 
\draw (0,0) rectangle (1,1); \draw (1,0) rectangle (2,1); 
\draw (0,-1) rectangle (1,0); \draw (1,-1) rectangle (2,0); 
\draw (0,-2) rectangle (1,-1); 
\end{tikzpicture}},7.1774,\bpi_z} \right\rangle \right|^2 = \sum_{z=1}^{10} \left| \left\langle  \bar{\bf f},  {\boldsymbol \varphi}_{{\begin{tikzpicture}[scale=.15,line width=1.0pt] 
\draw (0,0) rectangle (1,1); \draw (1,0) rectangle (2,1); \draw (2,0) rectangle (3,1); 
\draw (0,-1) rectangle (1,0); \draw (1,-1) rectangle (2,0); 
\end{tikzpicture}},0.8226,\bpi_z} \right\rangle \right|^2,$$ where $\bar{f}(\sigma)=\sign(\sigma)f(\sigma)$. 
We denote by $\H_n$ the set of integer partitions of $n$ that cannot be written as the transpose of a symmetry type that precedes it in lexicographic order (e.g., when $n=10$, $\H_n$ contains 22 of the 42 symmetry types, which are shown below in Fig. \ref{fig:dominance}). 

Third, and most efficaciously, we can simply compute transform coefficients for atoms associated with the first $k$ symmetry types, which again is often where the most interesting information resides. As an example, if $n=10$, $|\D|=\sum_{\gamma \vdash n} d_{\gamma} m_{\gamma}=419,571,370$ (redundancy factor of 115.6); $|\bar{\D}|=\sum_{\gamma \vdash n} d_{\gamma} z_{\gamma}=44,711,456$ (10.7\% of $|\D|$, redundancy factor of 12.3); $\sum_{\gamma \in \H_n} d_{\gamma} z_{\gamma}=18,004,348$ (40.3\% of $|\bar{\D}|$, redundancy factor of 5.0); there are $\sum_{\gamma \in \H_n^8} d_{\gamma} z_{\gamma}=98,866$ atoms generated from the first $k=8$ shapes (0.55\% of $\sum_{\gamma \in \H_n} d_{\gamma} z_{\gamma}$, redundancy
\begin{wrapfigure}[24]{r}{0.63\textwidth}
\vspace{-.35in}
\hspace{-.1in}
\begin{minipage}{.85\linewidth}
$$
\begin{array}{| c | c | c | c | c | c |} \hline
n 
&  \displaystyle{\sum_{\gamma \in {\cal F}_n}\! \! \! z_\gamma} 
& \displaystyle{\sum_{\gamma \in {\cal F}_n}\! \!  \! d_\gamma z_\gamma}
& \displaystyle{\sum_{\gamma \in {\cal F}_n}\! \!  \! d_\gamma z_\gamma/n!}
& \displaystyle{\sum_{\gamma \in {\cal F}_n^8}\!\! \!   d_\gamma z_\gamma}
& \displaystyle{\sum_{\gamma \in {\cal F}_n^8}\! \!\!   \min\{2,d_\gamma\} z_\gamma}
\\ \hline
 3 & 4 & 7 & 1.167 & 7 & 7 \\
 4 & 8 & 19 & 0.792 & 19 & 15 \\
 5 & 26 & 131 & 1.092 & 131 & 51 \\
 6 & 107 & 1326 & 1.842 & 1326 & 213 \\
 7 & 295 & 6987 & 1.386 & 6987 & 589 \\
 8 & 1570 & 96895 & 2.403 & 15975 & 759 \\
 9 & 5507 & 843313 & 2.324 & 42541 & 1255 \\
 10 & 34427 & 18004348 & 4.962 & 98866 & 1821 \\
 11 & 139877 & 181831409 & 4.555 & 211751 & 2553 \\
 12 & 823242 & 3657722234 & 7.636 & 424832 & 3479 \\
 \hline
 15 & 1.6839 \cdot 10^8 & 2.9336 \cdot 10^{13} & 22.434 & 2562211 & 7731 \\
 20 & 2.2896 \cdot  10^{12} & 2.9314\cdot 10^{20} & 120.490 & 26769956 & 21891 \\
 25 & 8.0413\cdot 10^{16} & 1.7136 \cdot 10^{28}& 1104.810 & 168121101 & 49551 \\
 30 & 5.0450\cdot 10^{21} & 2.7008\cdot 10^{36} & 10182.000 & 758346521 & 97211 \\
 \hline
\end{array}
$$
\end{minipage}
\caption{Number of dictionary atoms for different values of $n$, the number of candidates. The third column ($\sum_{\gamma \in {\cal F}_n}\! \!  \! d_\gamma z_\gamma$) gives the number of atoms in the proposed dictionary ${\bar{\cal D}}$ associated with the shapes in $\H_n$, and the fourth column is the resulting redundancy of the transform (number of atoms divided by the length of the data vector ${\bf f}$). The fifth and sixth columns contain the number of atoms when the dictionary is further restricted to those atoms associated with the top $k=8$ shapes in $\H_n^8$ and either all eigenvalues in those shapes or at most two eigenvalues in each shape, respectively.}\label{Fig:comp_tab}
\end{wrapfigure}
factor of 0.03); and, finally, there are $\sum_{\gamma \in \H_n^8} \min\{2,d_{\gamma}\} z_{\gamma}=1,821$ atoms generated from the first $k=8$ 
shapes and up to two eigenvalues from each shape (1.8\% of the previous quantity, redundancy factor of 0.0005). 
These values are shown in Fig. \ref{Fig:comp_tab} for a wider range of $n$. Importantly, if we only plan to use atoms generated from the first $k$ shapes in the analysis, we only need to compute the corresponding Schreier adjacency matrices, lifting matrices, paths, and eigendecompositions described in Secs. \ref{Se:sch_adj} and \ref{Se:sch_eigs} for these shapes. We demonstrate the resulting computational savings in Sec. \ref{Se:comp_summary}. Applications that could benefit from such a reduced transform with the top $k$ shapes include lossy compression and machine learning problems, for which the reduced transform coefficients can serve as low-dimensional feature vectors for the high-dimensional ranked data. 

\subsection{Computational Summary} \label{Se:comp_summary}
To summarize, when we perform the transform using all atoms associated with the shapes in $\H_n$, phase 1 of the setup portion is negligible when compared to phases 2 and 3 of the setup and the computation of the analysis coefficients.  Fig. \ref{fig:comp_times} shows the times required for our MATLAB implementations to perform the proposed transform over all shapes in $\H_n$, on a 2.3 GHz MacBook Pro laptop with 32GB of RAM. The code is not yet optimized in the sense that we have not precompiled any C/C++ subroutines into MEX functions. As an example, for $n=9$ candidates (362,880 vertices or possible rankings), the main implementation of the code performs the offline computations in the three phases of the setup portion in approximately 43 seconds, and then computes the analysis coefficients for each data vector ${\bf f}$ in approximately 26 seconds.

 \begin{figure}[t]
\centering
\hfill
\begin{minipage}{.45\linewidth}
 \centering
    \includegraphics[width = \linewidth,page=1]{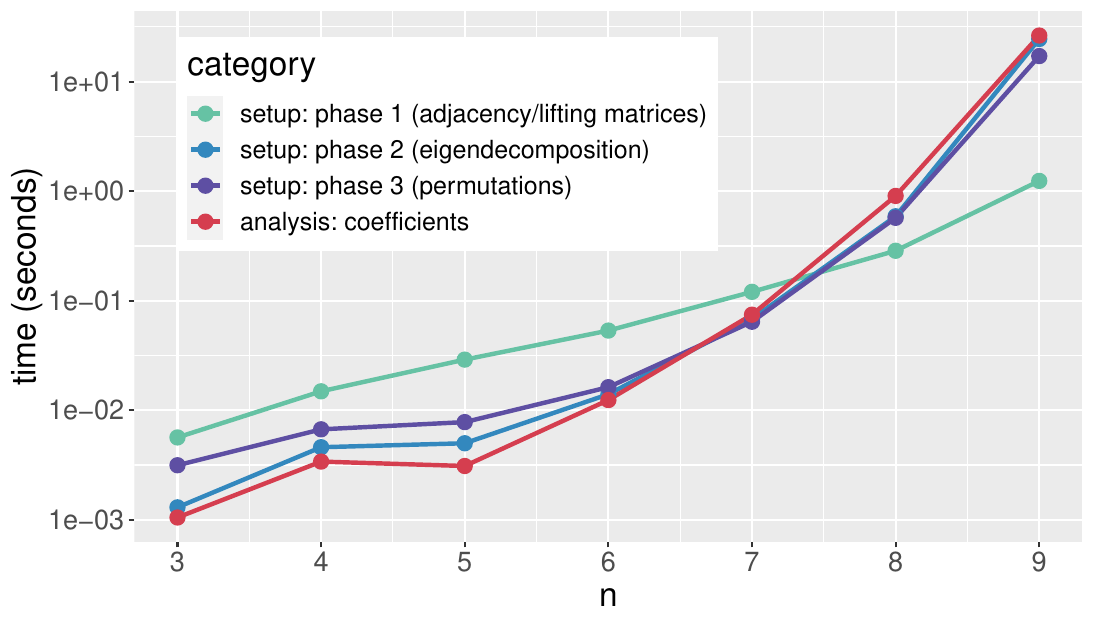} 
    \centerline{\small{~~~~~~~~(a)}}
\end{minipage}
\hfill
\begin{minipage}{.45\linewidth}
    \includegraphics[width = \linewidth,page=1]{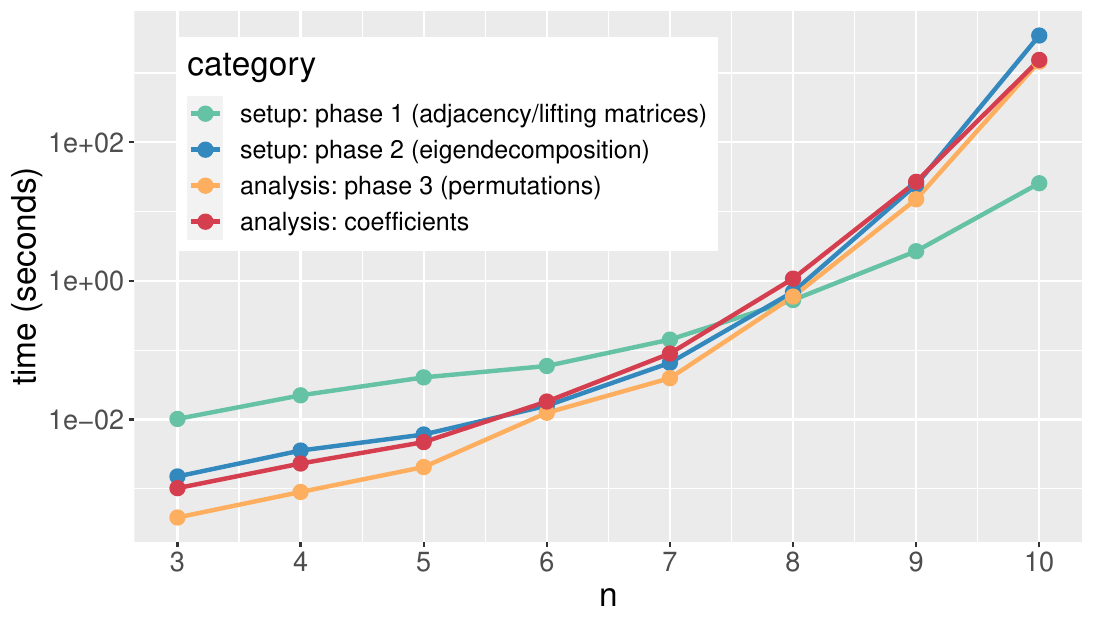} 
    \centerline{\small{~~~~~~~~(b)}}
    \end{minipage}
        \hfill
        \hfill
    \caption{Computation times for the setup and analysis portions in our (a) base implementation, and (b) more memory efficient implementation that moves some of the phase 3 setup computation into the analysis portion of the code.}
    \label{fig:comp_times}
      \vspace{0.1in}
\hrule height 1.5pt
\vspace{-.1in}

\end{figure}

When we only compute the analysis coefficients associated with the first $k$ symmetry types ($\H_n^k$), the two main computational bottlenecks are computing the permutations that rearrange ${\bf f}$ into $\rho_L(\sigma^{-1}) {\bf f}$ and computing the analysis coefficients, each of which has complexity of ${\cal O}(z_\gamma n!)$ for shape $\gamma$. The overall complexity is therefore equal to ${\cal O}(n^{\bar{m}_k}n!)$, where $\bar{m}_k$ is defined to be the smallest value of $m$ that results in $\sum_{i=1}^m p(i) \geq k$, where $p(i)$ is the partition number of integer $i$. For $k=5,8,15,30$, $\bar{m}_k$ is equal to 3,4,5,6, respectively. So, for the example above with the top $k=8$ shapes, which includes all shapes in lexicographic order up to $[n-4,4]$ (i.e., where most of the useful information resides), the complexity is ${\cal O}(n^4 n!)={\cal O}(n! \log^4(n!))$. With $n=10$ and $k=8$, which includes the computation of all of the largest magnitude coefficients in Fig. \ref{Fig:sushi_analysis} for the sushi preference data, both the setup and analysis portions run in under one minute.

Finally, to synthesize a signal from the analysis coefficients (i.e., perform the inverse transform), we need to lift each Schreier Laplacian eigenvector back to the permutahedron one time, reorder these vectors in different ways, and then take a linear combination of the reordered vectors:
$${\bf f}_{\hbox{rec}}=\sum_{\gamma \vdash n} \sum_{\lambda \in \Lambda_\gamma} \sum_{k=1}^{\kappa_{\gamma,\lambda}} \sum_{\bpi \in \bPi_\gamma} \alpha_{\gamma,\lambda,k,\bpi}{\boldsymbol \varphi}_{\gamma,\lambda,k,\bpi}=\sum_{\gamma \vdash n} \bar{c}_\gamma \sum_{\lambda \in \Lambda_\gamma} \sum_{k=1}^{\kappa_{\gamma,\lambda}} \sum_{\bpi \in \bPi_\gamma} \alpha_{\gamma,\lambda,k,\bpi} \rho_L(\sigma){\bf B}_{\pi_1} {\bf v}_{\gamma,\lambda,k},$$
where $\alpha_{\gamma,\lambda,k,\bpi}=\langle {\bf f},{\boldsymbol \varphi}_{\gamma,\lambda,k,\bpi} \rangle$.

%%%%%%%%%%%%%%%%%%%%%%%%%%%%%%%%%%%%%%%%%%%%%%%%%%%%%%%%%%%%%%%%%
\section{Related Work, Revisited} \label{Se:revisited}

\subsection{Bases, Frames, and Interpretability}

As discussed in Sec. \ref{Se:Fourier}, Diaconis \cite[p.~955]{diaconis1989generalization} remarks that there is not a natural choice of basis of the irreducible components $V_{\gamma,i}$ with interpretable basis elements, and circumvents this issue via Mallows' method of projecting an overcomplete spanning set of interpretable functions (two of which are shown in Fig. \ref{Fig:second_order}) onto the isotypic component $W_\gamma$. Interestingly, these projections also form a tight %group 
frame for $W_\gamma$.
\begin{proposition}\label{Prop:mallows}
For each $\pi,\xi \in \Pi_\gamma$, define ${\boldsymbol \delta}_{\pi,\xi} \in \Rbb[\SS_n]$ by
$${\boldsymbol \delta}_{\pi,\xi}(\sigma):=
\begin{cases}
1, &\hbox{if }{\sigma(\xi) = \pi} \\
0, &\hbox{otherwise} 
\end{cases}, \sigma \in \SS_n;$$ 
that is, ${\boldsymbol \delta}_{\pi,\xi}(\sigma) = 1$ if and only if $\sigma$ places the candidates in each block of $\pi$ in the ranking slots of the corresponding block of $\xi$. Let $({\boldsymbol \delta}_{\pi,\xi})_\gamma$ be the orthogonal projection of ${\boldsymbol \delta}_{\pi,\xi}$ onto $W_\gamma$. Then the collection of vectors $\{({\boldsymbol \delta}_{\pi,\xi})_\gamma\}_{\pi,\xi \in \Pi_\gamma}$ is a tight  
frame for $W_\gamma$. 
\end{proposition}

\begin{proof} For $\tau \in \SS_n$ we have $\tau {\boldsymbol \delta}_{\pi,\xi} = {\boldsymbol \delta}_{\tau( \pi),\xi}$ and ${\boldsymbol \delta}_{\pi,\xi} \tau =  {\boldsymbol \delta}_{ \pi,\tau^{-1}(\xi)}$, since $\sigma(\xi) = \pi$ if and only if $\tau \sigma(\xi) = \tau(\pi)$ if and only if $\sigma \tau (\tau^{-1} \xi) = \pi.$ Furthermore, projection onto the isotypic component $W_\gamma$ is in the center of the group algebra $\Rbb[\SS_n]$ (see \eqref{eq:isotypic_projector} and \cite[(13,19)]{waldron2018introduction}), so it commutes with both the left and right action of the group, giving  $\tau({\boldsymbol \delta}_{\pi,\xi})_\gamma = (\tau {\boldsymbol \delta}_{\pi,\xi})_\gamma = ( {\boldsymbol \delta}_{\tau(\pi),\xi})_\gamma$ and $({\boldsymbol \delta}_{\pi,\xi})_\gamma \tau = ( {\boldsymbol \delta}_{\pi,\xi}\tau )_\gamma = ( {\boldsymbol \delta}_{\pi,\tau^{-1}(\xi)})_\gamma$.  For any fixed $\pi, \xi \in \Pi_\gamma$, the irreducible left $\SS_n$-module generated by $({\boldsymbol \delta}_{\pi,\xi})_\gamma$ is given by 
$$
\Rbb[\SS_n] ({\boldsymbol \delta}_{\pi,\xi})_\gamma = \Rbb\text{-span}\{ \tau ({\boldsymbol \delta}_{\pi,\xi})_\gamma  \mid \tau \in \SS_n\} = \Rbb\text{-span}\{  ({\boldsymbol \delta}_{\tau(\pi),\xi})_\gamma  \mid \tau \in \SS_n\} =  \{({\boldsymbol \delta}_{\pi,\xi})_\gamma  \mid \pi \in \Pi_\gamma\},
$$
where the last equality follows from the fact that $\SS_n$ acts transitively on $\Pi_\gamma$. Since this module is irreducible, it follows from \cite[Thm 10.5]{waldron2018introduction}, that $\{  ({\boldsymbol \delta}_{\tau(\pi),\xi})_\gamma  \mid \tau \in \SS_n\}$ is a tight group frame for its span, and the span is isomorphic to $V_\gamma$. The symmetric argument with $\SS_n$ acting on the right produces a copy of the right irreducible $\SS_n$-module $V_\gamma^\ast$ and proves that $\{({\boldsymbol \delta}_{\pi,\xi})_\gamma  \mid \xi \in \Pi_\gamma\}$ is a tight  
frame for its span.  
The isotypic component is isomorphic to a tensor product $W_\gamma \cong V_\gamma \otimes V_\gamma^\ast$ of irreducible left and right $\SS_n$ modules, respectively.  Therefore, by \cite[Cor. 5.1]{waldron2018introduction},  $\{ ({\boldsymbol \delta}_{\pi,\xi})_\gamma  \mid \pi,\xi \in \Pi_\gamma\}$ is a tight frame for $W_\gamma$. 
 \qed
\end{proof}

On the other hand, we can distinguish a subset of our frame vectors that yield a basis for each $V_{\gamma,i}$. 
An ordered set partition $\pi \in \Pi_\gamma$ is \emph{standard} if when the rows are put in increasing order from left to right, the columns also are in increasing order from top to bottom. For example, the five ordered set partitions in the bottom right corner of the Schreier graph in Fig. \ref{Fig:P32} are standard. Standard ordered set partitions are in bijection with standard tableaux and $d_\gamma=\dim(V_\gamma)$ equals the number of standard ordered set partitions \cite{sagan2013symmetric}. Let $\Pi_\gamma^\text{std}$ denote the set of standard ordered set partitions. The atoms coming from liftings by standard set partitions provide the desired basis.
 
\begin{proposition} For $\gamma \vdash n$, $\lambda \in \Lambda_\gamma$, and $1 \le k \le \kappa_{\gamma,\lambda}$, the set  $\{{\boldsymbol \varphi}_{\gamma, \lambda, k, \pi} \mid \pi \in \Pi_\gamma^\text{std}\}$ is a basis for the irreducible $\SS_n$ module  $\Rbb[\SS_n]  \boldsymbol{\varphi}_{\gamma,\lambda,k,\pi_0} \cong V_\gamma$ (for any $\pi_0 \in \Pi_\gamma$).
\end{proposition}

\begin{proof} The module $V_\gamma^\ast \subseteq \Rbb[\Pi_\gamma]$ is spanned by the polytabloids $\{q_\pi \mid \pi \in \Pi_\gamma\}$ with symmetric group action $q_\pi \sigma = q_{\sigma^{-1}(\pi)}$ (see \eqref{eq:polytabloids}) and has a basis consisting of standard polytabloids $\{q_\pi \mid \pi \in \Pi_\gamma^\text{std}\}$ by \cite[Thm.\ 2.5.2]{sagan2013symmetric}. We show that this property transfers to $\Rbb[\SS_n]  \boldsymbol{\varphi}_{\gamma,\lambda,k,\pi_0} = \Rbb[\SS_n] (\B_{\pi_0}  \vb_{\gamma,\lambda,k})$.  
For any $0 \not = \vb \in V_\gamma^\ast$, the map $\Rbb[\SS_n] (\B_{\pi_0} \vb) \to V_\gamma^\ast$  given by sending $\sigma \B_{\pi_0} \vb$ to $\vb \sigma^{-1}$ is an isomorphism (the representing matrices are transposes of one another to account for the change from left to right modules). Moreover, for  $0 \not = \vb,\w \in V_\gamma^\ast \subseteq \Rbb[\Pi_\gamma]$, the action of the symmetric group on $(\Rbb[\SS_n] \B_{\pi_0}) \vb$ is identical to the action on $(\Rbb[\SS_n]  \B_{\pi_0})  \w$. To see this, observe that $\sigma \B_{\pi_0}  \vb = \B_{\sigma({\pi_0})} \vb$ and $\sigma \B_{\pi_0}  \w = \B_{\sigma({\pi_0})} \w$. Moreover, if $\sum_{\sigma \in \SS_n} c_\sigma \sigma \B_{\pi_0}  \vb = \sum_{\sigma \in \SS_n} c_\sigma \B_{\sigma(\pi_0)}  \vb=0$ is a dependence relation, then there is $x \in \Rbb[\SS_n]$ so that $\w = \vb x$ so $\sum_{\sigma \in \SS_n} c_\sigma \B_{\sigma(\pi_0)}  \w = \sum_{\sigma \in \SS_n} c_\sigma \sigma \B_{\pi_0}  (\vb x) =  (\sum_{\sigma \in \SS_n} c_\sigma \sigma \B_\pi \vb) x = 0$, since the left and right group actions commute. The converse argument holds, so dependence relations on $\{ \B_\pi \vb\}_{\pi \in \Pi_\gamma}$ and $\{ \B_\pi \w\}_{\pi \in \Pi_\gamma}$ are the same.
If $\vb = q_{\pi_0}$, then $\B_{\sigma(\pi_0)} q_{\pi_0} = \sigma \B_{\pi_0} q_{\pi_0}$ which maps to $q_{\pi_0} \sigma^{-1} = q_{\sigma(\pi_0)}$ under the isomorphism, so for each $\pi \in \Pi_\gamma$,  $\B_\pi q_{\pi_0}$ maps to $q_\pi$. It follows that $\{\B_\pi q_{\pi_0} \mid \pi \in \Pi_\gamma^\text{std}\}$ is a basis for $\Rbb[\SS_n] (\B_{\pi_0} q_{\pi_0})$. The result holds for $\Rbb[\SS_n] (\B_{\pi_0} \vb_{\gamma,\lambda,k})$ by the fact that the dependence relations are the same for $\{ \B_\pi q_{\pi_0}\}_{\pi \in \Pi_\gamma}$ and $\{ \B_\pi \vb_{\gamma,\lambda,k} \}_{\pi \in \Pi_\gamma}$. \qed \end{proof}

To recap, there are two important differences between the data analysis methods we propose here and those presented by Diaconis in \cite{diaconis1989generalization}. First, we directly construct each spanning vector as an interpretable function in $V_{\gamma,i}$, whereas the Mallows vectors are constructed as interpretable functions but then projected, during which some of the interpretability may be lost. Second, and probably more importantly, the atoms of the form ${\boldsymbol \varphi}_{\gamma,\lambda,k,\bpi}$ that we use reside in a single Laplacian eigenspace of the permutahedron,  whereas the atoms of the form $({\boldsymbol \delta}_{\pi,\xi})_\gamma$ from Prop. \ref{Prop:mallows} (and shown in the bottom of Fig.  \ref{Fig:second_order}) reside in multiple Laplacian eigenspaces of the permutahedron, making the corresponding analysis coefficients less interpretable from a smoothness perspective. Our proposed frame is therefore a more refined representation that enables us to capture both smoothness and symmetry information about the ranking data. 

 \subsection{Smoothness, Notions of Frequency, and Symmetry Types}

We mentioned in Sec. \ref{Se:gsp} that the symmetry types carry a notion of frequency: the Laplacian eigenvectors of the Cayley graph  
induced 
 by the generating set of all transpositions, $\Upgamma_n$, that reside in isotypic components that occur earlier in the dominance ordering are smoother functions with respect to the graph structure. This notion of frequency and relationship between dominance ordering and smoothness does not carry over directly to the setting of the permutahedron, as all but the first and last isotypic components contain Laplacian eigenvectors of $\PP_n$ associated with different eigenvalues. However, we believe there is still an interesting relationship between the dominance ordering 
of the isotypic components and the smoothness with respect to the permutahedron of the smoothest eigenvectors within each isotypic component. \\
\begin{wrapfigure}[25]{r}{0.45\textwidth}
\vspace{-.45in}
\centerline{\includegraphics[width=\linewidth,page=1]{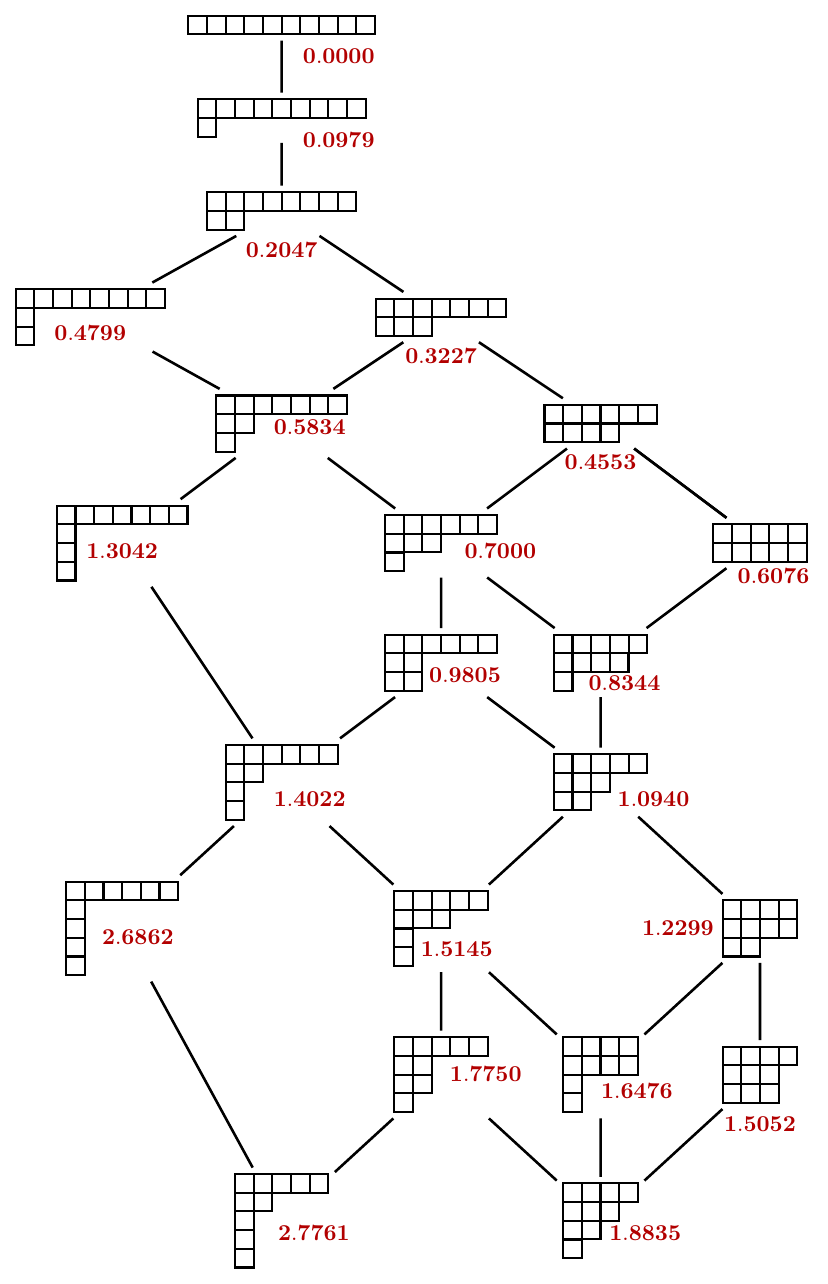}}
\caption{Dominance order poset with each of the 22 shapes in $\H_{10}$ labeled by the smallest eigenvalue $\tilde{\lambda}_\gamma$ associated with an eigenvector in the isotypic component $W_\gamma$. Conj. \ref{Conj:eig_order} also holds for the 20 conjugate shapes, which are not shown. \label{fig:dominance}}
\end{wrapfigure}
\vspace{-.3in}
\begin{conjecture}\label{Conj:eig_order}
For each $\gamma \vdash n$, define $\tilde{\lambda}_\gamma := \min_{\lambda \in \Lambda_\gamma}\{\lambda\}$. If $\nu \vartriangleright \gamma$, then $\tilde{\lambda}_\nu < \tilde{\lambda}_\gamma$.
\end{conjecture}
We have numerically verified this conjecture for $n\leq 10$. Fig. \ref{fig:dominance} shows an example of Conj. \ref{Conj:eig_order} with $n=10$.

\vspace{-.2in}
\subsection{Computational Complexity and Efficient Algorithms}

A naive implementation of the Fourier transform on $\SS_n$ requires  ${\cal O}((n!)^2)$ time. The FFT of Clausen and Baum \cite{clausen1993fast} reduces the time to ${\cal O}(n^3 n!)$ and Maslen \cite{Maslen} improves this to ${\cal O}(n^2 n!)$. It is conceivable that we could leverage some of the ideas from these works to further improve the complexity of applying our proposed transform. However, these methods also run into space constraints as they save time by storing many vectors of length $n!$ in their implementation. For example, most transforms on $\SS_n$ are built recursively from $\SS_{n-1}$ and the sequences of adjacent swaps $(i,i+1) (i+1,i+2) \cdots (n-1,n)$ for each $1 \le i \le n-1$ (coset representatives of $\SS_{n-1}$ in $\SS_n$). For these reasons in part, \cite{kondor2006snob,plumb2015snfft} implement algorithms to compute the FFT on individual isotypic components.

Another possible computational improvement we have not yet explored is to approximately compute the transform coefficients, avoiding the eigendecompositions of the Schreier Laplacians. This is the approach taken in most scalable graph signal processing algorithms \cite{hammond2011wavelets,shuman_distributed_SIPN_2018}. 

\subsection{Parametric Distance-Based Models}
Of the probability models for ranked data, the most closely related to our framework are distance-based and multistage models that use the Kendall distance metric 
(\cite[Ch. 6]{marden2014analyzing}, \cite[Ch. 8.3]{alvo2014statistical}, and \cite[Sec. 4.3]{yu2019analysis} all contain overviews of distance-based models). For example, Mallows' $\phi$-model \cite{mallows1957non}, \cite[Ch. 6]{Diaconis-book} assumes there is a single modal ranking $\sigma_0$ and the probability of a voter voting for another ranking $\sigma$ is proportional to $e^{-\tau d_K(\sigma,\sigma_0)}$, where $\tau$ is a single dispersion parameter and the Kendall distance $d_K(\sigma,\sigma_0)$ is equal to the number of hops separating $\sigma$ and $\sigma_0$ on the permutahedron. Generalizations of this model include mixture models with multiple modes and dispersion rates to represent heterogeneous voting populations \cite{murphy2003mixtures}, generalized Mallows models that allow for different dispersion parameters at each position of the permutation \cite{fligner1986distance}, and mixtures of generalized Mallows models \cite{meilua2010dirichlet}. A review and comparison of these models is in \cite{ceberio2015review} and software implementations are detailed in \cite{irurozki2016permallows,qian2019weighted}. 

From a graph signal processing perspective, fitting these models is closely related to blind deconvolution on graphs \cite{segarra2016blind,iwata2020graph,ramirez2021graph}; e.g.,
\begin{eqnarray}
\min_{{\bf x},h}\{ \lVert {\bf f} - h(\L){\bf x} \rVert_2\}~~~\hbox{s.t. }\lVert {\bf x} \rVert_0 \leq S,
\end{eqnarray}
where ${\bf x}$ is a sparse input signal whose support corresponds to the $S$ modes used in the mixture of Mallows models, $h$ is a diffusion filter which may be a polynomial of small degree or have a parametrized form such as $h(\lambda_{\ell})=\alpha e^{-\tau \lambda_{\ell}}$, and $\L$ is the graph Laplacian of either the (unweighted) permutahedron or a weighted variant that places different edge weights on edges corresponding to swaps of candidates in different adjacent positions in the rankings.

While our proposed approach models the length $n!$ data vector as the linear combination of a much higher number of atoms,
${\bf f} = \sum_{\phi \in \bar{\D}} \langle {\bf f}, \phi \rangle \phi$, the approach can also be used to generate compressed representations with atoms of the same form. This is because ranked data found in applications is typically smooth with respect to the permutahedron, and therefore the magnitudes of most of the coefficients  $ \langle {\bf f}, \phi \rangle$ are small (see, e.g., Fig. \ref{Fig:sushi_gftplus}). Thus, sparse linear combinations of these atoms can yield effective approximations to high-dimensional ranked data; i.e., ${\bf f} \approx \sum_{\phi \in \bar{\D}^{\prime}} \alpha_\phi \phi$, where $\bar{\D}^{\prime}$ is a subset of the tight frame dictionary $\bar{\D}$ with cardinality $|\bar{\D}^{\prime}| \ll n!$.

%%%%%%%%%%%%%%%%%%%%%%%%%%%%%%%%%%%%%%%%%%%%%%%%%%%%%%%%%%%%%%%%%
\section{Extensions and Future Directions} \label{Se:future}
We conclude with a brief mention of three lines of future work.

\vspace{-.2in}
\subsection{Closed Form Computation of the Schreier Eigenvalues and Eigenvectors} 
Availability of 
closed-form formulas for the Laplacian eigenvalues and eigenvectors of the Schreier graphs
could eliminate the need to perform the 
computations mentioned above altogether.  Closed forms are known for  Cayley graphs generated by all transpositions  \cite{DiaconisShahshahani} and for  Cayley graphs generated by transposing $i$ with $n$ for each $i$ (the star graph) \cite{Flatto}, but no general, closed formula is known in the case of adjacent transpositions (the permutahedron). Partial results are known \cite{Bacher,EdelmanWhite,friedman2000cayley}  when $\gamma$ is a ``hook shape," i.e., %of the form 
$\gamma = [n-k,{1,\ldots,1}]$ for $0 \le k \le  n-1$.

\begin{lemma}[\cite{Bacher,friedman2000cayley}]
Define the $k$-fold Cartesian product graph
\begin{align}\label{Eq:cube}
Q_{n,k} := Q_n^{\times k} = \underbrace{Q_n \times Q_n \times \cdots \times Q_n}_{k \text{copies}},
\end{align}
which is the $k$-dimensional cube graph of side length $n$. 
For each subset $I = \{i_1, \ldots, i_k\} \subseteq \{1, \ldots, n-1\}$ of size $k$, $\lambda_I:=\sum_{j=1}^k \lambda_{i_j}^{Q_n}$ is a Laplacian eigenvalue of $Q_{n,k}$, and the $k$-fold wedge product
\begin{align}\label{Eq:wedge_evec}
{\bf w}_{\gamma,\lambda_I} := {\bf v}_{i_1}^{Q_n} \wedge {\bf v}_{i_2}^{Q_n} \wedge \cdots \wedge  {\bf v}_{i_k}^{Q_n} = \frac{1}{\sqrt{k!}} \sum_{\sigma \in \SS_k} \sign(\sigma) 
{\bf v}_{i_\sigma(1)}^{Q_n} \otimes {\bf v}_{i_\sigma(2)}^{Q_n} \otimes \cdots \otimes  {\bf v}_{i_\sigma(k)}^{Q_n} 
\end{align}
is a Laplacian eigenvector of $Q_{n,k}$ associated with $\lambda_I$.
\end{lemma}

As shown in Fig. \ref{Fig:restriction}, the vertices of $Q_{n,k}$ can be labeled by the set of $k$-tuples from the alphabet $\{1, \ldots,n\}$:
$$
V(Q_{n,k}) = \left\{ (a_1, \ldots, a_{k}) \mid 1 \le a_j \le n \right\},
$$
and the vertices of the Schreier graph $\PP_{[n-k,{1,\ldots,1}]}$ can be labeled by the $k$-\emph{permutations} of $\{1, \ldots,n\}$:
$$
V(\PP_{[n-k,{1,\ldots,1}]}) = \left\{ (a_1, \ldots, a_{k}) \mid 1 \le a_j \le n, a_i \not= a_j \text{ if } i\not= j \right\}.
$$
Thus $V(\PP_{[n-k,{1,\ldots,1}]}) \subseteq V(Q_{n,k})$.

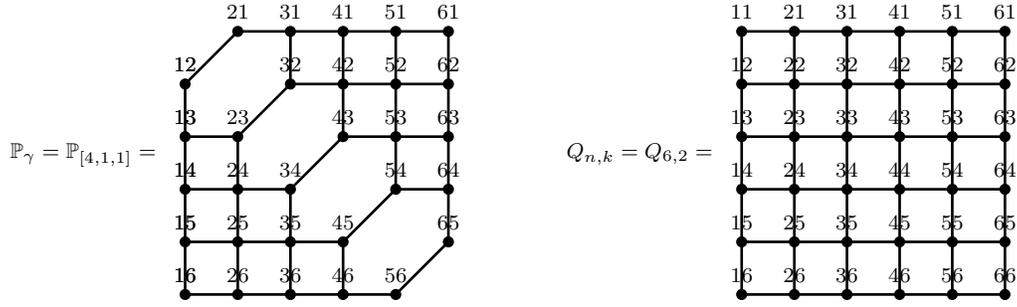
\begin{figure}[t]
$$
\PP_\gamma= \PP_{[4,1,1]} = \begin{array}{c}
\begin{tikzpicture}[scale=.7,line width=1.0pt]
\foreach \i in {1,...,6}  {  
	\foreach \j in {1,...,6} {
              \path (\i,-\j) coordinate (V\i\j);} }
%%%%%%%
\draw (V12) -- (V21) ; \draw (V23) -- (V32) ; \draw (V34) -- (V43) ; \draw (V45) -- (V54) ; \draw (V56) -- (V65) ;
\draw (V12) -- (V13) ;  \draw (V13) -- (V14) ; \draw (V14) -- (V15) ; \draw (V15) -- (V16) ;
\draw (V23) -- (V24) ; \draw (V24) -- (V25) ; \draw (V25) -- (V26) ;
\draw (V31) -- (V32) ; \draw (V34) -- (V35) ; \draw (V35) -- (V36) ;
\draw (V41) -- (V42) ; \draw (V42) -- (V43) ; \draw (V45) -- (V46) ;
\draw (V51) -- (V52) ; \draw (V52) -- (V53) ; \draw (V53) -- (V54) ;
\draw (V61) -- (V62) ; \draw (V62) -- (V63) ; \draw (V63) -- (V64) ; \draw (V64) -- (V65) ;
\draw (V13) -- (V23) ;  \draw (V14) -- (V24) ; \draw (V15) -- (V25) ; \draw (V16) -- (V26) ;
\draw (V21) -- (V31) ;  \draw (V24) -- (V34) ; \draw (V25) -- (V35) ; \draw (V26) -- (V36) ;
\draw (V31) -- (V41) ;  \draw (V32) -- (V42) ; \draw (V35) -- (V45) ; \draw (V36) -- (V46) ;
\draw (V41) -- (V51) ;  \draw (V42) -- (V52) ; \draw (V43) -- (V53) ; \draw (V46) -- (V56) ;
\draw (V51) -- (V61) ;  \draw (V52) -- (V62) ; \draw (V53) -- (V63) ; \draw (V54) -- (V64) ;
\draw  (V12)  node[above=0.05cm]{$\small 12$}; 
\draw  (V13)  node[above=0.05cm]{$\small 13$};
\draw  (V14)  node[above=0.05cm]{$\small 14$};
\draw  (V15)  node[above=0.05cm]{$\small 15$};
\draw  (V16)  node[above=0.05cm]{$\small 16$};
\foreach \i in {1}  {  \foreach \j in {2,3,4,5,6} { \fill (V\i\j) circle (3pt); \draw  (V\i\j)  node[above=0.05cm]{$\small \i\j$}; } } 
\foreach \i in {2}  {  \foreach \j in {1,3,4,5,6} { \fill (V\i\j) circle (3pt); \draw  (V\i\j)  node[above=0.05cm]{$\small \i\j$}; } } 
\foreach \i in {3}  {  \foreach \j in {1,2,4,5,6} { \fill (V\i\j) circle (3pt); \draw  (V\i\j)  node[above=0.05cm]{$\small \i\j$}; } } 
\foreach \i in {4}  {  \foreach \j in {1,2,3,5,6} { \fill (V\i\j) circle (3pt); \draw  (V\i\j)  node[above=0.05cm]{$\small \i\j$}; } } 
\foreach \i in {5}  {  \foreach \j in {1,2,3,4,6} { \fill (V\i\j) circle (3pt); \draw  (V\i\j)  node[above=0.05cm]{$\small \i\j$}; } }
\foreach \i in {6}  {  \foreach \j in {1,2,3,4,5} { \fill (V\i\j) circle (3pt); \draw  (V\i\j)  node[above=0.05cm]{$\small \i\j$}; } }
\end{tikzpicture}
\end{array}
\hskip.5in
Q_{n,k}=Q_{6,2} = 
\begin{array}{c}
\begin{tikzpicture}[scale=.7,line width=1.0pt]
\foreach \i in {1,...,6}  {  
	\foreach \j in {1,...,6} {
              \path (\i,-\j) coordinate (V\i\j);} }
%%%%%%%
\draw (V11) -- (V21) ; \draw (V11) -- (V12) ; 
\draw (V22) -- (V21) ; \draw (V22) -- (V12) ; \draw (V22) -- (V32) ; \draw (V22) -- (V23) ; 
\draw (V33) -- (V34) ; \draw (V33) -- (V43) ; \draw (V33) -- (V32) ; \draw (V33) -- (V23) ; 
\draw (V44) -- (V34) ; \draw (V44) -- (V43) ; \draw (V44) -- (V54) ; \draw (V44) -- (V45) ; 
\draw (V55) -- (V54) ; \draw (V55) -- (V45) ; \draw (V55) -- (V56) ; \draw (V55) -- (V65) ; 
\draw (V66) -- (V56) ; \draw (V66) -- (V65) ;
\draw (V12) -- (V13) ;  \draw (V13) -- (V14) ; \draw (V14) -- (V15) ; \draw (V15) -- (V16) ;
\draw (V23) -- (V24) ; \draw (V24) -- (V25) ; \draw (V25) -- (V26) ;
\draw (V31) -- (V32) ; \draw (V34) -- (V35) ; \draw (V35) -- (V36) ;
\draw (V41) -- (V42) ; \draw (V42) -- (V43) ; \draw (V45) -- (V46) ;
\draw (V51) -- (V52) ; \draw (V52) -- (V53) ; \draw (V53) -- (V54) ;
\draw (V61) -- (V62) ; \draw (V62) -- (V63) ; \draw (V63) -- (V64) ; \draw (V64) -- (V65) ;
\draw (V13) -- (V23) ;  \draw (V14) -- (V24) ; \draw (V15) -- (V25) ; \draw (V16) -- (V26) ;
\draw (V21) -- (V31) ;  \draw (V24) -- (V34) ; \draw (V25) -- (V35) ; \draw (V26) -- (V36) ;
\draw (V31) -- (V41) ;  \draw (V32) -- (V42) ; \draw (V35) -- (V45) ; \draw (V36) -- (V46) ;
\draw (V41) -- (V51) ;  \draw (V42) -- (V52) ; \draw (V43) -- (V53) ; \draw (V46) -- (V56) ;
\draw (V51) -- (V61) ;  \draw (V52) -- (V62) ; \draw (V53) -- (V63) ; \draw (V54) -- (V64) ;
\foreach \i in {1,2,3,4,5,6}  {  
\foreach \j in {1,2,3,4,5,6} { \fill (V\i\j) circle (3pt); \draw  (V\i\j)  node[above=0.05cm]{$\small \i\j$}; } } 
\end{tikzpicture}
\end{array}
$$
\caption{Embedding of the Schreier graph $\PP_\gamma$ in the $k$-dimensional cube graph $Q_{n,k}$.}\label{Fig:restriction}
 \vspace{0.1in}
\hrule height 1.5pt
\vspace{-.1in}
\end{figure}

\begin{proposition}[\cite{Bacher,friedman2000cayley}]
Let $\gamma$ be a ``hook shape"; i.e., $\gamma=[n-k,1^k]=[n-k,1,\dots,1]$ for $0\leq k\leq n-1$. For each subset $I = \{i_1, \ldots, i_k\} \subseteq \{1, \ldots, n-1\}$ of size $k$, $\lambda_I:=\sum_{j=1}^k \lambda_{i_j}^{Q_n}$ is a Laplacian eigenvalue of $\PP_\gamma$ with the associated eigenvector given by ${\bf v}_{\gamma,\lambda_I}:={\bf w}_{\gamma,\lambda_I} \vert_{V(\PP_\gamma)}$, the restriction of ${\bf w}_{\gamma,\lambda_I}$, viewed as a function on the vertices of $Q_{n,k}$, to the vertices of $\PP_\gamma$. Thus, for any hook shape $\gamma = [n-k,1^k]$, the complexity of computing each eigenvector of $\PP_\gamma$ is ${\mathcal O}(k!n^k)$, and the complexity of computing all eigenvectors of $\PP_\gamma$ is ${\mathcal O}\left(\frac{n!k!n^k}{(n-k)!}\right)$, which is bounded by ${\mathcal O}(k!n^{2k})$.
\end{proposition}

\begin{remark}
Earlier, Edelman and White \cite{EdelmanWhite} studied the integer eigenvalues of the permutahedron, found the integer eigenvalues from the hook shapes, and also conjectured that the only integer eigenvalues that appear in non-hook shapes also appear in the hook shapes for the same value of $n$. 
\end{remark}

\begin{remark}
The permutahedron $\PP_n$ is the same as the Full-Flag Johnson graph $FJ(n,1)$ studied in \cite{dai2018diameter}, in which Dai  uses the recursive structure of Full-Flag Johnson graphs to compute a subset of the adjacency spectrum of $\PP_n$, namely the eigenvalues of the matrix $M_n$ in \cite[Lem. 10]{dai2018diameter}. That matrix $M_n$ is actually equal to the adjacency matrix of the Schreier graph $\PP_{[n,n-1]}$, which is a path graph with additional self loops that make it regular.  Since, as mentioned in Sec. \ref{Se:path_schr}, the Laplacian eigenvalues of $\PP_{[n,n-1]}$ are known in closed form, an alternative closed form of the eigenvalues of Dai's $M_n$ matrix is $\{n-3+2\cos(\frac{\pi \l}{n})\}_{\l=0,1,\ldots,n-1} $, and we do indeed know which subset of the spectrum of $\PP_n$ is attained from this matrix.
\end{remark}

The Schreier graph $\PP_{[n-2,2]}$ (shown, e.g., in Fig. \ref{Fig:P32}) closely resembles the quartered Aztec diamond, the adjacency spectrum of which is studied and specified by Ciucu in \cite[Eq. (2.4)]{ciucu2008symmetry}. However, the quartered Aztec diamond graph studied by Ciucu does not include the self loops that are present in $\PP_{[n-2,2]}$, and is therefore not regular and the Laplacian eigenvalues do not follow immediately. We have not yet found a way to successfully adapt the approach of \cite{ciucu2008symmetry} to write down a closed-form formula for the Laplacian eigenvalues of $\PP_{[n-2,2]}$.

\vspace{-.2in}
\subsection{Partial and Incomplete Rankings} 
Our approach to construct tight frames for signals on the permutahedron can be extended to (i) \emph{partial rankings}, in which voters are allowed to include ties for one or more candidates \cite{critchlow2012metric,diaconis1989generalization,thompson1993generalized,baggerly1995visual,ukkonen2007visualizing,kidwell2008visualizing,kondor2010ranking}, and (ii) \emph{incomplete rankings}, in which voters may only rank a subset of the candidates \cite{kidwell2008visualizing,Malandro2013semigroup,clemenccon2014multiresolution,sibony2016multi,sibony2016multiresolution}. 
A common example of partial rankings is when voters list their top $k$ candidates, and the remaining $n-k$ are implicitly considered to be tied. An example of incomplete rankings is balanced incomplete block design, 
where each voter is assigned $k$ out of the $n$ objects to rank,
in such a way that each object is judged by the same number of voters and each pair of objects is 
presented to the same number of voters \cite[Ch. 11]{marden2014analyzing}, \cite{durbin1951incomplete,cochran1957,bailey2015spectral}. This is used, for example, if there are too many objects for an individual to effectively compare due to time constraints or cognitive limitations. Our frame vectors may yield a naturally interpretable spanning set for the possible vote tallies 
in this experimental design. Both of these extensions merit further investigation given the myriad applications in which partial and incomplete rankings appear.

\vspace{-.2in}
\subsection{Tight Frames for Analyzing Data on Other Cayley Graphs, Groups, and Combinatorial Structures}

The frame construction presented here works when the permutahedron is replaced by the Cayley graph $\Gamma(\SS_n, S)$ corresponding to any generating set $S \subseteq \SS_n$. For example, if $S$ is the set of all transpositions, then the Cayley graph is the graph $\Upgamma_n$ discussed in Sec.\ \ref{Se:gsp}, and if $S = \{(i,n)\mid 1 \le i < n\}$, then the Cayley graph is the star graph studied in \cite{Flatto}. In any of these cases, our methods yield a tight spectral frame with respect to the graph, and signals can be projected down to and interpreted on the corresponding Schreier graphs.   

Furthermore, these methods extend to any finite group. The equitable partition $\sim_\pi$ (see Def.\ \ref{Def:equivalencerel}) on $\SS_n$ induced by the ordered set partition $\pi \in \Pi_\gamma$ has equivalence classes equal to the right cosets of the stabilizer subgroup $\SS_\pi = \{\sigma \in \SS_n \mid \sigma(\pi) = \pi\}$, and the Schreier graph $\PP_\gamma$ is isomorphic to the Cayley graph determined by $\SS_n$ acting on these  cosets. When $\SS_n$ is replaced by any finite group ${\mathbb G}$ and $\SS_\mu$ is replaced by a subgroup $\mathbb H \le \mathbb G$, one obtains tight spectral frames (with the same geometric and energy-preserving properties) for studying data on $\mathbb G$ with respect to any Cayley graph. These methods naturally extend to groups that generalize the symmetric group such as Weyl groups, complex reflection groups, and finite general linear groups. For example, when applied to the hyperoctahedral group (the Weyl group of type $B$), these methods provide a setting to analyze data on signed permutations.  One can further extend these methods to analyze data on matchings, subsets, and set partitions by using the representation theory of the corresponding semisimple algebras with ``group-like"  structure, the Brauer, rook monoid, and partition algebras. Fourier analysis methods on these algebras are initiated in the recent work \cite{MRW2}.

\bibliographystyle{spmpsci.bst}      % mathematics and physical sciences
\bibliography{mbib.bib}   % name your BibTeX data base

\end{document}